\documentclass[hidelinks,11pt]{article} 
\usepackage{fullpage}
\usepackage{microtype}
\usepackage{graphicx}
\usepackage{subfig}
\usepackage{booktabs} 
\usepackage{hyperref}
\usepackage{xcolor}
\usepackage{dsfont}
\hypersetup{
    colorlinks=true,
    linkcolor=purple,
    citecolor=blue,
    filecolor=magenta,      
    urlcolor=cyan,
   }     
\usepackage{bm}

\usepackage[utf8]{inputenc} 
\usepackage[T1]{fontenc}    
\usepackage{url}            
\usepackage{amsfonts}       
\usepackage{nicefrac}       
\usepackage{amsmath,amssymb,amsthm}
\usepackage{color}
\usepackage{mathtools}
\usepackage{enumitem}
\usepackage{epstopdf}
\usepackage{tabulary}
\usepackage{cite}
\usepackage{chngcntr}
\usepackage{bm}
\usepackage[algo2e,linesnumbered,ruled,vlined,resetcount]{algorithm2e}
\RestyleAlgo{ruled}

\SetCommentSty{mycommfont}

\newtheorem{theorem}{Theorem}
\newtheorem{lemma}{Lemma}

\newtheorem{assumption}{Assumption}

\newtheorem{definition}{Definition}

\counterwithin{lemma}{section}
\counterwithin{theorem}{section}
\counterwithin{proposition}{section}
\counterwithin{corollary}{section}
\counterwithin{definition}{section}
\counterwithin{remark}{section}
\counterwithin{assumption}{section}
\counterwithin{fact}{section}

\newcommand{\norm}[1]{\ensuremath{\left\| #1\right\|}}

\newcommand{\A}{\ensuremath{\mathcal{A}}}
\newcommand{\K}{\ensuremath{\mathcal{K}}}
\newcommand{\CM}{\ensuremath{\mathcal{M}}}
\newcommand{\CN}{\ensuremath{\mathcal{N}}}
\newcommand{\R}{\ensuremath{\mathbb{R}}}
\renewcommand{\P}{\ensuremath{\mathbb{P}}}

\newcommand{\F}{\ensuremath{\mathcal{F}}}
\newcommand{\E}{\ensuremath{\mathbb{E}}}
\newcommand{\CE}{\ensuremath{\mathcal{E}}}

\newcommand{\X}{\ensuremath{\mathcal{X}}}
\newcommand{\BZ}{\ensuremath{\mathbb{Z}}}

\newcommand{\CO}{\ensuremath{\mathcal{O}}}

\DeclareMathOperator{\Tr}{\mathrm{tr}}

\DeclareMathOperator*{\argmin}{arg\,min}

\title{Online Learning of Kalman Filtering:\\ From Output to State Estimation}%
\author{%
Lintao Ye%
\thanks{School of Artificial Intelligence and Automation at the Huazhong University of Science and Technology, Wuhan, China; \texttt{\{yelintao93,zak,chiming\}@hust.edu.cn}}
\and
Ankang Zhang%
\footnotemark[1]
\and
Ming Chi%
\footnotemark[1]
\and
Bin Du
\thanks{College of Automation Engineering at Nanjing University of Aeronautics~and Astronautics,~Nanjing, China; \texttt{iniesdu@nuaa.edu.cn}}
\and
Jianghai Hu%
\thanks{Elmore Family School of Electrical and Computer Engineering at Purdue University, West Lafayette, IN, USA; \texttt{jianghai@purdue.edu}}
}
\begin{document}
\maketitle

\begin{abstract}
In this paper, we study the problem of learning Kalman filtering with unknown system model in partially observed linear dynamical systems. We propose a unified algorithmic framework based on online optimization that can be used to solve both the output estimation and state estimation scenarios. By exploring the properties of the estimation error cost functions, such as conditionally strong convexity, we show that our algorithm achieves a $\log T$-regret in the horizon length $T$ for the output estimation scenario. More importantly, we tackle the more challenging scenario of learning Kalman filtering for state estimation, which is an open problem in the literature. We first characterize a fundamental limitation of the problem, demonstrating the impossibility of any algorithm to achieve sublinear regret in $T$. By further introducing a random query scheme into our algorithm, we show that a $\sqrt{T}$-regret is achievable when rendering the algorithm limited query access to more informative measurements of the system state in practice. Our algorithm and regret readily capture the trade-off between the number of queries and the achieved regret, and shed light on online learning problems with limited observations. We validate the performance of our algorithms using numerical examples.
\end{abstract}

\section{Introduction}
Kalman filtering is a foundational approach to estimation and prediction of time-series data corrupted by stochastic noise \cite{harvey1990forecasting,commandeur2007introduction}. When applied to estimate the state of partially observed linear dynamical systems with Gaussian process and/or measurement noise, the Kalman filter achieves the minimum mean square estimation error among all state estimators of the system \cite{10.1115/1.3662552,anderson2005optimal}. However, constructing the optimal Kalman filter requires complete and exact knowledge of the underlying system model, which is generally absent in many real-world applications \cite{brunton2019data,ljung1999system}. A natural idea is to use the so-called certainty equivalence approach: Identifying the system model from measurement data and then using the identified system model in the Kalman filter design as if it was the true model \cite{simon1956dynamic}. Unfortunately, for partially observed linear dynamical system, 
one can only identify the system model up to some similarity transformation using system data \cite{ljung1999system,oymak2021revisiting,tu2016low}. This makes the approach of learning the Kalman filter using identified system model not applicable \cite{tsiamis2020sample}.

There is a pioneered line of research that takes an alternative approach to learning the Kalman filter with unknown system model \cite{kozdoba2019line,ghai2020no,tsiamis2022online,rashidinejad2020slip,qian2025model}, which leverages the fact that the Kalman filter constitutes a {\it linear regression} from past system outputs to the current state and thus does not need to explicitly identify the system's internal dynamics. Nonetheless, a caveat in the aforementioned work is that since the system state is not accessible, the regression is instead performed from past system outputs to the current output, which yields an estimate of the system output at the current time step, and the problem thus becomes the output estimation (or prediction) problem with unknown system model \cite{hazan2017learning,hardt2018gradient}. While it has been shown in \cite{ghai2020no,tsiamis2022online,rashidinejad2020slip,qian2025model} that least squares methods can be successfully used to obtain the regressed output estimates that are comparable and converging fast to the output estimates returned by the optimal Kalman filter, it does not solve the state estimation problem for which the Kalman filter was originally designed in Kalman's seminal work \cite{10.1115/1.3662552}. Motivated by this gap in the literature, we aim to answer the following questions in this paper: {\it Can we design a unified algorithmic framework that moves beyond learning Kalman filtering for output estimation to state estimation? Is there any fundamental limitation of learning Kalman filtering for state estimation, and if so, how to resolve it?}

As we readily mentioned above, a key challenge for learning Kalman filtering for state estimation is that the state information is absent hence the performance (i.e., state estimation error) of the designed filter cannot be directly evaluated. This creates a paradox in the algorithm design: Optimizing the performance of the learned filter while the performance itself cannot be evaluated. Such a paradox can be bypassed by considering the output estimation problem \cite{hazan2017learning,tsiamis2022online}, since the output estimation error can be directly evaluated given the output and the learned output estimate. Hence, we are further motivated to study whether the paradox creates any fundamental limitation in the problem of learning Kalman filtering for state estimation. Meanwhile, following the online learning setup studied in previous work \cite{hazan2018spectral,tsiamis2020sample,rashidinejad2020slip}, we consider the scenario where the system outputs become available in an online manner (i.e., over time). The filter needs to be designed based on the outputs seen before to produce an estimate of the system output or state at the current time step, and the learning performance is measured by the regret of the online algorithm \cite{hazan2016introduction}. The online learning setup is arguably more practical yet more challenging than its offline counterpart, where the learner acts after the whole sequence of data samples becomes available \cite{shalev2025online}. Our main contributions are summarized below.
\begin{itemize}[leftmargin=*]
\item {\bf Unified online optimization framework.} We propose a unified online optimization framework that encompasses both output and state estimation setups. Since the Kalman filter in hindsight produces the estimate of the output or state of the system at the current time step as a linear combination of the past outputs, we focus on optimizing over the set of linear filters and show that the estimation error cost is {\it conditionally strongly-convex} in the filter parameter. Such strong convexity allows us to show that an online convex optimization algorithm achieves a $\log T$-regret with respect to the horizon length $T$ for the output estimation problem in the first place, which matches those achieved in \cite{ghai2020no,rashidinejad2020slip,tsiamis2022online,qian2025model} using least squares methods. In addition, our proposed online optimization algorithmic framework is conceptually simpler than the least squares methods used, further allowing us to achieve the regret 
with minimal requirements on prior knowledge of the system model and noise statistics. Nonetheless, {\it the most important advantage} of our algorithmic framework is its ability to tackle the more challenging state estimation problem.

\item {\bf Fundamental limitation of learning state estimation.} Before delving into providing regret upper bounds for learning Kalman filtering for state estimation, we first characterize a fundamental performance limitation of any online learning algorithm applied to the state estimation problem. We demonstrate that for any online learning algorithm that is {\it only given access} to the system output history, the state estimates returned by the algorithm must incur a regret that scales {\it at least linearly} in the horizon length $T$, which is no better than, e.g., a naive algorithm that simply returns a bounded state estimate at each time step. Such a regret lower bound result indicates that learning state estimation is provably more difficult than learning output estimation.

\item {\bf Achievable performance with informative measurements.} As indicated by the above fundamental limitation, to hope for a meaningful regret sublinear in $T$, an online learning algorithm needs to receive {\it extra} information beyond the system output history. To this end, we consider a reasonable and practical scenario where the online algorithm has limited queries to more informative measurements of the underlying system's state. We show that under the extra queries of the measurements, achieving a $\sqrt{T}$-regret is possible. We also explicitly characterize the trade-off between the number of extra measurement points and the achievable regret. Such a trade-off is made possible by a novel random query scheme that we introduce into our online algorithm design, which may be of independent interest to the field of online optimization and online learning. 
\end{itemize}

{\bf Related Work.} The problem of adaptive filtering for output estimation of a linear dynamical system has been studied in the literature for decades \cite{ljung1978convergence,fuller1980predictors,wei1987adaptive,lai1991recursive,ding2006performance}. While being able to tackle uncertain or changing system model and noise statistics, this classic line of work mainly focuses on the asymptotic performance guarantees of the proposed methods. Recently, algorithms for learning Kalman filtering for output estimation with finite-time performance guarantees have been proposed with sample complexity analysis \cite{tsiamis2020sample} or regret bound analysis \cite{kozdoba2019line,ghai2020no,rashidinejad2020slip,tsiamis2020sample} for stable or marginally stable linear dynamical systems. The algorithm design in these papers is based on least squares methods, and the technical tools are carefully built upon recent advances in finite-sample analysis of control and system identification \cite{tsiamis2023statistical,oymak2021revisiting,faradonbeh2018finite,sarkar2021finite,jedra2022finite}. While the problem of learning Kalman filtering for state estimation has been considered in \cite{zhang2023learning,umenberger2022globally}, they assume that there is an oracle that can directly return the system state or state estimation error for the proposed algorithms.

It is well-known that the Kalman filter for (state) estimation and the Linear Quadratic Regulator (LQR) for control are dual problems, and when the system model is known, they yield the same solution strategy based on solving a set of Riccati equations \cite{anderson2005optimal,anderson2007optimal}. The important problem of learning LQR with unknown system model has been well-studied in the literature with an optimal $\sqrt{T}$-regret   \cite{abbasi2011regret,faradonbeh2020optimism,dean2020sample,mania2019certainty,kargin2022thompson,ye2024online}. However, when the system model is unknown, learning Kalman filtering and learning LQR are no longer dual of each other, since the learning landscape of the Kalman filter contains an extra set of parameters due to the measurement matrix in partially observed linear dynamical system. This is a common phenomenon identified in the literature on learning in partially observed systems \cite{tang2023analysis,umenberger2022globally} and learning in decentralized systems with limited state information \cite{ye2024learning,fattahi2020efficient,ye2022sample}.

{\bf Notations.}
The sets of integers and real numbers are denoted as $\mathbb{Z}$ and $\mathbb{R}$, respectively. 
For a real number $a$, let $\lfloor a \rfloor$ be the largest integer that is greater than or equal to $a$. 
A positive semidefinite matrix $P$ is denoted by $P\succeq0$, and $P\succeq Q$ if and only if $P-Q\succeq0$. For a vector $x\in\R^n$, let   $\norm{x}$
be its Euclidean norm. For a matrix $P\in\R^{n\times n}$, let $\rho(P)$, $\norm{P}$, and  $\norm{P}_F$
be its spectral radius, spectral norm, and Frobenius norm, respectively. The Kronecker product is denoted as $\otimes$. For a matrix $P\succeq0$, let $P^{1/2}$ be any matrix $P_1$ such that $P_1^{\top}P_1=P$, and let $\lambda_{\min}(P)$ (resp., $\lambda_{\max}(P)$) be its smallest (resp., greatest) eigenvalue.  
Let $I_n$ be an $n$-dimensional identity matrix. Given any integer $n\ge1$, $[n]\triangleq\{1,\dots,n\}$. 
Let $\sigma(\cdot)$ represent the sigma-field generated by the random variables in its argument. Let $\CN(0,P)$ be a zero-mean Gaussian distribution with covariance $P\succeq0$. Let $[P_i]_{i\in[n]}=\begin{bmatrix}P_1 &\cdots & P_n\end{bmatrix}$ be a concatenation of matrices $P_i\in\R^{n\times m_i}$.  

\section{Problem Formulation and Preliminaries}\label{sec:problem formulation and preliminaries}
Consider a discrete-time linear time-invariant (LTI) system given by
\begin{equation}
\label{eqn:LTI}
\begin{split}
x_{t+1}&=Ax_t+w_t\\
y_t&=Cx_t+v_t,
\end{split}
\end{equation}
where $x_t\in\R^n$ and $y_t\in\R^p$ are the state and output of the system, respectively, and $A\in\R^{n\times n}$ and $C\in\R^{p\times n}$ are the system matrices. The noise $w_t\overset{i.i.d.}{\sim}\CN(0,W)$ and $v_t\overset{i.i.d.}{\sim}\CN(0,V)$ with $W,V\succeq0$ are the process noise and measurement noise, respectively, which are assumed to be independent. For simplicity, we assume that the initial state satisfies $x_0=0$. We also make the following assumption throughout the paper.
\begin{assumption}
\label{ass:pairs}
The pair $(A,W^{1/2})$ is stabilizable and the pair $(A,C)$ is detectable.
\end{assumption}

{\bf Kalman filter for state estimation.} Let $\hat{x}_t^{\tt KF}$ be the estimator of the state $x_t$ given by the Kalman filter. The standard form of the Kalman filter is given by the following recursion \cite{anderson2005optimal}:
\begin{equation}\label{eqn:KF}
\hat{x}_{t+1}^{\tt KF}=(A-L_tC)\hat{x}_t^{\tt KF}+L_t  y_t
\end{equation}
initialized with $\hat{x}^{\tt KF}_0=0$. The matrix $L_t\in\R^{n\times p}$ is known as the {\it Kalman gain} given by 
\begin{equation}\label{eqn:Kalman gain}
L_t=A\Sigma_t C^{\top}(C\Sigma_t C^{\top}+V)^{-1},
\end{equation}
and $\Sigma_t$ is given by the following recursion initialized with $\Sigma_0=0$:
\begin{align}
\Sigma_{t}=A\Sigma_t A^{\top}-A\Sigma_{t-1} C^{\top}(C\Sigma_{t-1} C^{\top}+V)^{-1}C\Sigma_{t-1} A^{\top}+W.\label{eqn:Riccati}
\end{align}
It is well-known (see, e.g. \cite{anderson2005optimal}) that for any $t\ge0$, the Kalman filter $\hat{x}_t^{\tt KF}$ achieves the minimum mean square estimation error (MSEE), i.e.,
\begin{equation}\label{eqn:KF optimization problem}
\hat{x}_t^{\tt KF}\in\argmin_{\hat{x}_t\in\sigma(y_0,\dots,y_{t-1})}\E\big[\norm{x_t-\hat{x}_t}^2\big].
\end{equation}

{\bf Kalman filter for output estimation.} Apart from the state estimation problem, Kalman filtering can also be used to construct an estimator (or predictor) for the output $y_t$ of system~\eqref{eqn:LTI} as $\hat{y}_t^{\tt KF}=C\hat{x}_t^{\tt KF}$ \cite{hazan2018spectral,ghai2020no,tsiamis2022online}, where $\hat{x}_t^{\tt KF}$ is given by \eqref{eqn:KF}. Note that the output estimator $\hat{y}_t^{\tt KF}$ inherits the optimality of the state estimator $\hat{x}_t^{\tt KF}$ as we argued above, i.e., for any $t\ge0$, $\hat{y}_t^{\tt KF}$ achieves the minimum MSEE
\begin{equation}\label{eqn:optimality of output estimator}
\hat{y}_t^{\tt KF}\in\argmin_{\hat{y}_t\in\sigma(y_0,\dots,y_{t-1})}\E\Big[\norm{y_t-\hat{y}_t}^2\Big].
\end{equation}

Although the Kalman filter is optimal, its implementation through the steps above relies on the complete knowledge of the system matrices $A,C$ and the noise covariance matrices $W,V$. To overcome this limitation, the goal of this paper is to propose algorithms leveraging online optimization methods that learn a near-optimal filter without knowing $A,C,W,V$ of the system. 

\subsection{Steady-State Kalman Filter and Truncated Filters}\label{sec:truncated filter}
In this subsection, we introduce the class of filters considered in the online optimization algorithms that will be proposed in later sections. We first observe from \eqref{eqn:KF} and \eqref{eqn:KF optimization problem} that the idea of Kalman filtering is to compute the current state estimate $\hat{x}_t^{\tt KF}$ based on the output history $y_0,\dots,y_{t-1}$. Specifically, initialized with $\hat{x}_0^{\tt KF}=0$, we unroll \eqref{eqn:KF} to express $\hat{x}^{\tt KF}$ as
\begin{equation}\label{eqn:KF unrolled}
\hat{x}_t^{\tt KF}=\sum_{s=1}^t\Psi_{t,s}L_{s-1}y_{s-1},
\end{equation}
where for any $k,l\ge\BZ_{\ge0}$, $\Psi_{k,l}\triangleq(A-L_{k-1}C)(A-L_{k-2}C)\times\cdots\times(A-L_lC)$ if $k>l$ and $\Psi_{k,l}\triangleq I_n$ if $k\le l$. Thus, we may view $\big[\Psi_{t,s}L_{s-1}\big]_{s\in[t]}$ as the (optimal) parameter associated with the Kalman filter $\hat{x}_t^{\tt KF}$. However, the parameter $\big[\Psi_{t,s}L_{s-1}\big]_{s\in[t]}$ is time varying over $t$ and its dimension grows linearly as $t$. To overcome these drawbacks, we adopt the steady-state version of $\hat{x}_t^{\tt KF}$ (see, e.g., \cite[Chapter~4]{anderson2005optimal}) and truncate it to a fixed-length interval of the past outputs. Formally, denote 
\begin{equation}
\label{eqn:optimal M in KF}
M_{s}^{\star}\triangleq(A-LC)^{s-1}L\in\R^{n\times p},\ \forall s\ge1,
\end{equation}
where 
\begin{align}
L=A\Sigma C^{\top}(C\Sigma C^{\top}+V)^{-1},\label{eqn:steady state Kalman gain}
\end{align}
and $\Sigma$ is the unique positive semidefinite solution to the following discrete algebraic Ricatti equation:\footnote{The existence of a unique positive semidefinite solution to \eqref{eqn:Riccati steady} is guaranteed by Assumption~\ref{ass:pairs} \cite[Chapter~4]{anderson2005optimal}.}
\begin{equation}\label{eqn:Riccati steady}
\Sigma=A\Sigma A^{\top}-A\Sigma C^{\top}(C\Sigma C^{\top}+V)^{-1}C\Sigma A^{\top}+W.
\end{equation} 

Defining $y_t=0$ for all $t<0$, the steady-state Kalman filter $\hat{x}_t^{\tt SK}$ is given by 
\begin{align}
\hat{x}_t^{\tt SK}=\sum_{s=1}^tM_s^{\star}y_{t-s},\ \forall t\ge0,\label{eqn:steady-state KF}
\end{align} 
and a truncated steady-state Kalman filter $\hat{x}_t^{\tt TK}$ is given by 
\begin{equation}\label{eqn:KF truncated}
\hat{x}_t^{\tt TK}=\sum_{s=1}^h M_{s}^{\star}y_{t-s},\ \forall t\ge0,
\end{equation}
where $h\in\BZ_{\ge1}$ will be a tuning parameter in our algorithm that is kept fixed for all $t$. It follows from \eqref{eqn:KF truncated} that the parameter associated with $\hat{x}_t^{\tt TK}$ satisfies that $[M_s^{\star}]_{s\in[h]}\in\R^{n\times(ph)}$ and does not change over time $t$.

In the remainder of this paper, we will assume that there exist known constants $\alpha_0,\alpha_1,\psi,\bar{\sigma}\in\R_{>0}$ such that 
\begin{align}
\alpha_0I_n\preceq W\preceq\alpha_1I_n,\ \alpha_0I_p\preceq V\preceq\alpha_1I_p,\ 
\norm{C}\le\psi,\ \norm{\Sigma}\le\bar{\sigma}.\label{eqn:known constants}
\end{align}
In addition, we denote
\begin{align}
\kappa_F=\sqrt{\frac{\bar{\sigma}}{\alpha_0}},\ \gamma_F=1-\frac{\alpha_0}{2\bar{\sigma}},\label{eqn:def of gamma_F and kappa_F}
\end{align}
where $0<\gamma_F<1$ holds by definition (see Lemma~\ref{lemma:strongly stable of L} in Appendix~\ref{app:aux results}).

We now prove the following properties for the parameter $M^{\star}=[M_s^{\star}]_{s\in[h]}$ associated with $\hat{x}_t^{\tt TK}$; the proof is included in Appendix~\ref{app:preliminary proofs}.
\begin{lemma}
\label{lemma:truncated KF}
For any $s\in\BZ_{\ge1}$, $\norm{M_s^{\star}}\le\kappa_F^2\gamma_F^{s-1}$, $\norm{M_s^{\star}}_F\le\sqrt{\min\{p,n\}}\kappa_F^2\gamma_F^{s-1}$ and $\norm{M^{\star}}_F\le\sqrt{\min\{p,n\}}\frac{\kappa_F^2}{1-\gamma_F}$, where $M_s^{\star}$ given by~\eqref{eqn:optimal M in KF}.
\end{lemma}

Lemma~\ref{lemma:truncated KF} shows that the dependence of $\hat{x}^{\tt TK}_t$ on the past information $y_{t-s}$ decays exponentially as $s\in\BZ_{\ge1}$ increases, which will guide our choice of the interval length $h\in\BZ_{\ge1}$ used in $\hat{x}_t^{\tt TK}$. In addition, Lemma~\ref{lemma:truncated KF} specifies a set of matrices that the optimal parameter $M^{\star}$ belongs to, i.e., $M^{\star}\in\K_M$ with 
\begin{equation}\label{eqn:candidate truncated KF}
\K_M\triangleq\Big\{M=[M_{s}]_{s\in[h]}:\norm{M}_F\le\sqrt{\min\{p,n\}}\frac{\kappa_F^2}{1-\gamma_F}\Big\}.
\end{equation}
Based on the above arguments, one can show that optimizing over all $M\in\K_M$ will give us the parameter $M^{\star}$ in the truncated steady-state Kalman filter $\hat{x}_t^{\tt TK}$; the specific form of the corresponding optimization problem will be specified later in the next subsection.

The above ideas of using the steady-state version of the Kalman filter and truncating to a fixed interval can also be applied to the output estimation problem. Specifically, similarly to \eqref{eqn:steady-state KF}-\eqref{eqn:KF truncated}, the steady-state Kalman filter for output estimation and its truncated version are respectively given by 
\begin{align}
&\hat{y}_t^{\tt SK}=\sum_{s=1}^tN_s^{\star}y_{t-s},\ \forall t\ge0,\label{eqn:output steady-state KF}\\
&\hat{y}_t^{\tt TK}=C\hat{x}_t^{\tt TK}=\sum_{s=1}^hN_{s}^{\star}y_{t-s}, \forall t\ge0,\label{eqn:output KF truncated}
\end{align}
where we also define $y_t=0$ for all $t<0$ and
\begin{equation}
\label{eqn:optimal N in KF}
N_{s}^{\star}\triangleq C(A-LC)^{s-1}L\in\R^{p\times p},\ \forall s\ge1.
\end{equation}
The following result can be proved in a similar way to Lemma~\ref{lemma:truncated KF}.
\begin{lemma}
\label{lemma:truncated KF output}
For any $s\in\BZ_{\ge1}$, $\norm{N_s^{\star}}\le\kappa_F^2\psi\gamma_F^{s-1}$, $\norm{N_s^{\star}}\le\sqrt{p}\kappa_F^2\psi\gamma_F^{s-1}$, and $\norm{N^{\star}}\le\sqrt{p}\psi\frac{\kappa_F^2}{1-\gamma_F}$, where $N^{\star}=[N_s^{\star}]_{s\in[h]}$ with $N_s^{\star}$ given by~\eqref{eqn:optimal N in KF}.
\end{lemma}
Finally, similarly to \eqref{eqn:candidate truncated KF}, Lemma~\ref{lemma:truncated KF output} specifies a set of matrices such that the optimal parameter $N^{\star}=[N_s^{\star}]_{s\in[h]}$ associated with $\hat{y}_t^{\tt TK}$ belongs to, i.e., $N^{\star}\in\K_N$ with
\begin{equation}\label{eqn:candidate truncated KF output}
\K_N\triangleq\Big\{N=[N_{s}]_{s\in[h]}:\norm{N}_F\le\sqrt{p}\psi\frac{\kappa_F^2}{1-\gamma_F^2}\Big\}.
\end{equation}

\subsection{Online Optimization Framework and Problem Considered}\label{sec:online optimization framework}
The key to our algorithm design and analysis of learning the Kalman filter is to view the problem as an online (convex) optimization problem \cite{hazan2016introduction,shalev2012online}. Specifically, the basic framework of online optimization can be viewed as an interaction between a decision-maker and an environment over a time horizon of length $T\in\BZ_{\ge1}$. At each time step $t\ge0$, the decision-maker first chooses $x_t\in\X$ from an action set $\X\subseteq\R^n$ and incurs a cost given by $f_t(x_t)$, where $f_t(\cdot)$ is a (convex) cost function that only becomes available after the action $x_t$. The goal of the decision maker is to minimize a {\it regret} metric defined as ${\tt R}(T)=\sum_{t=0}^{T-1}f_t(x_t)-\sum_{t=0}^{T-1}\min_{x\in\X}f_t(x)$, which compares the cost incurred by the decision-maker and the optimal cost in hindsight, which can be obtained when the functions $f_0(\cdot),\dots,f_{T-1}(\cdot)$ are given a priori. If one can show that ${\tt R}(T)=o(T)$, then ${\tt R}(T)/T\to0$ as $T\to\infty$, i.e., $x_t$ approaches an optimal solution $x^{\star}\in\argmin_{x\in\X}f_t(x)$ as $t\to\infty$.

{\bf Learning Kalman filtering for output estimation.} Viewing the above framework in the language of learning Kalman filtering for output estimation, we see that at each time step $t\ge0$, the decision-maker first chooses a filter parameter $N_t=[N_{t,s}]_{s\in[h]}\in\K_N$ with $N_{t,s}\in\R^{p\times p}$ and $\K_N$ defined in \eqref{eqn:candidate truncated KF output}, and incurs the output estimation error $f_t(N_t)=\norm{y_t-\hat{y}_t(N_t)}^2$ with the output estimate $\hat{y}_t(N_t)=\sum_{s=1}^hN_{t,s}y_{t-s}$.\footnote{To simplify notations in the paper, we use a single notation $f_t(\cdot)$ to denote the objective function in different scenarios; its argument will indicate the underlying scenario that is being considered.} Since the decision-maker has access to $y_t$ at each time step $t\ge0$ and $\hat{y}_t(N_t)$ is determined by the decision-maker, the cost function $f_t(\cdot)=\norm{y_t-\hat{y}_t(\cdot)}$ is available to the decision-maker after choosing $N_t$ for any $t\ge0$. Such a scenario corresponds to the {\it full information} setting in online optimization literature \cite{zinkevich2003online,hazan2007logarithmic}. In other words, at each time step $t\ge0$, the cost function $f_t(\cdot)$ is fully characterized, which also allows the computation of the gradient $\nabla f_t(N)$ at any point $N\in\K_N$. As in the general online optimization framework, our goal in the problem of learning Kalman filtering for output estimation is to minimize the following regret:
\begin{align}
{\tt R_y}(T)=\sum_{t=0}^{T-1}\norm{y_t-\hat{y}_t(N_t)}^2-\sum_{t=0}^{T-1}\norm{y_t-\hat{y}_t^{\tt KF}}^2,\label{eqn:R_T^y}
\end{align}
where $\hat{y}_t^{\tt KF}$ is the Kalman filter for output estimation in hindsight given by \eqref{eqn:optimality of output estimator}.

{\bf Learning Kalman filtering for state estimation.} Similarly, we can use the online optimization framework to describe the problem of learning Kalman filtering for state estimation. At each time step $t\ge0$, the decision-maker chooses a filter parameter $M_t=[M_{t,s}]_{s\in[h]}\in\K_M$ with $M_{t,s}\in\R^{n\times p}$ and $\K_M$ defined in \eqref{eqn:candidate truncated KF}, and incurs the state estimation error $f_t(M_t)=\norm{x_t-\hat{x}_t(M_t)}^2$ with the state estimate $\hat{x}_t(M_t)=\sum_{s=1}^hM_{t,s}y_{t-s}$. Different from output estimation, we do not have direct access to the state $x_t$ at any time step $t\ge0$, so that the problem does not correspond to the full information setting in the online optimization framework. Still, the goal in the problem of learning Kalman filtering for state estimation is to minimize the regret defined as
\begin{align}
{\tt R_x}(T)=\sum_{t=0}^{T-1}\norm{x_t-\hat{x}_t(M_t)}^2-\sum_{t=0}^{T-1}\norm{x_t-\hat{x}_t^{\tt KF}}^2,\label{eqn:R_T^x}
\end{align}
where $\hat{x}_t^{\tt KF}$ is the Kalman filter for state estimation in hindsight given by \eqref{eqn:KF optimization problem}. The major challenge here is that the decision-maker does not have direct access to the cost $f_t(M_t)=\norm{x_t-\hat{x}_t(M_t)}^2$ incurred after choosing $M_t$ at each time step $t\ge0$. We will provide a remedy for this issue when we study the state estimation problem in Section~\ref{sec:state estimation} of the paper.

{\bf Summary of our main results.} In Section~\ref{sec:cost functions}, we characterize properties of the cost functions $f_t(\cdot)$ in both the output and state estimation problems, which will be useful throughout. In Section~\ref{sec:output estimation}, we propose to use an online gradient descent algorithm for learning Kalman filtering for output estimation and show that it yields an $\CO(\log^4 T)$ regret via the conditionally strong convexity of the cost function. In Section~\ref{sec:lower bounds for state estimation}, we first characterize a fundamental limitation of learning Kalman filtering for state estimation, proving that for any online algorithm given only access to system outputs $y_0,y_1,\dots$, the algorithm inevitably incurs an $\Omega(T)$ regret. In Section~\ref{sec:state estimation with C=I measurements}, we consider a scenario where the online algorithm has limited queries to more informative measurements of the system state (besides the original system output history), and propose a novel variant of online gradient descent algorithm based on random queries that achieves an $\tilde{\CO}(\tau+\sqrt{T})$ regret for learning Kalman filtering for state estimation, where $\tau$ is the number of queries for informative measurements.

\section{Cost Functions of the State and Output Estimation Problems}\label{sec:cost functions}
In this section, we elaborate on the expression of the cost functions in the state and output estimation problems. We begin with the following definitions.
\begin{definition}
\label{def:strongly convex}
(Strongly-convexity) A differentiable function $f:\X\to\R$ is $\mu$-strongly convex with $\mu>0$ if
\begin{equation*}
f(y)-f(x)\ge\langle y-x,\nabla f(x)\rangle+\frac{\mu}{2}\norm{y-x}^2, \forall x,y\in\X.
\end{equation*}
\end{definition}

\begin{definition}
\label{def:smooth}
(Smoothness) A differentiable function $f:\X\to\R$ is $\beta$-smooth with $\beta>0$ if 
\begin{equation*}
\norm{\nabla f(y)-\nabla f(x)}\le\beta\norm{y-x},\ \forall x,y\in\X.
\end{equation*}
\end{definition}

\begin{definition}
\label{def:Lip}
(Lipschitzness) A function $f:\X\to\R$ is $l$-Lipschitz  continuous with $l>0$ if 
\begin{equation*}
|f(y)-f(x)|\le l\norm{y-x},\ \forall x,y\in\X.
\end{equation*}
\end{definition}
Some useful facts are summarized below. For a twice differentiable function $f:\X\to\R$ with $\X\subseteq\R^d$, it is well-known that $\mu$-strongly convex (resp., $\beta$-smooth) as per Definition~\ref{def:strongly convex} (resp., Definition~\ref{def:smooth}) is equivalent to $\nabla^2f(x)\succeq\mu I_d$ (resp., $\nabla^2f(x)\preceq\beta I_d$) for all $x\in\X$ (see, e.g. \cite[Chapter~3]{bubeck2015convex}), where we also note that $\nabla^2 f(x)\preceq\norm{\nabla^2 f(x)}I_d$. Moreover, a direct consequence of $\beta$-smoothness is that $f(y)-f(x)\le\nabla f(x)^{\top}(y-x)+\frac{\beta}{2}\norm{y-x}^2$ \cite[Chapter~3]{bubeck2015convex}. For a differentiable function $f:\X\to\R$, the condition $\norm{\nabla f(x)}\le l$ for all $x\in\X$ implies that $f(\cdot)$ is $l$-Lipschitz \cite[Chapter~3]{bubeck2015convex}.

We will make the following assumption, which can also be found in existing work for learning in {\it partially observed} LTI systems \cite{lale2020logarithmic,oymak2021revisiting,sarkar2021finite,ye2024learning}.
\begin{assumption}
\label{ass:stable A matrix}
The system matrix $A$ satisfies that $\rho(A)<1$, i.e., there exist $\kappa_A\in\R_{\ge1}$ and $\gamma_A\in\R$ with $\rho(A)<\gamma_A<1$ such that $\norm{A^k}\le\kappa_A\gamma_A^k$ for all $k\in\BZ_{\ge0}$.
\end{assumption}
Note that the existence of $\kappa_A$ and $\gamma_A$ in Assumption~\ref{ass:stable A matrix} follows from the Gelfand formula \cite[Chapter~5]{horn2012matrix}.

{\bf Cost function of the state estimation problem.} Recall that in the problem of learning Kalman filtering for state estimation, the cost function is given by $f_t(M)=\norm{x_t-\hat{x}_t(M)}^2$ with $\hat{x}_t(M)=\sum_{s=1}^hM_s y_{t-s}$ and $M=[M_s]_{s\in[h]}$. Note that we may write
\begin{align}\nonumber
\hat{x}_t(M)=\sum_{s=1}^hM_{s}y_{t-s}&=[M_{s}]_{s\in[h]}[y_{t-s}^{\top}]_{s\in[h]}^{\top}\\\nonumber
&=\big(I_p\otimes[y_{t-s}^{\top}]_{s\in[h]}\big)\cdot\bm{vec}(M),
\end{align}
where $\bm{vec}(M)\in\R^{nph}$ concatenates the rows of $M$ into a (column) vector. Further denoting 
\begin{equation}\label{eqn:Y_t-h:t-1}
Y_{t-1:t-h}\triangleq I_{p}\otimes [y_{t-s}^{\top}]_{s\in[h]},
\end{equation}
for a given $M\in\R^{n\times(ph)}$, we can compactly write 
\begin{align*}
f_t(\bm{vec}(M))=(Y_{t-1:t-h}\bm{vec}(M)-x_t)^{\top}(Y_{t-1:t-h}\bm{vec}(M)-x_t),
\end{align*}
which yields
\begin{align}
&\nabla f_t(\bm{vec}(M))=2Y_{t-1:t-h}^{\top}(Y_{t-1:t-h}\bm{vec}(M)-x_t),\label{eqn:gradient of f_t(M_t)}\\
&\nabla^2 f_t(\bm{vec}(M))=2Y_{t-1:t-h}^{\top}Y_{t-1:t-h}.\label{eqn:Hessian of f_t(M_t)}
\end{align}

To proceed, we introduce the following probabilistic events regarding the noise $w_t,v_t$ in system~\eqref{eqn:LTI}: 
\begin{equation}
\begin{split}
&\CE_w=\Big\{\max_{0\le t\le T-1}\norm{w_t}\le\sqrt{5\Tr(W)\log\frac{3T}{\delta}}\Big\},\\
&\CE_v=\Big\{\max_{0\le t\le T-1}\norm{v_t}\le\sqrt{5\Tr(V)\log\frac{3T}{\delta}}\Big\},\label{eqn:probabilistic event}
\end{split}
\end{equation}
where $\delta\in(0,1)$. It can be shown that the above events hold with high probability; the proof of Lemma~\ref{lemma:high probability noise bound} is given in Appendix~\ref{app:preliminary proofs}.
\begin{lemma}
\label{lemma:high probability noise bound}
For any $T\ge1$, $\P(\CE_w)\ge1-\delta/3$, $\P(\CE_v)\ge1-\delta/3$ and $\P(\CE_w\cap\CE_v)\ge1-(2\delta)/3$.
\end{lemma}
For notational simplicity in the sequel, let us denote
\begin{equation}\label{eqn:R_x and R_y}
\begin{split}
&R_x\triangleq\sqrt{5\Tr(W)\log\frac{3T}{\delta}}\cdot\frac{\kappa_A}{1-\gamma_A},\\
&R_y\triangleq\norm{C}\sqrt{5\Tr(W)\log\frac{3T}{\delta}}\cdot\frac{\kappa_A}{1-\gamma_A}+\sqrt{5\Tr(V)\log\frac{3T}{\delta}}.
\end{split}
\end{equation}
Under the event $\CE_w\cap\CE_v$, we first provide the following bound on the norm of the state $x_t$ and output $y_t$ of system~\eqref{eqn:LTI}; the proof can be found in Appendix~\ref{app:preliminary proofs}.
\begin{lemma}
\label{lemma:upper bound on state and output}
Under the event $\CE_w\cap\CE_v$ and for any $t\in\{0,\dots,T-1\}$, it holds that $\norm{x_t}\le R_x$, $\norm{y_t}\le R_y$, and $\norm{Y_{t-1:t-h}}\le\sqrt{h}R_y$.
\end{lemma}

We then obtain the following properties for the cost function $f_t(\cdot)$ in the state estimation problem; the proof can be found in Appendix~\ref{app:preliminary proofs}. 
\begin{lemma}
\label{lemma:properties of f_t(M)}
For any $M\in\K_M$ with $\K_M$ defined in~\eqref{eqn:candidate truncated KF}, $\norm{\bm{vec}(M)}\le\sqrt{\min\{p,n\}}\frac{\kappa_F^2}{1-\gamma_F}$. Moreover, under the event $\CE_w\cap\CE_v$, for any $t\in\{0,\dots,T-1\}$ and any $M\in\K_M$, 
\begin{align}
&\norm{\nabla f_t(\bm{vec}(M))}\le2\sqrt{h}R_y(\sqrt{\min\{p,n\}}\kappa_F^2\frac{\sqrt{h}R_y}{1-\gamma_F}+R_x),\label{eqn:upperbound on the gradient of f_t(M)}\\
&\norm{\nabla^2f_t(\bm{vec}(M))}\le2hR_y^2,\label{eqn:upper bound on the hessian of f_t(M)}\\
&f_t(M)\le2R_x^2+2hR_y^2\min\{p,n\}\kappa_F^4\frac{1}{(1-\gamma_F)^2}.\label{eqn:upper bound on f_t(M)}
\end{align}
\end{lemma}
As we discussed before, the upper bounds in Lemma~\ref{lemma:properties of f_t(M)} imply that $f_t(\cdot)$ is Lipschitz and smooth.

{\bf Cost function of output estimation.} Recall that in the problem of learning Kalman filtering for output estimation, the cost function is given by $f_t(N)=\norm{y_t-\hat{y}_t(N)}^2$ with $\hat{y}_t(N)=\sum_{s=1}^hN_sy_{t-s}$ and $N=[N_s]_{s\in[h]}$. Similarly to the state estimation problem, we have the following results for the cost function $f_t(\cdot)$ in the output estimation problem; the proof is similar to that of Lemma~\ref{lemma:properties of f_t(N)} and is thus omitted for conciseness. 
\begin{lemma}
\label{lemma:properties of f_t(N)}
For any $N\in\K_N$ with $\K_N$ defined in \eqref{eqn:candidate truncated KF output}, $\norm{\bm{vec}(N)}\le\sqrt{p}\psi\frac{\kappa_F^2}{1-\gamma_F}$ and
\begin{align*}
&f_t(\bm{vec}(N))=(Y_{t-1:t-h}\bm{vec}(N)-y_t)^{\top}(Y_{t-1:t-h}\bm{vec}(N)-y_t),\\
&\nabla f_t(\bm{vec}(N))=2Y_{t-1:t-h}^{\top}(Y_{t-1:t-h}\bm{vec}(N)-y_t),\\ &\nabla^2f_t(\bm{vec}(N))=2Y_{t-1:t-h}^{\top}Y_{t-1:t-h}.
\end{align*}
Moreover, under the event $\CE_w\cap\CE_v$, for any $t\in\{0,\dots,T-1\}$ and any $N\in\K_N$,  
\begin{align*}
&\norm{\nabla f_t(\bm{vec}(N))}\le2\sqrt{h}R_y(\sqrt{p}\kappa_F^2\psi\frac{\sqrt{h}R_y}{1-\gamma_F}+R_y),\\
&\norm{\nabla^2 f_t(\bm{vec}(N))}\le2hR_y^2,\\
&f_t(N)\le2R_y^2+2hR_y^2p\kappa_F^4\psi^2\frac{1}{(1-\gamma_F)^2}.
\end{align*}
\end{lemma}

\section{Learning Kalman Filtering for Output Estimation}\label{sec:output estimation}
In this section, we study the output estimation problem, which corresponds to the full information setting in the online optimization framework as we discussed in Section~\ref{sec:online optimization framework}. Recall that the cost of the output estimation problem is given by $f_t(N_t)=\norm{y_t-\hat{y}_t(N_t)}^2$ with $\hat{y}_t(N_t)=\sum_{s=1}^hN_{t,s}y_{t-s}$, where the filter parameter $N_t=[N_{t,s}]_{s\in[h]}$ is chosen from the set $\K_N$ defined in ~\eqref{eqn:candidate truncated KF output}. 

\begin{algorithm2e}[h]
\SetNoFillComment
\caption{Online learning for output estimation}
\label{alg:OCO for output estimation}
\KwIn{Step sizes $\{\eta_t\}_{t\ge0}$, truncated filter length $h$.}
Initialize $N_0=0$. \\
\For{$t=0,1,\dots$}{
    Estimate $\hat{y}_t(N_t)=\sum_{s=1}^hN_{t,s}y_{t-s}$\\
    Obtain $f_t(N_t)=\norm{y_t-\hat{y}_t(N_t)}^2$\\
    Update $N_{t+1}=\Pi_{\K_N}(N_t-\eta_t\nabla f_t(N_t))$
}
\KwOut{$N_0,N_1,\dots$.}
\end{algorithm2e}

\subsection{Online Learning Algorithm}\label{sec:OGD for output estimation}
We propose using the online learning algorithm (Algorithm~\ref{alg:OCO for output estimation}), which employs the Online Gradient Descent (OGD) method. Specifically, in line~5 of the algorithm, $\Pi_{\K_N}(\cdot)$ denotes the projection of a matrix $N\in\R^{p\times(ph)}$ onto the set $\K_N$ defined in \eqref{eqn:candidate truncated KF output}. Noting that $\norm{N}_F=\norm{\bm{vec}(N)}$ for any $N\in\R^{p\times(ph)}$, the projection of $N\in\R^{p\times(ph)}$ onto the set $\K_N$ is equivalent to the projection of $\bm{vec}(N)$ onto the $p^2h$-dimensional ball with radius $\sqrt{p}\psi\frac{\kappa_F^2}{1-\gamma_F}$ centered around the origin. Therefore, we see that the projection $\Pi_{\K_N}(\cdot)$ in line~5 of Algorithm~\ref{alg:OCO for output estimation} can be computed as $\Pi_{\K_N}(N)=\sqrt{p}\psi\frac{\kappa_F^2}{1-\gamma_F}N/\norm{N}_F$ for all $N\in\R^{p\times(ph)}$. 

We now characterize the regret ${\tt R_y}(T)$ defined in \eqref{eqn:R_T^y} of Algorithm~\ref{alg:OCO for output estimation} when applied to solve the problem of learning Kalman filtering for output estimation.

\begin{theorem}
\label{thm:regret for output estimation}
Let the step sizes be $\eta_t=\frac{2}{(\alpha_0\log^{2}T)t}$ for all $t\ge1$ and $\eta_0=0$. Let the length $h$ of the truncated filter $\hat{y}_t^{\tt TK}$ given by \eqref{eqn:output KF truncated} be $h=\lfloor\log T/\log(1/\gamma_F)\rfloor$, where $\gamma_F$ is defined in \eqref{eqn:def of gamma_F and kappa_F}. Then, for any $T>1$ and any $\delta\in(0,1)$, it holds that ${\tt R}_y(T)=\CO(\log^{4} T)$, where $\CO(\cdot)$ hides polynomial factors in $\log\frac{1}{\delta}$ and problem parameters $p,n,\bar{\sigma},\alpha_0^{-1},\alpha_1,\psi,\frac{\kappa_A}{1-\kappa_A},\frac{1}{\log(2\bar{\sigma}/(2\bar{\sigma}-\alpha_0))}$.\footnote{Exact expression of ${\tt R}_y(T)$ can be found in the proof of Theorem~\ref{thm:regret for output estimation}.}
\end{theorem}

{\bf Challenges in the regret analysis.} For the full information setting in online optimization, it has been shown in the literature that for strongly convex cost functions $f_0(\cdot),\dots,f_{T-1}(\cdot)$ with $f_t:\R^d\to\R$ such that $\nabla^2f_t(N)\succeq\mu I_d$ for some $\mu>0$, the OGD algorithm achieves $\log T$-regret (e.g., \cite{hazan2007logarithmic}). However, since the output $y_t$ of system~\eqref{eqn:LTI} is random (due to the noise terms $w_t$ and $v_t$), the objective function $f_t(N)=\norm{y_t-\hat{y}_t(N_t)}^2$ is also random given any $N\in\K_N$, which implies that $\nabla^2 f_t(N)\succ0$ may no longer hold deterministically for any $t$. Thus, the existing results of OGD for strongly convex functions cannot be applied to our problem. In fact, using only the convexity of $f_t(\cdot)$ (that holds deterministically for all $t\ge0$), it was shown in \cite{kozdoba2019line} that OGD achieves an $\tilde{\CO}(\sqrt{T})$ regret for learning Kalman filtering for output estimation (of scalar linear dynamical system), where $\tilde{\CO}(\cdot)$ hides factors in $\log T$.  

{\bf Prior knowledge required by Algorithm~\ref{alg:OCO for output estimation}.} As mentioned in \eqref{eqn:known constants}, we assume known constants $\alpha_0,\psi,\bar{\sigma}$ (i.e., upper bounds on system-related matrices) so that the projection $\Pi_{\K_N}(\cdot)$ in line~5 of the algorithm can be implemented. Meanwhile, recalling from \eqref{eqn:def of gamma_F and kappa_F} that $\gamma_F=\sqrt{\bar{\sigma}{\alpha_0}}$,   setting the step sizes $\eta_t$ for all $t\ge1$ and truncated filter length $h$ as Theorem~\ref{thm:regret for output estimation} require $\alpha_0,\bar{\sigma}$ and the knowledge of the horizon length $T$. In Appendix~\ref{app:omitted proofs for output estimation}, we argue that using a doubling trick that has been widely used for online algorithms (see, e.g., \cite{lattimore2020bandit}), one can remove the dependency of $\eta_t$ and $h$ on $T$ and achieve the same $\CO(\log^4T)$ regret as Theorem~\ref{thm:regret for output estimation} without the knowledge of $T$. Finally, we note that assuming known constants such as $\alpha_0,\psi,\bar{\sigma}$ described above is typical in learning to estimate or control of unknown LTI systems \cite{rashidinejad2020slip,cohen2019learning,cassel2021online}.

{\bf Comparisons to existing work.} Compared to the existing work on learning the Kalman filter for output estimation \cite{ghai2020no,rashidinejad2020slip,tsiamis2022online,qian2025model} based on least squares methods, our approach has the following advantages: 1) our algorithm design based on OGD method is conceptually simpler, and our regret analysis is also more constructive, achieving comparable regret bounds (e.g., $\CO(\log^4T)$ in \cite{tsiamis2022online} and $\CO(\log ^3T)$ in \cite{qian2025model}); 2) we do not require a warm-up period in our algorithm (i.e., the regret holds for any $T>1$) or 
the regularity assumption that the matrix $A-LC$ in the steady-state Kalman filter (see \eqref{eqn:steady-state KF}) is diagonalizable \cite{tsiamis2022online,rashidinejad2020slip,qian2025model}; 3) our regret holds against the more general finite-step Kalman filter benchmark (see \eqref{eqn:R_T^y}) as apposed to the steady-state Kalman filter benchmark considered in \cite{ghai2020no,rashidinejad2020slip,tsiamis2022online,qian2025model}. While the regret bounds in \cite{ghai2020no,rashidinejad2020slip,tsiamis2022online,qian2025model} also hold for open-loop marginally stable system, we argue later in Section~\ref{sec:extensions} that our regret analysis can be extended to handle open-loop unstable system that can be stabilized with closed-loop input. Finally, we stress that the main merit of our algorithmic framework is its ability to handle the state estimation problem that will be studied in Section~\ref{sec:state estimation}.

\subsection{Proof of Theorem~\ref{thm:regret for output estimation}}\label{sec:proof for output estimation}
{\bf Notations used in this proof.} For notational simplicity in this proof, we assume that any filter parameter $N\in\R^{p\times(ph)}$ has been readily written in its vectorized form $\bm{vec}(N)\in\R^{p^2h}$ so that we simply denote $N$. We also reload the notations and use  $\Pi_{\K_N}(N)$ to represent the projection of a vector $N\in\R^{p^2h}$ onto $\K_N=\{N\in\R^{p^2h}:\norm{N}\le\sqrt{p}\psi\frac{\kappa_F^2}{1-\gamma_F}\}$, i.e., $\Pi_{\K_N}(N)=\argmin_{N^{\prime}\in\K_N}\norm{N^{\prime}-N}$. The following notations will also be used in this proof:
\begin{equation}\label{eqn:notations used in the proof for output estimation}
\begin{split}
&R_N\triangleq\sqrt{p}\psi\frac{\kappa_F^2}{1-\gamma_F},\ \beta_y\triangleq2hR_y^2,\\ &l_y\triangleq 2\sqrt{h}R_y^2(\sqrt{p}\psi\kappa_F^2\frac{\sqrt{h}}{1-\gamma_F}+1),
\end{split}
\end{equation}
where $R_y$ is defined in \eqref{eqn:R_x and R_y}. 

In the following, we provide the main steps in the proof and the omitted details are included in Appendix~\ref{app:omitted proofs for output estimation}. We begin by introducing an auxiliary function based on conditional expectation, which will be used throughout this proof.
\begin{definition}
\label{def:conditional cost}
Let $\{\F_t\}_{t\ge0}$ be a filtration with $\F_t\triangleq\sigma(w_0,\dots,w_{t-1},v_0,\dots,v_{t-1})$ for $t\ge1$ and $\F_0\triangleq\emptyset$. For any $N
\in\R^{p^2h}$ and some $h\in\BZ_{\ge0}$, define $f_{t;h}(N)=\E[f_t(N)|\F_{t-h}]$ for all $t \ge h$, where the expectation $\E[\cdot]$ is taken with respect to $\{w_t\}_{t\ge0}$ and $\{v_t\}_{t\ge0}$. 
\end{definition}

We show that the function $f_{t;h}(\cdot)$ defined above has the following properties.
\begin{lemma}
\label{lemma:hessian of f_t:k(N)}
For any $t\ge h$, $f_{t;h}(\cdot)$ is $\alpha_0$-strongly convex. In addition, under the event $\CE_w\cap\CE_v$ defined in \eqref{eqn:probabilistic event}, $f_{t;h}(\cdot)$ is $l_y$-Lipschitz.
\end{lemma}

{\bf Regret decomposition.} We decompose the regret ${\tt R}_y(T)$ in \eqref{eqn:R_T^y} as
\begin{align}\nonumber
{\tt R}_y(T)&=\sum_{t=0}^{T-1}f_t(N_t)-\sum_{t=0}^{T-1}\norm{y_t-\hat{y}_t^{\tt KF}}^2\\
&=\underbrace{\sum_{t=0}^{T-1}\big(f_t(N_t)-f_t(N^{\star})\big)}_{\text{online optimization regret}}+\underbrace{\sum_{t=0}^{T-1}\big(f_t(N^{\star})-\norm{y_t-\hat{y}_t^{\tt KF}}^2\big)}_{\text{truncation \& steady-state regret}},\label{eqn:regret decomposition output estimation}
\end{align}
where $f_t(N^{\star})=\norm{y_t-\hat{y}_t^{\tt TK}}^2$ and the comparator point  $N^{\star}=[N_s^{\star}]_{s\in[h]}$ with $N_s^{\star}$ defined in \eqref{eqn:optimal N in KF}. In the sequel, we separately analyze the two regret terms in \eqref{eqn:regret decomposition output estimation} that correspond to the regret incurred by the gradient-based update in Algorithm~\ref{alg:OCO for output estimation} and the regret incurred by the truncated Kalman filter considered in \eqref{eqn:output KF truncated}.

{\bf Online optimization regret.} 
We first prove the following intermediate result, leveraging the strong convexity of $f_{t;h}(\cdot)$ shown in Lemma~\ref{lemma:hessian of f_t:k(N)}.
\begin{lemma}
\label{lemma:upper bound on f_t;k difference output estimation} For any $t \ge h$, let $\varepsilon_t^s=\nabla f_t(N_t)-\nabla f_{t;h}(N_t)$. It holds that 
\begin{multline*}
\sum_{t=h}^{T-1}\big(f_{t;h}(N_t)-f_{t;h}(N^{\star})\big)\le-\frac{\alpha_0\log^{2}T}{4}\sum_{t=h}^{T-1}\norm{N_t-N^{\star}}^2+\frac{R_N\alpha_0h\log^{2}T}{2}\\
+\sum_{t=h}^{T-1}\frac{\eta_t}{2}\norm{\nabla f_t(N_t)}^2-\sum_{t=h}^{T-1}\langle\varepsilon_t^s,N_t-N^{\star}\rangle.
\end{multline*}
\end{lemma}

Next, we relate $f_{t;h}(\cdot)$ to $f_t(\cdot)$ and transform the upper bound in Lemma~\ref{lemma:upper bound on f_t;k difference output estimation} to an upper bound based on $f_t(\cdot)$.
\begin{lemma}
\label{lemma:upper bound on f_t difference output estimation}
Under the event $\CE_w\cap\CE_v$, it holds that 
\begin{multline*}
\sum_{t=h}^{T-1}\big(f_t(N_t)-f_t(N^{\star})\big)\le-\frac{\alpha_0\log^{2}T}{4}\sum_{t=h}^{T-1}\norm{N_t-N^{\star}}^2+\frac{R_N^2\alpha_0h\log^{2}T}{2}\\
+\sum_{t=h}^{T-1}\frac{\eta_t}{2}\norm{\nabla f_t(N_t)}+\frac{8l_yh}{\alpha_0\log^{2}T}(l_y+R_N\beta_y)(1+\log T)+\sum_{t=h}^{T-1}X_t(N^{\star}),
\end{multline*}
where $X_t(N^{\star})\triangleq\langle\nabla(f_t-f_{t;h})(N_{t-h}),N_{t-h}-N^{\star}\rangle+(f_t-f_{t;h})(N_{t-h})+(f_{t;h}-f_t)(N^{\star})$. Also, it holds that $\E[X_t(N^{\star})|\F_{t-h}]=0$ for all $t \ge h$.
\end{lemma}

To proceed, we provide a high probability upper bound on the summation of the stochastic terms $X_k(N^{\star}),\dots,X_{T-1}(N^{\star})$.
\begin{lemma}\label{lemma:high prob bound on X_t output estimation}
Suppose the event $\CE_w\cap\CE_v$ holds. Then, for any $\delta\in(0,1)$, the following holds with probability at least $1-\delta/3$:
\begin{align*}
\sum_{t=h}^{T-1}X_t(N^{\star})\le\frac{32l_y^2h\log\frac{3h}{\delta}}{\alpha_0\log^{2}T}+\frac{\alpha_0\log^{2}T}{4}\sum_{t=0}^{T-1}\norm{N_t-N^{\star}}^2.
\end{align*}
\end{lemma}

Finally, combining Lemmas~\ref{lemma:upper bound on f_t difference output estimation}-\ref{lemma:high prob bound on X_t output estimation}, we see that under the event $\CE_w\cap\CE_v$, the following holds with probability at least $1-\delta/3$:
\begin{align}\nonumber
&\sum_{t=0}^{T-1}\big(f_t(N_t)-f_t(N^{\star})\big)\\\nonumber
\le&\sum_{t=0}^{h-1}\big(f_t(N_t)-f_t(N^{\star})\big)+\frac{\alpha_0\log^{2}T}{4}\sum_{t=0}^{h-1}\norm{N_t-N^{\star}}^2+\frac{R_N^2\alpha_0h\log^{2}T}{2}\\\nonumber
&\qquad+\sum_{t=h}^{T-1}\frac{\eta_t}{2}\norm{\nabla f_t(N_t)}^2+\frac{8l_yh}{\alpha_0}(l_y+R_N\beta_y)(1+\log T)+\frac{32l_y^2h\log\frac{3h}{\delta}}{\alpha_0}\\\nonumber
\overset{(a)}{\le}&\frac{3R_N^2\alpha_0h}{2}\log^{2}T+4(R_y^2+hR_y^2R_N^2)h+\frac{l_y^2}{\alpha_0\log^2 T}(1+\log T)\\\nonumber
&\qquad\qquad\qquad\qquad+\frac{8l_yh}{\alpha_0\log^{2}T}(l_y+R_N\beta_y)(1+\log T)+\frac{32l_y^2h\log\frac{3h}{\delta}}{\alpha_0}\\
\overset{(b)}{=}&\CO(\log^4T),\label{eqn:final upper bound on f_t difference output}
\end{align}
To obtain $(a)$, we use the upper bound $\norm{N}\le R_N$ for any $N\in\K_N$ shown by Lemma~\ref{lemma:properties of f_t(N)}, the upper bound on $f_t(N)$ for any $N\in\K_N$ given by Lemma~\ref{lemma:properties of f_t(N)}, and the following derivation:
\begin{align*}
\sum_{t=h}^{T-1}\frac{\eta_t}{2}\norm{\nabla f_t(N_t)}^2&\le l_y^2\sum_{t=h}^{T-1}\frac{1}{(\alpha_0\log^2 T)t}\\
&\le\frac{l_y^2}{\alpha_0\log^2T}(1+\log T),
\end{align*}
where we use the upper bound $\norm{f_t(N_t)}\le l_y$ from Lemma~\ref{lemma:properties of f_t(N)} and the choice of the step size $\eta_t=\frac{2}{(\alpha_0\log^2 T)t}$ for all $t\ge1$. To obtain $(b)$, we first recall the expressions of $\kappa_F,\gamma_F$ in \eqref{eqn:known constants} and the choice of $h=\CO(\log T)$, and then deduce from \eqref{eqn:R_x and R_y} and \eqref{eqn:notations used in the proof for output estimation} that $R_y=\CO(\sqrt{\log T})$, $\beta_y=\CO(\log^2T)$, $l_y=\CO(\log^2T)$ and $R_N=\CO(1)$, where $\CO(\cdot)$ hides the factors stated in Theorem~\ref{thm:regret for output estimation}.

{\bf Truncation and steady-state regret.} To upper bound the truncation and steady-state regret in \eqref{eqn:regret decomposition output estimation}, we first recall from our discussions in Section~\ref{sec:truncated filter} that $\hat{y}_t^{\tt TK}$ in \eqref{eqn:output KF truncated} is a truncated steady-state version of the original Kalman filter $\hat{y}_t^{\tt KF}=C\hat{x}_t^{\tt KF}$ with $\hat{x}_t^{\tt KF}$ given by \eqref{eqn:KF}. Therefore, the truncation and steady-state regret in \eqref{eqn:regret decomposition output estimation} consists of two components corresponding to the truncation and the consideration of the steady-state of the Kalman filter. We have
\begin{equation}\label{eqn:truncation regret decomposition}
\sum_{t=0}^{T-1}\big(f_t(N^{\star})-\norm{y_t-\hat{y}_t^{\tt KF}}^2\big)=\underbrace{\sum_{t=0}^{T-1}\big(\norm{y_t-\hat{y}_t^{\tt TK}}^2-\norm{y_t-\hat{y}_t^{\tt SK}}^2\big)}_{(i)}+\underbrace{\sum_{t=0}^{T-1}\big(\norm{y_t-\hat{y}_t^{\tt SK}}^2-\norm{y_t-\hat{y}_t^{\tt KF}}^2\big)}_{(ii)},
\end{equation}
where $\hat{y}_t^{\tt TK}$ and $\hat{y}_t^{\tt SK}$ are defined in \eqref{eqn:output steady-state KF} and \eqref{eqn:output KF truncated}, respectively. We now prove the following results.

\begin{lemma}
\label{lemma:upper bound on term (i) in output truncation regret}
Under the event $\CE_w\cap\CE_v$, it holds that 
\begin{align}
\sum_{t=0}^{T-1}\big(\norm{y_t-\hat{y}_t^{\tt TK}}^2-\norm{y_t-\hat{y}_t^{\tt SK}}^2\big)\le\Big(2+\sqrt{h}R_N+\norm{L}\norm{C}\frac{\kappa_F}{1-\gamma_F}\Big)\frac{R_y^2\kappa_F\norm{L}\norm{C}}{1-\gamma_F}.\label{eqn:upper bound on truncation regret output estimation}
\end{align}
\end{lemma}

\begin{lemma}
\label{lemma:upper bound on term (ii) in output truncation regret}
Under the event $\CE_w\cap\CE_v$, it holds that 
\begin{align}
\sum_{t=0}^{T-1}\big(\norm{y_t-\hat{y}_t^{\tt SK}}^2-\norm{y_t-\hat{y}_t^{\tt KF}}^2\big)=\CO(\log^4T).\label{eqn:upper bound on term (ii) in output truncation regret}
\end{align}
\end{lemma}

{\bf Overall regret upper bound.} Since $\P(\CE_w\cap\CE_v)\ge1-(2\delta)/3$ from Lemma~\ref{lemma:high probability noise bound} and the upper bound in \eqref{eqn:final upper bound on f_t difference output} holds with probability at least $1-\delta/3$, we can obtain the overall regret upper bound $\CO(\log^4T)$ in Theorem~\ref{thm:regret for output estimation} by combining \eqref{eqn:final upper bound on f_t difference output} and  \eqref{eqn:upper bound on truncation regret output estimation}-\eqref{eqn:upper bound on term (ii) in output truncation regret} via a union bound, where we also recall the bounds $\norm{C}\le\psi$ from \eqref{eqn:known constants} and $\norm{L}\le\kappa_F$ from Lemma~\ref{lemma:strongly stable of L}. 

\section{Learning Kalman Filtering for State Estimation}\label{sec:state estimation}
In this section, we turn our attention to the state estimation problem, which is more challenging as we discussed in Section~\ref{sec:online optimization framework}. Recall that the cost function of the state estimation problem is given by $f_t(M_t)=\norm{x_t-\hat{x}_t(M_t)}^2$ with $\hat{x}_t(M_t)=\sum_{s=1}^hM_{t,s}y_{t-s}$, where the filter parameter  $M_t=[M_{t,s}]_{s\in[h]}$ is chosen from the set $\K_M$ defined in \eqref{eqn:candidate truncated KF}. 

\subsection{Lower Bound Results}\label{sec:lower bounds for state estimation}
We first characterize the fundamental limitations of any algorithm for the problem of learning Kalman filtering for state estimation; the proof of Theorem~\ref{thm:regret lower bound} is included in Appendix~\ref{app:state estimation}.

\begin{theorem}\label{thm:regret lower bound}
Consider any (potentially randomized) algorithm for learning the Kalman filter for state estimation that only has access to the outputs $y_0,\dots,y_{t-1}$ when designing the state estimate $\hat{x}_t$ for each time step $t\ge0$. Let $\sigma_w\in\R_{>0}$. Then, for any $T\ge2$, there exists a scalar and stable LTI system such that ${\tt R_x}(T)$ defined in \eqref{eqn:R_T^x} satisfies 
\begin{align}
\E\big[{\tt R_x}(T)\big]\ge6T\sigma_w-\frac{56+50\sigma_w}{1875\sigma_w},\label{eqn:expected lower bound}
\end{align}
where $\E[\cdot]$ is taken with respect to the potential randomness of the algorithm and $\{w_t\}_{t\ge0},\{v_t\}_{t\ge0}$. Additionally, consider any algorithm that returns $\hat{x}_t$ with $\norm{\hat{x}_t}\le c_{\A}$ for all $t\ge0$ almost surely. Then, for any $T\ge3$ and any $\delta\in(0,1)$, it holds with probability at least $1-\delta$ that   
\begin{align}
{\tt R_x}(T)\ge6T\sigma_w-\frac{56+50\sigma_w}{1875\sigma_w}-{\CO}\left(\sqrt{T(c_{\A}+\log^{1.5}T)\log\frac{3}{\delta}}\right).
\label{eqn:high probability lower bound}
\end{align} 
\end{theorem}
{\bf Proof idea.} The proof hinges on the fact that for partially observed systems, there can be multiple systems with different system matrices $A,C$ that yield the same output history $y_0,y_1,\dots$ (given the same realization of the noise sequences $\{w_t\}_{t\ge0}$ and $\{v_t\}_{t\ge0}$). Hence, any algorithm (given only access to $y_0,y_1,\dots$) cannot perform well simultaneously on all such systems, i.e., any algorithm must incur large regret for at least one of such systems. 

Theorem~\ref{thm:regret lower bound} demonstrates that any algorithm for the problem of learning Kalman filtering for state estimation has to incur a regret at least $\Omega(T)$, and the lower bound $\Omega(T)$ holds for both expected regret and high probability regret. Thus, we know from Theorem~\ref{thm:regret lower bound} that the problem of learning Kalman filtering for state estimation is provably more difficult than the problem of learning output estimation studied in Section~\ref{sec:output estimation} and also the problem of learning LQR (with unknown system model). In particular, for the latter problem, it has been shown that an $\tilde{\CO}(\sqrt{T})$ regret is achievable \cite{mania2019certainty}, matching the $\Omega(\sqrt{T})$ regret lower \cite{simchowitz2020naive}.

\subsection{Online Learning Algorithm with Informative Measurements}\label{sec:state estimation with C=I measurements}
From the fundamental limitation results given by Theorem~\ref{thm:regret lower bound}, one cannot hope for sublinear regret in $T$ for the state estimation problem if the online learning algorithm only has access to the outputs $y_0,y_1,\dots$. One natural idea is then to consider that the algorithm has access to limited yet more informative measurements from system~\eqref{eqn:LTI}. In particular, we assume that at certain time steps $t\ge0$, the following measurements can potentially be made available to the learner:
\begin{align}
\tilde{x}_t=x_t+\tilde{v}_t,\label{eqn:informative measurement}
\end{align}
where the measurement noise satisfies that $\tilde{v}_t\overset{i.i.d.}{\sim}\CN(0,\tilde{V})$ with $\tilde{V}\succeq0$ and $\{\tilde{v}_t\}_{t\ge0}$ is assumed to be independent of $\{w_t\}_{t\ge0}$ and $\{v_t\}_{t\ge0}$ in system~\eqref{eqn:LTI}. For example, the measurement $\tilde{x}_t$ may come from a sensor on system~\eqref{eqn:LTI} with measurement matrix $C=I_n$ or a sensor with a known and full column rank measurement matrix $C\in\R^{p\times n}$. While being more informative, it can be costly to use such a sensor to collect measurements from system~\eqref{eqn:LTI}. This motivates us to consider the scenario where the learner may only query for the measurements $\tilde{x}_t$ at a limited number of time steps, as will be detailed in our algorithm design next. Finally, we note that the measurement model in \eqref{eqn:informative measurement} has been widely adopted in the literature on learning for unknown linear systems \cite{soderstrom2007errors,khorasani2020non,nonhoff2024online,zhang2025sample}.

\begin{algorithm2e}[h]
\SetNoFillComment
\caption{Online learning for state estimation}
\label{alg:OCO for state estimation}
\KwIn{Step sizes $\{\eta_j\}_{j\ge0}$, truncated filter length $h$, parameter $\tau\in\BZ_{\ge1}$.}
Initialize $M_0=0$ and $j=0$.\\
\For{$i=0,1,\dots$}{
    Sample $b_i\overset{i.i.d.}{\sim}{\tt Unif}(\{0,\dots,\tau-1\})$\\
    \For{$t=i\tau,\dots,(i+1)\tau-1$}{
    Compute $\hat{x}_t(M_t)=\sum_{s=1}^hM_{t-s}y_{t-s}$\\
    \If{$t\mod \tau=b_i$}{
    Query to obtain $\tilde{x}_t=x_t+\tilde{v}_t$\\
    Compute $\tilde{f}_t(M_t)=\norm{\tilde{x}_t-\hat{x}_t(M_t)}^2$\\
    Update $M_{t+1}=\Pi_{\K_M}(M_t-\eta_j\nabla\tilde{f}_t(M_t))$\\
    $j\gets j+1$
    }
    \Else{$M_{t+1}=M_t$}
    }
}
\KwOut{$M_0,M_1,\dots$.}
\end{algorithm2e}

{\bf Description of the online learning algorithm.} We now introduce Algorithm~\ref{alg:OCO for state estimation} for learning the Kalman filter for state estimation under the informative measurement setting introduced above. To respect the limited access to informative measurements in \eqref{eqn:informative measurement}, we leverage a randomly sampled $b_i\overset{i.i.d.}{\sim}{\tt Unif}(\{0,\dots,\tau-1\})$ in line~3 of Algorithm~\ref{alg:OCO for state estimation}, where ${\tt Unif}(\{0,\dots,\tau-1\})$ denotes a uniform distribution over the integers in $\{0,\dots,\tau-1\}$ and $\tau$ is an input parameter to the algorithm. In a time interval of length $\tau$, Algorithm~\ref{alg:OCO for state estimation} then queries to obtain $\tilde{x}_t$ once when $t\mod\tau=b_i$, and updates the filter parameter $M_t$ accordingly (lines~6-10). Otherwise, Algorithm~\ref{alg:OCO for state estimation} keeps the same filter parameter $M_t$ to the next time step $t+1$ (line~12). Note that the update in line~9 of Algorithm~\ref{alg:OCO for state estimation} can be viewed as a noisy version of the update in line~5 of Algorithm~\ref{alg:OCO for output estimation}, i.e., a noisy gradient $\nabla\tilde{f}_t(M_t)$ is used rather than the true gradient $\nabla f_t(M_t)$. Finally, note that the for loop from lines~4 to 12 in Algorithm~\ref{alg:OCO for state estimation} terminates when $t$ reaches $T-1$; moreover, for any horizon length $T\ge1$, Algorithm~\ref{alg:OCO for state estimation} queries to obtain the information measurement in \eqref{eqn:informative measurement} for at most $\lfloor T/\tau\rfloor$ time steps.

We now provide upper bounds on the regret ${\tt R_x}(T)$ defined in \eqref{eqn:R_T^x} of Algorithm~\ref{alg:OCO for state estimation}.

\begin{theorem}
\label{thm:regret for state estimation}
Let the step sizes be $\eta_j=\frac{2}{\alpha_0j}$ for all $j\ge1$ and $\eta_0=0$. Let the length $h$ of the truncated filter $\hat{x}_t^{\tt TK}$ given by \eqref{eqn:KF truncated} be $h=\lfloor\log T/\log(1/\gamma_F)\rfloor$ and let $\tau\ge h$, where $\gamma_F$ is defined in \eqref{eqn:def of gamma_F and kappa_F}. Then, for any $T>1$ and any $\delta\in(0,1)$, with probability at least $1-\delta$, it holds that ${\tt R_x}(T)=\CO\big(\tau(\log^4T)\log(T/\tau)+\sqrt{T}\log^2T\big)$, where $\CO(\cdot)$ hides polynomial factors in $\log\frac{1}{\delta}$ and problem parameters $p,n,\bar{\sigma},\alpha_0^{-1},\alpha_1,\psi,\frac{\kappa_A}{1-\kappa_A},\frac{1}{\log(2\bar{\sigma}/(2\bar{\sigma}-\alpha_0))},\Tr(\tilde{V})$.\footnote{Exact expression of ${\tt R_x}(T)$ can be found in the proof of Theorem~\ref{thm:regret for state estimation}.}
\end{theorem}
{\bf Trade-off between queries and regret.} As can be seen from the results in Theorem~\ref{thm:regret for state estimation}, there is a clear trade-off between the value of the input parameter $\tau$ to Algorithm~\ref{alg:OCO for state estimation} and the overall regret of the algorithm. In particular, we may write ${\tt R_x}(T)=\tilde{\CO}(\tau+\sqrt{T})$, where $\tilde{\CO}(\cdot)$ further compresses $\log T$ factors. Recalling that $\lfloor T/\tau\rfloor$ is the number of times that Algorithm~\ref{alg:OCO for state estimation} queries $\tilde{x}_t$, the algorithm achieves a sublinear regret in $T$ provided that $\lfloor T/\tau\rfloor$ is also sublinear in $T$. A seemly more straightforward way to reduce the queries of $\tilde{x}_t$ in Algorithm~\ref{alg:OCO for state estimation} is to sample $b_t\overset{i.i.d.}{\sim}{\tt Bernoulli}(1/\tau)$ and query for $\hat{x}_t$ if and only if $b_t=1$. However, such a sampling scheme only guarantees $\lfloor T/\tau\rfloor$ number of queries for $\tilde{x}_t$ in expectation. In contrast, our proposed sampling scheme in line~3 of Algorithm~\ref{alg:OCO for state estimation} ensures that the number of queries is given by $\lfloor T/\tau\rfloor$ deterministically, which facilitates our regret analysis in Theorem~\ref{thm:regret for state estimation}.

{\bf The role of random queries.} Algorithm~\ref{alg:OCO for state estimation} invokes the gradient-based update for the filter parameter $M_t$ in line~9 using the queried $\tilde{x}_t$, when the value of the random variable $b_i$ sampled in line~3 satisfies the condition in line~6. In our proof of Theorem~\ref{thm:regret for state estimation} below, we crucially leverage this randomness in the gradient-based update to {\it extend} the regret over the steps when $M_t$ is updated {\it to} the steps over the whole time horizon $t=0,\dots,T-1$, achieving the desired balance between the number of queries of $\tilde{x}_t$ and the regret of Algorithm~\ref{alg:OCO for state estimation}. Such a novel random query (or sampling) scheme also sheds light on online learning problems with costly or limited observations (e.g., \cite{hazan2012linear,ghoorchian2024contextual}). 

{\bf Prior knowledge required by Algorithm~\ref{alg:OCO for state estimation}.} Similarly to our arguments for Algorithm~\ref{alg:OCO for state estimation}, to achieve the regret in Theorem~\ref{thm:regret for state estimation}, Algorithm~\ref{alg:OCO for state estimation} requires the knowledge of $\alpha_0,\psi,\bar{\sigma}$ given by \eqref{eqn:known constants} and the horizon length $T$. Again, we argue in Appendix~\ref{app:state estimation} that one can use a doubling trick to achieve the same regret as Theorem~\ref{thm:regret for state estimation} without the knowledge of $T$.

\subsection{Proof of Theorem~\ref{thm:regret for state estimation}}\label{sec:proof for state estimation}
{\bf Notations used in this proof.} Similarly to the proof of Theorem~\ref{thm:regret for output estimation}, we assume that any filter parameter $M\in\R^{n\times (ph)}$ has been readily written in its vectorized form $\bm{vec}(M)\in\R^{nph}$ so that we simply denote $M$. We also use $\Pi_{\K_M}(M)$ to denote the projection of a vector $M\in\R^{nph}$ onto the set $\K_M=\{M\in\R^{nph}:\norm{M}\le\sqrt{\min\{p,n\}}\frac{\kappa_F^2}{1-\gamma_F}\}$. For our analysis in this proof, we also define the following probabilistic event regarding the noise $\{\tilde{v}_t\}_{t\ge0}$ in \eqref{eqn:informative measurement}:
\begin{align}
\CE_{\tilde{v}}=\Big\{\max_{0\le t\le T-1}\norm{\tilde{v}_t}\le\sqrt{5\Tr(\tilde{V})\log\frac{12T}{\delta}}\Big\}.\label{eqn:event E_tilde v}
\end{align}
Using similar arguments to those for Lemma~\ref{lemma:high probability noise bound}, one can show that $\P(\CE_{\tilde{v}})=1-\delta/12$. Noting that Algorithm~\ref{alg:OCO for state estimation} updates the filter parameter $M_t$ in line~9 if and only if $t\mod \tau=b_i$ with $b_i$ sampled in line~3 of the algorithm, we denote the set of time steps $t$ in $\{0,\dots,T-1\}$ such Algorithm~\ref{alg:OCO for state estimation} updates $M_t$ as $S=\{t\ge0:t\mod\tau=b_i\}=\{t_0,\dots,t_{H-1}\}$, where $H=|S|$ and we know from the definition of Algorithm~\ref{alg:OCO for state estimation} that $t_i\in\{i\tau,\dots,(i+1)\tau-1\}$ for all $i\in\{0,\dots,H-1\}$. Additionally, we denote
\begin{align}
\label{eqn:set S prime for state estimation}
S^{\prime}=\{t_j:j\in[H],t_j\ge \tau\}=\{t_1,\dots,t_{H-1}\}.
\end{align}
One can check that $|S|=\lfloor T/\tau\rfloor$ and $|S^{\prime}|=\lfloor T/\tau\rfloor-1$.
Finally, the following notations will also be used in this proof:
\begin{equation}\label{eqn:notations used in the proof for state estimation}
\begin{split}
&R_M\triangleq\sqrt{\min\{p,n\}}\kappa_F^2\frac{1}{1-\gamma_F},\ \beta_x\triangleq2hR_y^2,\\ &l_x\triangleq 2\sqrt{h}R_y(\sqrt{\min\{p,n\}}\kappa_F^2\frac{\sqrt{h}R_y}{1-\gamma_F}+R_x),\\
&R_{\tilde{v}}\triangleq\sqrt{5\Tr(\tilde{V})\log\frac{12T}{\delta}},
\end{split}
\end{equation}
where $R_x$ and $R_y$ are defined in \eqref{eqn:R_x and R_y}.

In the following, we provide the main steps in the proof and the omitted details are included in Appendix~\ref{app:state estimation}.

{\bf Regret decomposition.} Similarly to \eqref{eqn:regret decomposition output estimation} in the proof of Theorem~\ref{thm:regret for output estimation}, we decompose the regret ${\tt R_x}(T)$ in \eqref{eqn:R_T^x} as 
\begin{align}\nonumber
{\tt R_x}(T)&=\sum_{t=0}^{T-1}f_t(M_t)-\sum_{t=0}^{T-1}\norm{x_t-\hat{x}_t^{\tt KF}}\\
&=\underbrace{\sum_{t=0}^{T-1}\big(f_t(M_t)-f_t(M^{\star})\big)}_{\text{online optimization regret}}+\underbrace{\sum_{t=0}^{T-1}\big(f_t(M^{\star})-\norm{x_t-\hat{x}_t^{\tt KF}}^2\big)}_{\text{truncation \& steady-state regret}},\label{eqn:regret decomposition state estimation}
\end{align}
where $f_t(M^{\star})=\norm{x_t-\hat{x}_t^{\tt TK}}^2$ and the comparator point $M^{\star}=[M_s^{\star}]_{s\in[h]}$ with $M_s^{\star}$ defined in \eqref{eqn:optimal M in KF}.

{\bf Online optimization regret.} Similarly to Definition~\ref{def:conditional cost}, we introduce the following auxiliary function.
\begin{definition}
\label{def:conditional cost state estimation}
Let $\{\F_t\}_{t\ge0}$ be a filtration with $\F_t\triangleq\sigma(w_0,\dots,w_{t-1},v_0,\dots,v_{t-1})$ for $t\ge1$ and $\F_0\triangleq\emptyset$. For any $N
\in\R^{nph}$ and some $h\in\BZ_{\ge0}$, define $f_{t;h}(M)=\E[f_t(M)|\F_{t-h}]$ for all $t \ge h$, where the expectation $\E[\cdot]$ is taken with respect to $\{w_t\}_{t\ge0}$ and $\{v_t\}_{t\ge0}$. 
\end{definition}
\begin{lemma}
\label{lemma:properties of f_t;h state estimation}
For any $t\ge h$, $f_{t;h}(\cdot)$ is $\alpha_0$-strongly convex. In addition, under the event $\CE_w\cap\CE_v$, $f_{t;h}(\cdot)$ is $l_x$-Lipschitz.
\end{lemma}
 
Similarly to Lemma~\ref{lemma:upper bound on f_t;k difference output estimation}, we can leverage the $\alpha_0$-strong convexity of $f_{t;h}(\cdot)$ and prove the following result.
\begin{lemma}
\label{lemma:upper bound on f_t;h difference state estimation in set S}
For any $t\ge h$, let $\varepsilon_t^s=\nabla f_t(M_t)-\nabla f_{t;h}(M_t)$ and $\tilde{\varepsilon}_t^s=\nabla\tilde{f}_t(M_t)-\nabla f_t(M_t)$. It holds that
\begin{align*}
\sum_{t_j\in S^{\prime}}\big(f_{t;h}(M_{t_j}-f_{t;h}(M^{\star}))&\le-\frac{\alpha_0}{4}\sum_{t_j\in S^{\prime
}}\norm{M_{t_j}-M^{\star}}^2+\frac{R_M^2\alpha_0}{2}+\sum_{t_j\in S^{\prime}}\frac{\eta_j}{2}\norm{\nabla\tilde{f}_{t_j}(M_{t_j})}^2\\
&\qquad\qquad-\sum_{t_j\in S^{\prime}}\langle\varepsilon_{t_j}^s,M_{t_j}-M^{\star}\rangle-\sum_{t_j\in S^{\prime}}\langle\tilde{\varepsilon}_{t_j}^s,M_{t_j}-M^{\star}\rangle.
\end{align*}
\end{lemma}

\begin{lemma}
\label{lemma:upper bound on f_t difference in set S}
Under the event $\CE_w\cap\CE_v$, it holds that 
\begin{align*}
&\sum_{t_j\in S^{\prime}}\big(f_{t_j}(M_{t_j})-f_{t_j}(M^{\star})\big)\le-\frac{\alpha_0}{4}\sum_{t_j\in S^{\prime}}\norm{M_{t_j}-M^{\star}}+\frac{R_M^2\alpha_0}{2}+\sum_{t_j\in S^{\prime}}\frac{\eta_j}{2}\norm{\nabla\tilde{f}_{t_j}(M_{t_j})}^2\\
&\quad+4(l_x+R_M\beta_x)\sum_{t_j\in S^{\prime}}\norm{M_{t_j}-M_{t_j-h}}+\sum_{t_j\in S^{\prime}}X_{t_j}(M^{\star})-\sum_{t_j\in S^{\prime}}\langle\tilde{\varepsilon}_{t_j}^s,M_{t_j}-M^{\star}\rangle,
\end{align*}
where $X_t(M^{\star})\triangleq\langle\nabla(f_t-f_{t;h})(M_{t-h}),M_{t-h}-M^{\star}\rangle+(f_t-f_{t;h})(M_{t-h})+(f_{t;h}-f_t)(M^{\star})$ for all $t\ge h$. Moreover, it holds that $\E[X_t(M^{\star})|\F_{t-h}]=0$ for all $t\ge h$.
\end{lemma}
\begin{proof}
The proof follows directly from that of Lemma~\ref{lemma:upper bound on f_t difference output estimation} by considering $f_t(\cdot)$ and $f_{t;h}(\cdot)$ in the state estimation problem and summing over $t_j\in S^{\prime}$. 
\end{proof}

As we mentioned before, we now extend the regret over the time steps in $S^{\prime}$ to the {\it expected} regret over the steps $t=\tau,\dots,T-1$.
\begin{lemma}
\label{lemma:upper bound on f_t difference state estimation}
Under the event $\CE_w\cap\CE_v$, it holds that 
\begin{align*}
&\E_{b}\Big[\sum_{t=\tau}^{T-1}\big(f_t(M_t)-f_t(M^{\star})\big)\Big]\le-\frac{\alpha_0}{4}\E_{b}\Big[\sum_{t=k}^{T-1}\norm{M_t-M^{\star}}\Big]+\frac{R_M^2\alpha_0\tau}{2}+\E_{b}\Big[\sum_{t=\tau}^{T-1}X_t(M^{\star})\Big]\\
&+\tau\E_{b}\Big[\sum_{t_j\in S^{\prime}}\frac{\eta_j}{2}\norm{\nabla\tilde{f}_{t_j}(M_{t_j})}^2\Big]+4\tau(l_x+R_M\beta_x)\E_{b}\Big[\sum_{t_j\in S^{\prime}}\norm{M_{t_j}-M_{t_j-h}}\Big]-\E_{b}\Big[\sum_{t=\tau}^{T-1}\langle\tilde{\varepsilon}_t^s,M_t-M^{\star}\rangle\Big],
\end{align*}
where $\E_{b}[\cdot]$ denotes the expectation with respect to $b_0,\dots,b_{\lfloor T/\tau\rfloor}$ sampled in line~3 of Algorithm~\ref{alg:OCO for state estimation}.
\end{lemma}

We further upper bound various terms on the right-hand side of the inequality in Lemma~\ref{lemma:upper bound on f_t difference state estimation}.
\begin{lemma}
\label{lemma:upper bound on nabla tilde f_t}
Suppose the event $\CE_w\cap\CE_v$ holds and consider any $\delta\in(0,1)$. 

\noindent(i) It holds with probability at least $1-\delta/12$ that
\begin{align*}
\E_b\Big[\sum_{t=\tau}^{T-1}X_t(M^{\star})\Big]\le\frac{32l_x^2h\log\frac{12h}{\delta}}{\alpha_0}+\frac{\alpha_0}{4}\E_b\Big[\sum_{t=0}^{T-1}\norm{M_t-M^{\star}}\Big].
\end{align*}
(ii) Further supposing the event $\CE_{\tilde{v}}$ defined in \eqref{eqn:event E_tilde v} holds, 
\begin{align*}
\E_b\Big[\sum_{t_j\in S^{\prime}}\frac{\eta_j}{2}\norm{\nabla\tilde{f}_{t_j}(M_{t_j})}^2\Big]\le\frac{(l_x+2\sqrt{h}R_yR_{\tilde{v}})^2}{\alpha_0}\big(1+\log\lfloor T/\tau\rfloor\big).
\end{align*}
(iii) It holds that 
\begin{align*}
\E_b\Big[\sum_{t_j\in S^{\prime}}\norm{M_{t_j}-M_{t_j-h}}\Big]\le\frac{l_x+2\sqrt{h}R_yR_{\tilde{v}}}{\alpha_0}\big(1+\log\lfloor T/\tau\rfloor\big).
\end{align*}

\noindent(iv) Further supposing the event $\CE_{\tilde{v}}$ holds, the following holds with probability at least $1-\delta/12$:
\begin{align*}
-\E_b\Big[\sum_{t=\tau}^{T-1}\langle\tilde{\varepsilon}_t^s,M_t-M^{\star}\rangle\Big]\le4\sqrt{h}R_yR_{\tilde{v}}R_M\sqrt{T\log\frac{12}{\delta}}.
\end{align*}
\end{lemma}

Finally, we relate the expected regret to high probability regret.
\begin{lemma}
\label{lemma:upper bound on f_t difference state estimation expected to true}
Under the event $\CE_w\cap\CE_v$, it holds with probability at least $1-\delta/12$ that 
\begin{align}
\sum_{t=\tau}^{T-1}\Big(f_t(M_t)-f_t(M^{\star})-\E_b\big[f_t(M_t)-f_t(M^{\star})\big]\Big)\le\big(8R_x^2+8hR_y^2R_M^2\big)\sqrt{2T\log\frac{24}{\delta}}.\label{eqn:relate f_t difference to f_t expected difference}
\end{align}
\end{lemma}

Recalling from Lemma~\ref{lemma:high probability noise bound} that $\P(\CE_w\cap\CE_v)\ge1-2\delta/3$ and $\P(\CE_{\tilde{v}})\ge1-\delta/12$ as argued at the beginning of Section~\ref{sec:proof for state estimation}, we can combine Lemmas~\ref{lemma:upper bound on f_t difference state estimation}-\ref{lemma:upper bound on f_t difference state estimation expected to true} and obtain from a union bound that under the event $\CE_w\cap\CE_v$, the following holds with probability at least $1-\delta/3$: 
\begin{align*}
\sum_{t=\tau}^{T-1}\big(f_t(M_t)-f_t(M^{\star})\big)=\CO\big(\tau(\log^4T)\log(T/\tau)+\sqrt{T}\log^2T\big),
\end{align*}
where we again use the facts deduced from \eqref{eqn:known constants}-\eqref{eqn:def of gamma_F and kappa_F}, \eqref{eqn:R_x and R_y} and \eqref{eqn:notations used in the proof for state estimation} that $h=\CO(\log T)$, $l_x=\CO(\log^2T)$, $R_y=\CO(\sqrt{\log T})$, $R_{\tilde{v}}=\CO(\sqrt{\log T})$, $\beta_x=\CO(\log^2T)$ and $R_M=\CO(1)$, where $\CO(\cdot)$ hides the factors stated in Theorem~\ref{thm:regret for state estimation}. In addition, we get from Lemma~\ref{lemma:properties of f_t(M)} that under the event $\CE_w\cap\CE_v$,
\begin{align*}
\sum_{t=0}^{\tau-1}\big(f_t(M_t)-f_t(M^{\star})\big)\le \tau(R_x^2+4hR_y^2R_M^2)=\CO(\tau\log^2T).
\end{align*}
Combining the above two inequalities, we obtain that under the event $\CE_w\cap\CE_v$, the following holds with probability at least $1-\delta/3$:
\begin{align}
\sum_{t=0}^{T-1}\big(f_t(M_t)-f_t(M^{\star})\big)=\CO\big(\tau(\log^4T)\log(T/\tau)+\sqrt{T}\log^2T\big),\label{eqn:online optimization regret in state estimation}
\end{align}
which completes upper bounding the online optimization regret in \eqref{eqn:regret decomposition state estimation}.

{\bf Truncation and steady-state regret.} Investigating the proofs of Lemma~\ref{lemma:upper bound on term (i) in output truncation regret}-\ref{lemma:upper bound on term (ii) in output truncation regret}, one can observe that the proofs also show that under the event $\CE_w\cap\CE_v$,
\begin{align}
\sum_{t=0}^{T-1}\big(f_t(M^{\star})-\norm{x_t-\hat{x}_t^{\tt KF}}^2\big)=\CO(\log^4T).\label{eqn:truncation regret state estimation}
\end{align}

{\bf Overall regret upper bound.} Combining \eqref{eqn:online optimization regret in state estimation}-\eqref{eqn:truncation regret state estimation} and noting that $\P(\CE_w\cap\CE_v)\ge1-2/(3\delta)$, we complete the proof of Theorem~\ref{thm:regret for state estimation}. 

\section{Extensions and Further Comparisons to Existing work}\label{sec:extensions}
In this section, we extend our algorithmic framework and regret analysis to (open-loop) unstable system and non-Gaussian noise cases.

\subsection{Unstable System with Closed-Loop Input}
While the Kalman filter can be applied to estimate the state of potentially unstable system \cite{anderson2005optimal}, the unboundedness of the state of unstable system could make the task of state estimation vacuous, and closed-loop inputs can thus be applied to make the system stable beforehand \cite{anderson2007optimal}. Existing work on learning the Kalman filter can only handle open-loop (marginally) stable system without any external inputs \cite{tsiamis2022online,umenberger2022globally} or with only open-loop bounded inputs \cite{ghai2020no,rashidinejad2020slip,qian2025model}. Note that when there are inputs $u_t$ to system~\eqref{eqn:LTI}, the state equation is given by $x_{t+1}=Ax_t+Bu_t+w_t$. In contrast with the aforementioned work, our algorithm design and regret analysis extend gracefully to open-loop unstable systems that can be stabilized by a static or dynamic output feedback controller \cite{hespanha2018linear,fatkhullin2021optimizing,tang2023analysis,duan2023optimization,zhang2026output}.\footnote{A static output feedback controller is of the form $u_t=Ky_t$ for some $K\in\R^{m\times p}$ and the closed-loop system is stable if $\rho(A-BKC)<1$ \cite{fatkhullin2021optimizing}. A dynamic output feedback controller is of the form $u_t=C_Ks_t$, where $s_t$ by design is the state of another dynamical system $s_{t+1}=A_Ks_t+B_Ky_t$, $A_K\in\R^{n\times n}$, $B_K\in\R^{n\times p}$, $C_K\in\R^{m\times n}$ and the closed-loop system is stable if an augmented system matrix depending on $A_K,B_K,C_K$ is stable \cite{tang2023analysis}.} Specifically, our current regret analysis for Theorems~\ref{thm:regret for output estimation} and \ref{thm:regret for state estimation} holds under an $\CO(\sqrt{\log T})$ bound on the system state norm $\norm{x_t}$ for all $t\in\{0,\dots,T-1\}$,  which is a consequence of the assumption that $A$ is stable. One can show that such an $\CO(\sqrt{\log T})$ bound on $\norm{x_t}$ still holds when the closed-loop system is stable under the output feedback controls described above. Finally, extending to open-loop unstable system with closed-loop input is generally not possible for the existing works that use least squares methods to learn the filter parameter, since the analysis of these methods typically relies on certain persistency excitation conditions that may not hold under closed-loop inputs \cite{tsiamis2022online}.

\subsection{Non-Gaussian Noise and Beyond}
Our current regrets in Theorems~\ref{thm:regret for output estimation} and \ref{thm:regret for state estimation} hold when the noise sequences $\{w_t\}_{t\ge0}$ and $\{v_t\}_{t\ge0}$ are i.i.d. Gaussian. To achieve the regret results in Theorems~\ref{thm:regret for output estimation} and \ref{thm:regret for state estimation}, we leverage two properties of the Gaussian noise. First, we prove in Lemma~\ref{lemma:high probability noise bound} that $\norm{w_t}$ and $\norm{v_t}$ are bounded by some $\log T$-factor (with high probability), which further leads to the $\log T$-factor bound on $\norm{x_t}$ for all $t\in\{0,\dots,T-1\}$. Second, leveraging the lower bound on the noise covariance $\lambda_{\min}(V)\ge\alpha_0>0$, we show, e.g., in Lemma~\ref{lemma:hessian of f_t:k(N)} that the output estimation error $f_t(\cdot)$ is conditionally strongly-convex. Hence, if other types of noise distributions are considered, as long as the stochastic noise $\{w_t\}_{t\ge0}$ and $\{v_t\}_{t\ge0}$ posses the above two properties (examples include sub-Gaussian noise or bounded stochastic noise \cite{vershynin2018high}), the regret results in Theorems~\ref{thm:regret for output estimation} and \ref{thm:regret for state estimation} continue to hold (with a possibly inflation in the power of the $\log T$-factor).

Besides stochastic noise, one may also consider potentially adversarially generated non-stochastic noise $\{w_t\}_{t\ge0}$ and $\{v_t\}_{t\ge0}$ with bounded norm \cite{simchowitz2020improper,hazan2025introduction} (e.g., a constant or a $\log T$ upper bound). Since the non-stochastic noise does not have a covariance matrix, we lose the aforementioned conditional strong convexity of the estimation error cost functions. However, the cost functions in both the output and state estimation problems are always convex (as they are quadratic), one can use this convexity to show $\tilde{\CO}(\sqrt{T})$ and $\tilde{\CO}(\sqrt{\tau T})$ regret for Algorithms~\ref{alg:OCO for output estimation} and~\ref{alg:OCO for state estimation}, when applied to solving learning Kalman filtering for output and state estimation, respectively. Note that the regret performance degradation is typical when moving from stochastic noise with positive definite covariance to non-stochastic noise \cite{simchowitz2020improper,ghai2020no}. We leave a formal treatment of non-stochastic noise as future work.

Finally, we remark that when considering non-Gaussian or non-stochastic noise, the Kalman filter in hindsight loses its global optimality in terms of minimizing the MSEE but is optimal within the class of linear state estimators \cite[Chapter~5]{anderson2005optimal}. In addition, the methods based on least squares proposed in \cite{ghai2020no,tsiamis2020sample,rashidinejad2020slip} cannot directly handle non-stochastic noise, as the analysis there relies on more subtle properties inherited from the stochastic noise. Existing works that can handle non-stochastic noise only focus on the problem of learning output estimation \cite{hazan2017learning,kozdoba2019line}.

\section{Numerical Experiments}\label{sec:simulations}
In this section, we validate the theoretical results provided in Theorems~\ref{thm:regret for output estimation} and~\ref{thm:regret for state estimation} for Algorithms~\ref{alg:OCO for output estimation} and~\ref{alg:OCO for state estimation}, respectively.
\begin{figure}[htbp]
    \centering
    \subfloat[a][${\tt R_y}(T)$ v.s. $T$]{\includegraphics[width=0.47\linewidth]{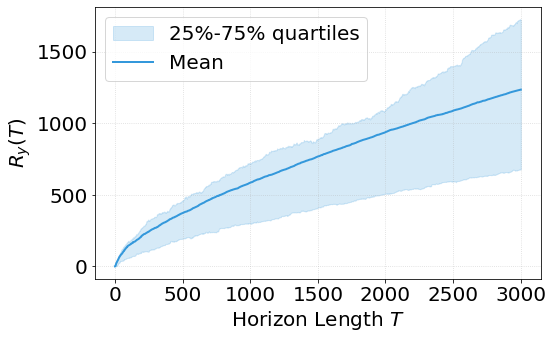}}
    \subfloat[b][${\tt R_y}(T)/\log^4T$ v.s. $T$]{\includegraphics[width=0.47\linewidth]{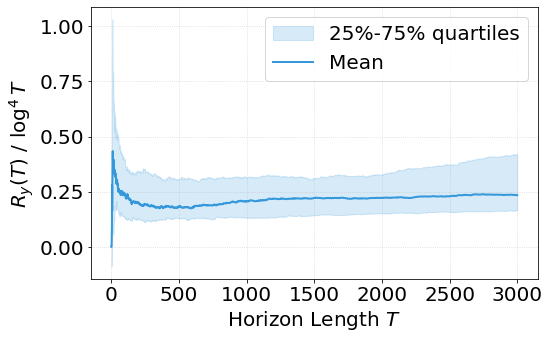}}
    \caption{The regret of Algorithm~\ref{alg:OCO for output estimation} for learning Kalman filtering for output estimation.}
    \label{fig:output estimation}
\end{figure}

\subsection{Output Estimation}
To validate the regret bound provided in Theorem~\ref{thm:regret for output estimation}, we use randomly generated instances of system~\eqref{eqn:LTI}. Specifically, we generate random system matrices $A\in\R^{4\times 4}$ and $C\in\R^{2\times 4}$, where the entries of these matrices are drawn uniformly from $(0,1)$ and the matrix $A$ is scaled to have $\rho(A)=0.9$. The covariance matrices of the Gaussian noise $\{w_t\}_{t\ge0}$ and $\{v_t\}_{t\ge0}$ are set to be $W=0.25I_4$ and $V=0.25I_2$, respectively. We test the performance of Algorithm~\ref{alg:OCO for output estimation} when applied to solve the above problem instances for the problem of learning Kalman filtering for output estimation, where we set the step sizes $\eta_t=\frac{1}{(\log^2T)t}$ for all $t\ge1$ with $\eta_0=0$ and set the truncated filter length $h=\lfloor\log T\rfloor$ according to Theorem~\ref{thm:regret for output estimation}. We plot curves pertaining to the regret ${\tt R_y}(T)$ defined in \eqref{eqn:R_T^y} of Algorithm~\ref{alg:OCO for output estimation} in Fig.~\ref{fig:output estimation}, where we vary $T$ from $1$ to $3000$ and conduct $50$ independent experiments. We see from Fig.~\ref{fig:output estimation}(a)-(b) that ${\tt R_y}(T)$ scales as $\log^4 T$, as the curve ${\tt R_y}(T)/\log^4T$ in Fig.~\ref{fig:output estimation}(b) becomes flat as $T$ increases, which matches our regret bound ${\tt R_y}(T)=\CO(\log^4T)$ provided in Theorem~\ref{thm:regret for output estimation} and shows that the regret bound ${\tt R_y}(T)=\CO(\log^4T)$ is also tight when Algorithm~\ref{alg:OCO for output estimation} is applied to solve the problem instances constructed above.

\subsection{State Estimation}
To validate the regret bound provided in Theorem~\ref{thm:regret for state estimation}, we use the example of longitudinal flight control of Boeing~747 with linearized dynamics given by (see, e.g., \cite{ishihara1992design,kargin2022thompson})
\begin{align*}
A = \begin{bmatrix}
    0.99 & 0.03 & -0.02 & -0.32 \\
    0.01 & 0.47 & 4.7 & 0.0 \\
    0.02 & -0.06 & 0.4 & 0.0 \\
    0.01 & -0.04 & 0.72 & 0.99
\end{bmatrix},\ B = \begin{bmatrix}
    0.01 & 0.99 \\
    -3.44 & 1.66 \\
    -0.83 & 0.44 \\
    -0.47 & 0.25
\end{bmatrix},\ C=
\begin{bmatrix}
1 & 0 & 0 & 0\\
0 & 1 & 0 & 0\\
0 & 0 & 1 & 0
\end{bmatrix}.
\end{align*}
The system state $x_t$ is represented by a four-dimensional vector. Its components correspond to the aircraft's velocity along the body axis, the velocity perpendicular to the body axis, the pitch angle (defined as the angle between the body axis and the horizontal), and the aircraft's angular velocity, respectively. Furthermore, the system is driven by a two-dimensional control input consisting of the elevator angle and the engine thrust. We set the covariance matrices of the Gaussian noise $\{w_t\}_{t\ge0}$ and $\{v_t\}_{t\ge0}$ to be $W=0.0025I_4$ and $V=0.0025I_3$, respectively.

\begin{figure}[htbp]
    \centering
    \subfloat[a][${\tt R_x}(T)$ v.s. $T$]{\includegraphics[width=0.47\linewidth]{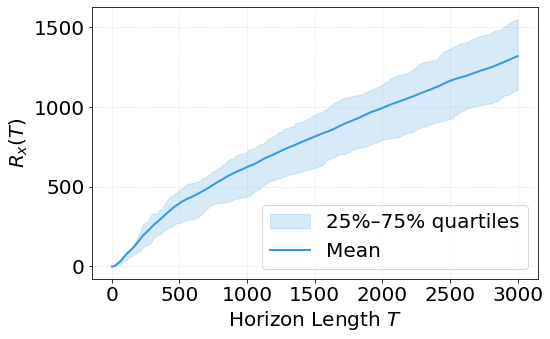}}
    \subfloat[b][${\tt R_x}(T)/\sqrt{T}$ v.s. $T$]{\includegraphics[width=0.47\linewidth]{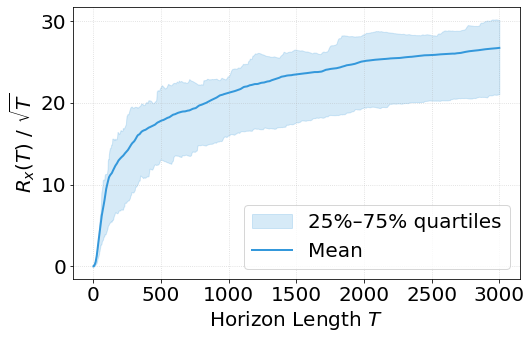}}
    
    \caption{The regret of Algorithm~\ref{alg:OCO for state estimation} for learning Kalman filtering for state estimation.}
    \label{fig:state estimation}
\end{figure}

Given the above system matrices $A,B,C$, we find an output feedback control $u_t=KCy_t$ such that the closed-loop system $A+BKC$ is stable, and apply Algorithm~\ref{alg:OCO for state estimation} to the closed-loop system for the problem of learning Kalman filtering for state estimation (see also our discussions in Section~\ref{sec:extensions}). We set the step sizes $\eta_j=\frac{1}{j}$ for all $j\ge1$ with $\eta_0=0$ and set the truncated filter length $h=\lfloor\log T\rfloor$. Additionally, we set $\tau=\lfloor\sqrt{T}\rfloor$ in Algorithm~\ref{alg:OCO for state estimation} (i.e., the number of queries for the informative measurement $\tilde{x}_t$ in \eqref{eqn:informative measurement} is bounded by   $\lfloor T/\tau\rfloor=\lfloor\sqrt{T}\rfloor$). In Fig~\ref{fig:state estimation}, we plot curves pertaining to the regret ${\tt R_x}(T)$ defined in \eqref{eqn:R_T^x} of Algorithm~\ref{alg:OCO for state estimation}, where we vary $T$ from $1$ to $3000$ and conduct $50$ independent experiments. Noting that the curve ${\tt R_x}(T)/\sqrt{T}$ in Fig.~\ref{fig:state estimation}(b) yields a $\log T$ tendency with respect to $T$, the results in Fig.~\ref{fig:state estimation} again match our regret bound ${\tt R_x}(T)=\tilde{\CO}(\sqrt{T})$ under the choice of $\tau=\lfloor\sqrt{T}\rfloor$ (with $\tilde{\CO}(\cdot)$ hiding $\log T$-factors) provided in Theorem~\ref{thm:regret for state estimation} and show that the $\sqrt{T}$-factor in the regret bound is also tight when Algorithm~\ref{alg:OCO for state estimation} is applied to the above Boeing~747 example.

\section{Conclusion}
We provided a unified online learning algorithmic framework that solves both the problems of learning Kalman filtering for output estimation and state estimation in unknown linear dynamical systems. While achieving comparable regret to existing work on learning Kalman filtering for output estimation, our proposed algorithm and regret analysis are conceptually simpler and more constructive in that they rely on certain generic properties of the estimation error cost function such as strongly-convexity that we have shown to hold. Our main contribution is to solve the state estimation scenario, which is also an open problem in the literature. While we showed that achieving a sublinear regret is impossible in general, by introducing an addition query access to more informative measurements of the system state, we showed that a sublinear regret is achievable, successfully balancing queries and performance. Our algorithm and result also shed light on online learning problems with limited or costly observations or feedback. Future work includes moving beyond linear systems to nonlinear systems and considering learning extended Kalman filtering or more general output/state estimators.

\bibliography{main}
\bibliographystyle{unsrt}

\appendix
\counterwithin{lemma}{section}
\counterwithin{theorem}{section}
\counterwithin{proposition}{section}
\counterwithin{corollary}{section}
\counterwithin{definition}{section}
\counterwithin{equation}{section}
\onecolumn

\section*{Appendix}

{
\hypersetup{linkcolor=blue}
  \tableofcontents
}

\section{Organization of Appendix and Notations}

In Appendix~\ref{app:preliminary proofs}, we provide the proofs of several preliminary results stated in Sections~\ref{sec:problem formulation and preliminaries}-\ref{sec:cost functions}, including properties of the truncated steady-state Kalman filter $\hat{x}_t^{\tt TK}$ introduced in \eqref{eqn:KF truncated} and properties of estimation error cost function $f_t(\cdot)$. In Appendix~\ref{app:omitted proofs for output estimation}, we provide proofs pertaining to the regret bound of Algorithm~\ref{alg:OCO for output estimation} when applied to the problem of learning Kalman filtering for output estimation. In Appendix~\ref{app:state estimation}, we provide the proof of the regret lower bounds for learning Kalman filtering for state estimation and the proofs pertaining to the regret bounds of Algorithm~\ref{alg:OCO for state estimation} when applied to solve the problem of learning Kalman filtering for state estimation. In Appendix~\ref{app:aux results}, we prove some technical lemmas that are used in the proofs of our main results.

\begin{table}[h]
    \centering
 \caption{Notations}
    \begin{tabular}{c | c}
        \hline
        {\bf Notation} & {\bf Definition}\\
        \hline
        $\K_N,\K_M$ & Sets of candidate filter parameters for output and state estimation, resp. \\
        $R_N,R_M$ & Upper bounds on the norm of filter parameters in $\K_N$ and $\K_M$, resp.\\
        $\beta_y,l_y$ & Smooth and Lipschitz parameters of $f_t(N)$ in output estimation, resp.\\
        $R_x,R_y$ &  Upper bounds on the norm of state and output of system~\eqref{eqn:LTI}, resp.\\
        $R_{\tilde{v}}$ & Upper bound on the noise in the informative measurements given by \eqref{eqn:informative measurement}\\
        $\alpha_0,\alpha_1,\psi,\bar{\sigma}$ & Bounds on system parameters given in \eqref{eqn:known constants}\\
        $N^{\star}=[N^{\star}_s]_{s\in[h]}$ & The optimal filter parameter of $\hat{y}_t^{\tt TK}$ given in~\eqref{eqn:output KF truncated}   \\
        $M^{\star}=[M^{\star}_s]_{s\in[h]}$ & The optimal filter parameter of $\hat{x}_t^{\tt TK}$ given in~\eqref{eqn:KF truncated}\\
        $\kappa_A,\gamma_A$ & Parameters pertaining to matrix $A$ described in Assumption~\ref{ass:stable A matrix}\\
        $\kappa_F,\gamma_F$ & Parameters pertaining to the Kalman filter given in \eqref{eqn:def of gamma_F and kappa_F}\\
        $\hat{x}_t^{\tt SK},\hat{y}_t^{\tt SK}$ & Steady-state Kalman filter given in~\eqref{eqn:steady-state KF} and~\eqref{eqn:output steady-state KF}, resp.\\
        $Y_{t-1:t-h}$ & Matrix form of past output history given in~\eqref{eqn:Y_t-h:t-1}\\
        \hline
    \end{tabular}
    \label{tab:my_table}
\end{table}

\section{Proofs of Preliminary Results}\label{app:preliminary proofs}
\subsection{Proof of Lemma~\ref{lemma:truncated KF}}
Recalling that $M_s^{\star}=(A-LC)^{s-1}L\in\R^{n\times p}$ for all $s\in[h]$, we have
\begin{align*}
\norm{M_s^{\star}}&=\norm{(A-LC)^{s-1}L}\\
&\le\norm{(A-LC)^{s-1}}\norm{L}\\
&\le\kappa_F^2\gamma_F^{s-1},
\end{align*}
where the last relation uses Lemma~\ref{lemma:strongly stable of L}. Noting the basic inequality $\norm{P}_F^2\le\min\{m,n\}\norm{P}^2$ for any $P\in\R^{m\times n}$, we have $\norm{M_s^{\star}}_F\le\sqrt{\min\{p,n\}}\kappa_F^2\gamma_F^{s-1}$ for all $s\in[h]$. Moreover, we have
\begin{align*}
\norm{M^{\star}}_F&=\norm{[M_s^{\star}]_{s\in[h]}}_F\\
&\le\big(\sum_{s=1}^h\norm{M_s^{\star}}_F^2\big)^{1/2}\\
&\le\big(\min\{p,n\}\kappa_F^4\sum_{s=1}^h\gamma_F^{2(s-1)}\big)^{1/2}\\
&\le\sqrt{\min\{p,n\}}\frac{\kappa_F^2}{1-\gamma_F
}.
\end{align*}
$\hfill\blacksquare$

\subsection{Proof of Lemma~\ref{lemma:high probability noise bound}}
Recall that the process noise satisfies that $w_t\overset{i.i.d.}{\sim}\K_N(0,W)$. For any $t\in\{0,\dots,T-1\}$, it then follows from Lemma~\ref{lemma:gaussian concentration bound} in Appendix~\ref{app:aux results} that $\norm{w_t}\le\sqrt{5\Tr(W)\log\frac{3T}{\delta}}$ with probability at least $1-\delta/(3T)$. By a union bound over all $t\in\{0,\dots,T-1\}$, we get that $\P(\CE_w)\ge1-\delta/3$. Similarly, we can show $\P(\CE_v)\ge1-\delta/3$. Further taking a union bound yields that $\P(\CE_w\cap\CE_v)\ge1-(2\delta)/3$.$\hfill\blacksquare$

\subsection{Proof of Lemma~\ref{lemma:upper bound on state and output}}
Unrolling \eqref{eqn:LTI} for $t\in\{0,\dots,T-1\}$ and recalling the assumption that $x_0=0$, we get
\begin{align}
x_t=\sum_{s=0}^{t-1}A^{t-(s+1)}w_s.
\end{align}
Suppose the event $\CE_w\cap\CE_v$ holds. It follows from \eqref{eqn:probabilistic event} and Assumption~\ref{ass:stable A matrix} that
\begin{align*}
\norm{x_t}&=\norm{\sum_{s=0}^{t-1}A^{t-(s+1)}w_s}\\
&\le\sum_{s=0}^{t-1}\norm{A^{t-(s+1)}}\norm{w_s}\\
&\le\sqrt{5\Tr(W)\log\frac{3T}{\delta}}\sum_{s=0}^{t-1}\kappa_A\gamma_A^{t-(s+1)}\\
&\le\sqrt{5\Tr(W)\log\frac{3T}{\delta}}\cdot\frac{\kappa_A}{1-\gamma_A}.
\end{align*}
In addition, we have
\begin{align*}
\norm{y_t}&=\norm{\big(C\sum_{s=0}^{t-1}A^{t-(s+1)}w_s\big)+v_t}\\
&\le\norm{C}\sqrt{5\Tr(W)\log\frac{3T}{\delta}}\cdot\frac{\kappa_A}{1-\gamma_A}+\sqrt{5\Tr(V)\log\frac{3T}{\delta}}.
\end{align*}
Finally, recall from~\eqref{eqn:Y_t-h:t-1} that $Y_{t-1:t-h}=I_p\otimes[y_{t-s}^{\top}]_{s\in[h]}^{\top}$, which yields
\begin{align*}
Y_{t-1:t-h}^{\top}Y_{t-1:t-h}=I_p\otimes([y_{t-s}^{\top}]_{s\in[h]}[y_{t-s}^{\top}]_{s\in[h]}^{\top}).
\end{align*}
We then obtain that for any $t\in\{h,\dots,T-1\}$,
\begin{align*}
\norm{Y_{t-1:t-h}}&=\sqrt{\lambda_{\max}(Y_{t-1:t-h}Y_{t-1:t-h}^{\top})}\\
&=\sqrt{[y_{t-s}^{\top}]_{s\in[h]}[y_{t-s}^{\top}]_{s\in[h]}^{\top}}\\
&=\sqrt{\sum_{s=1}^h\norm{y_{t-s}}^2}\\
&\le\sqrt{h}R_y,
\end{align*}
which completes the proof of the lemma.
$\hfill\blacksquare$

\subsection{Proof of Lemma~\ref{lemma:properties of f_t(M)}}
First, noting that $\norm{M}_F=\norm{\bm{vec}(M)}$, the upper bound on $\norm{\bm{vec}(M)}$ follows directly from $M\in\CM$ and \eqref{eqn:candidate truncated KF}.

We then prove~\eqref{eqn:upperbound on the gradient of f_t(M)}. Recalling the expression of $\nabla f_t(\bm{vec}(M))$ given by~\eqref{eqn:gradient of f_t(M_t)}, we have
\begin{align*}
\norm{\nabla f_t(\bm{vec}(M))}&\le2\norm{Y_{t-1:t-h}}\norm{Y_{t-1:t-h}\bm{vec}(M)-x_t}\\
&\le2\norm{Y_{t-1:t-h}}\big(\norm{Y_{t-1:t-h}}\norm{\bm{vec}(M)}+\norm{x_t}\big).
\end{align*}
It then follows from Lemma~\ref{lemma:upper bound on state and output} and \eqref{eqn:candidate truncated KF} that 
\begin{align*}
\norm{\nabla f_t(\bm{vec}(M))}&\le2\sqrt{h}R_y(\sqrt{\min\{p,n\}}\kappa_F^2\frac{\sqrt{h}R_y}{1-\gamma_F}+R_x).
\end{align*}

Next, we prove~\eqref{eqn:upper bound on the hessian of f_t(M)}. Recalling the expression of $\nabla^2f_t(\bm{vec}(M))$ given by~\eqref{eqn:Hessian of f_t(M_t)}, we have
\begin{align*}
\norm{\nabla^2f_t(\bm{vec}(M))}&=2\norm{Y_{t-1:t-h}^{\top}Y_{t-1:t-h}}\\
&\le2\norm{Y_{t-1:t-h}}^2\\
&\le 2hR_y^2,
\end{align*}
where we used the fact that $\norm{Y_{t-1:t-h}^{\top}}=\norm{Y_{t-1:t-h}}$ and Lemma~\ref{lemma:upper bound on state and output}.

Finally, we prove \eqref{eqn:upper bound on f_t(M)}.
Recalling that $f_t(M)=\norm{x_t-\hat{x}_t(M)}^2$ with $\hat{x}_t(M)=\sum_{s=1}^hM_{s}y_{t-s}$ for any $M\in\CM$ and $Y_{t-1:t-h}$ defined in \eqref{eqn:Y_t-h:t-1}, we have
\begin{align*}
f_t(M)&\le2\norm{x_t}^2+2\norm{\sum_{s=1}^hM_{s}y_{t-s}}^2\\
&\le 2R_x^2+2\norm{Y_{t-1:t-h}}^2\norm{\bm{vec}(M)}^2\\
&\le2R_x^2+2hR_y^2\min\{p,n\}\kappa_F^4\frac{1}{(1-\gamma_F)^2},
\end{align*}
where the last relation follows from the upper bounds on $\norm{Y_{t-1:t-h}}$ and $\norm{\bm{vec}(M)}$ that we have shown above. $\hfill\blacksquare$

\section{Omitted Proofs and Discussions in Section~\ref{sec:output estimation} for Output Estimation}\label{app:omitted proofs for output estimation}
\subsection{Proof of Lemma~\ref{lemma:hessian of f_t:k(N)}}
By Definition~\ref{def:conditional cost}, for any $N\in\R^{p\times(ph)}$, we have $\nabla^2f_{t;h}(N)=\nabla^2\E[f_t(N)|\F_{t-k}]=\E[\nabla^2f_t(N)|\F_{t-k}]$, where we swap $\E$ and $\nabla^2$ since $\nabla^2f_t(N)$ is continuous in $N$. Moreover, since $k=h$, we get from Lemma~\ref{lemma:properties of f_t(N)} that for any $t \ge h$,
\begin{align*}
&\E[\nabla^2f_t(N)|\F_{t-h}]\\
=&2\E[Y_{t-1:t-h}^{\top}Y_{t-1:t-h}|\F_{t-h}]\\
=&2I_p\otimes \E\big[[y_{t-s}^{\top}]_{s\in[h]}^{\top}[y_{t-s}^{\top}]_{s\in[h]}|\F_{t-h}\big]\\
=&I_p\otimes\E\Big[\Big(\begin{bmatrix}Cx_{t-1}\\\vdots\\ Cx_{t-h}\end{bmatrix}+\begin{bmatrix}v_{t-1}\\\vdots\\ v_{t-h}\end{bmatrix}\Big)\Big(\begin{bmatrix}Cx_{t-1}\\\vdots\\ Cx_{t-h}\end{bmatrix}+\begin{bmatrix}v_{t-1}\\\vdots\\ v_{t-h}\end{bmatrix}\Big)^{\top}\Big|\F_{t-h}\Big]\\
\overset{(a)}{\succeq}&I_{p}\otimes\E\Big[\begin{bmatrix}
v_{t-1}\\\vdots\\ v_{t-h}    
\end{bmatrix}\begin{bmatrix}v_{t-1}^{\top} & \cdots &v_{t-h}^{\top}\end{bmatrix}\Big]+I_p\otimes\E\Big[\begin{bmatrix}Cx_{t-1}\\\vdots\\ Cx_{t-h}\end{bmatrix}\begin{bmatrix}x_{t-1}^{\top}C^{\top} & \cdots x_{t-h}^{\top}C^{\top}\end{bmatrix}\Big|\F_{t-h}\Big]\\
\succeq&\lambda_{min}(V)I_{p^2h}.
\end{align*}
where $(a)$ follows from the fact that the measurement noise $v_k$ is independent of the state $x_k$ for all $k\ge0$, and the definition of the filtration $\F_{t-h}=\sigma(w_0,\dots,w_{t-h-1},v_0,\dots,v_{t-h-1})$. This proves that $f_{t:h}(\cdot)$ is $\alpha_0$-strongly convex for all $t\ge h$. 

Next, we recall from Lemma~\ref{lemma:properties of f_t(M)} that $f_t(\cdot)$ is $l_y$-Lipschitz under the event $\CE_w\cap\CE_v$. Supposing $\CE_w\cap\CE_v$ holds, we then get from the definition of $f_{t;h}(\cdot)$ that for any given  $N,N^{\prime}\in\K_N$, 
\begin{align*}
|f_{t;h}(N)-f_{t;h}(N^{\prime})|&=\big|\E[f_{t}(N)-f_t(N^{\prime})\big|\F_{t-h}]\big|\\
&\le\E\big[|f_t(N)-f_t(N^{\prime})|\big|\F_{t-h}\big]\\
&\le l_y\E\big[\rVert N-N^{\prime}\rVert\big|\F_{t-h}\big]\\
&\le l_y\lVert N-N^{\prime}\rVert,
\end{align*}
which shows that $f_{t;h}(\cdot)$ is $l_y$-Lipschitzness.
$\hfill\blacksquare$

\subsection{Proof of Lemma~\ref{lemma:upper bound on f_t;k difference output estimation}}
Consider any $t \ge h$. By the $\alpha_0$-strong convexity of $f_{t;h}(\cdot)$ shown in Lemma~\ref{lemma:hessian of f_t:k(N)}, we know from Definition~\ref{def:strongly convex} that
\begin{align}
f_{t;h}(N_t)-f_{t;h}(N^{\star})\le\nabla f_{t;h}(N_t)^{\top}(N_t-N^{\star})-\frac{\alpha_0}{2}\norm{N_t-N^{\star}}^2.\label{eqn:strongly convex ineq output estimation}
\end{align}
Next, we aim to provide an upper bound on $f_{t;h}(N_t)^{\top}(N_t-N^{\star})$. We first obtain that 
\begin{align}\nonumber
\norm{N_{t+1}-N^{\star}}^2&\le\norm{N_t-N^{\star}-\eta_t\nabla f_t(N_t)}\\
&=\norm{N_t-N^{\star}}^2+\eta_t^2\norm{\nabla f_t(N_t)}^2-2\eta_t\nabla f_t(N_t)^{\top}(N_t-N^{\star}),\label{eqn:output estimation projection shrink}
\end{align}
where the inequality follows from the fact that $N_{t+1}=\Pi_{\K_N}(N_{t}-\eta_t\nabla f_t(N_t))$ and \cite[Proposition~2.2]{bansal2019potential}. It then follows from \eqref{eqn:output estimation projection shrink} that 
\begin{align*}
2\nabla f_t(N_t)^{\top}(N_t-N^{\star})\le\frac{\norm{N_t-N^{\star}}^2-\norm{N_{t+1}-N^{\star}}^2}{\eta_t}+\eta_t\norm{\nabla f_t(N_t)}^2.
\end{align*}
Noting that $\nabla f_t(N_t)=\nabla f_{t;h}(N_t)+\varepsilon_t^s$, we have
\begin{align*}
\langle \nabla f_t(N_t),N_t-N^{\star}\rangle=\langle\nabla f_{t;h}(N_t),N_t-N^{\star}\rangle+\langle\varepsilon_t^s,N_t-N^{\star}\rangle.
\end{align*}
Combining the above two relations, we obtain
\begin{align}
\nabla f_{t;h}(N_t)^{\top}(N_t-N^{\star})\le\frac{\norm{N_t-N^{\star}}^2-\norm{N_{t+1}-N^{\star}}^2}{2\eta_t}+\frac{\eta_t}{2}\norm{\nabla f_t(N_t)}^2-\langle\varepsilon_t^s,N_t-N^{\star}\rangle.\label{eqn:upper bound on nabla f_t;h dot N_t-N*}
\end{align}
Going back to \eqref{eqn:strongly convex ineq output estimation} and summing over $t\in\{k,\dots,T-1\}$, we get 
\begin{align*}
&\sum_{t=h}^{T-1}\big(f_{t;h}(N_t)-f_{t;h}(N^{\star})\big)\\
\le&\sum_{t=h}^{T-1}\big(\nabla f_{t;h}(N_t)^{\top}(N_t-N^{\star})-\frac{\alpha_0}{2}\norm{N_{t}-N^{\star}}^2\big)\\
\le&\frac{1}{2}\sum_{t=h}^{T-1}\Big(\frac{1}{\eta_{t+1}}-\frac{1}{\eta_t}-\alpha_0\Big)\norm{N_t-N^{\star}}^2+\frac{\norm{N_h-N^{\star}}^2}{2\eta_h}+\sum_{t=h}^{T-1}\frac{\eta_t}{2}\norm{\nabla f_t(N_t)}^2
-\sum_{t=h}^{T-1}\langle\varepsilon_t^s,N_t-N^{\star}\rangle\\
\le&-\frac{\alpha_0\log^{2}T}{4}\sum_{t=h}^{T-1}\norm{N_t-N^{\star}}^2+\frac{R_N^2\alpha_0h}{2}+\sum_{t=h}^{T-1}\frac{\eta_t}{2}\norm{\nabla f_t(N_t)}^2-\sum_{t=h}^{T-1}\langle\varepsilon_t^s,N_t-N^{\star}\rangle,
\end{align*}
where the last inequality follows from the choice of the step size $\eta_t=\frac{2}{(\alpha_0\log^{2}T)t}$ for all $t \ge h$ and the upper bound on the norm of any $N\in\K_N$ shown by Lemma~\ref{lemma:properties of f_t(N)}.$\hfill\blacksquare$

\subsection{Proof of Lemma~\ref{lemma:upper bound on f_t difference output estimation}}
Throughout this proof, we assume that the event $\CE_w\cap\CE_v$ holds. First, relating $f_t(\cdot)$ to $f_{t;h}(\cdot)$, we know from the upper bound $\norm{\nabla f_t(N_t)}\le l_y$ from Lemma~\ref{lemma:properties of f_t(N)} that
\begin{align*}
\sum_{t=h}^{T-1}\big(f_t(N_t)-f_t(N^{\star})\big)&=\sum_{t=h}^{T-1}\big(f_{t;h}(N_t)-f_{t;h}(N^{\star})\big)+\sum_{t=h}^{T-1}\big((f_t-f_{t;h})(N_t)+(f_{t;h}-f_t)(N^{\star})\big),
\end{align*}
where we write $(f_t-f_{t;h})(N)=f_t(N_t)-f_{t;h}(N)$ for some $N$ to simplify the notations. We then get from Lemma~\ref{lemma:upper bound on f_t;k difference output estimation} that 
\begin{align}\nonumber
\sum_{t=h}^{T-1}\big(f_t(N_t)-f_t(N^{\star})\big)&\le-\frac{\alpha_0\log^{2}T}{4}\sum_{t=h}^{T-1}\norm{N_t-N^{\star}}^2+\frac{R_N^2\alpha_0h}{2}+\sum_{t=h}^{T-1}\frac{\eta_t}{2}\norm{\nabla f_t(N_t)}^2\\
&\quad+\underbrace{\sum_{t=h}^{T-1}\langle\varepsilon_t^s,N^{\star}-N_{t}\rangle}_{(i)}+\underbrace{\sum_{t=h}^{T-1}\big((f_t-f_{t;h})(N_t)+(f_{t;h}-f_t)(N^{\star})\big)}_{(ii)}.\label{eqn:upper bound on f_t difference output estimation}
\end{align}
In the following, we will provide upper bounds on $(i)$ and $(ii)$ on the right-hand side of \eqref{eqn:upper bound on f_t difference output estimation}.

{\bf Upper bound on term~$(i)$ in \eqref{eqn:upper bound on f_t difference output estimation}.} We decompose the term $\langle\varepsilon_t^s,N^{\star}-N_t\rangle$ to get a term with zero expectation conditioned on the filtration $\F_t=\sigma(w_0,\dots,w_{t-1},v_0,\dots,v_{t-1})$, where recall $\varepsilon_t^s=\nabla f_t(N_t)-\nabla f_{t;h}(N_t)$. We have 
\begin{align}\nonumber
\langle\varepsilon_t^s,N_t-N^{\star}\rangle&=\langle\nabla(f_t-f_{t;h})(N_{t-h}),N_{t-h}-N^{\star}\rangle+\langle\nabla(f_t-f_{t;h})(N_{t-h}),N_t-N_{t-h}\rangle\\
&\quad+\langle\nabla(f_t-f_{t;h})(N_t)-\nabla(f_t-f_{t;h})(N_{t-h}),N_t-N^{\star}\rangle.\label{eqn:decompose gradient error cross term output estimation}
\end{align}
Recalling the definition of the filtration $\F_{t-h}=\sigma(w_0,\dots,w_{t-h-1},v_{0},\dots,v_{t-h-1})$, we see from Algorithm~\ref{alg:OCO for output estimation} that $N_{t-h}$ is $\F_{t-h}$-measurable. It follows that 
\begin{align*}
\E\big[\langle\nabla (f_t-f_{t;h})(N_{t-h}),N_{t-h}-N^{\star}\rangle|\F_{t-k}\big]=\langle\E[\nabla(f_t-f_{t;h})(N_{t-h})|\F_{t-h}],N_{t-h}-N^{\star}\rangle=0,
\end{align*}
where we also recall the definition $f_{t;h}(N_{t-k})=\E[f_t(N_{t-h})|\F_{t-k}]$. We now upper bound the remaining two terms on the right-hand side of \eqref{eqn:decompose gradient error cross term output estimation}. We have
\begin{align*}
&\sum_{t=h}^{T-1}\langle\nabla(f_t-f_{t;h})(N_{t-h}),N_t-N_{t-k}\rangle+\langle\nabla(f_t-f_{t;h})(N_t)-\nabla(f_t-f_{t;h})(N_{t-h}),N_t-N^{\star}\rangle\\
\le&\sum_{t=h}^{T-1}\Big(\norm{\nabla(f_t-f_{t;h})(N_{t-h})}\norm{N_t-N_{t-h}}\\
&\qquad+\big(\norm{\nabla f_t(N_t)-\nabla f_{t}(N_{t-h})}+\norm{\nabla f_{t;h}(N_{t-h})-\nabla f_{t;h}(N_t)}\big)\norm{N_t-N^{\star}}\Big)\\
\overset{(a)}{\le}&\sum_{t=h}^{T-1}\Big(2l_y\norm{N_t-N_{t-k}}+2\beta_y\norm{N_{t}-N_{t-h}}\norm{N_t-N^{\star}}\Big)\\
\overset{(b)}{\le}&2(l_y+2R_N\beta_y)\sum_{t=h}^{T-1}\norm{N_t-N_{t-h}},
\end{align*}
where to obtain $(a)$, we use the upper bound on $\norm{\nabla f_t(N_{t-h})}$ shown by Lemma~\ref{lemma:properties of f_t(N)}, which also implies via the definition $f_{t;h}(N_{t-h})=\E[f_t(N_{t-h})|\F_{t-h}]$ that $\nabla f_{t;h}(N_{t-k})=\E[\nabla f_t(N_{t-k})|\F_{t-k}]$ and thus $\norm{\nabla f_{t;h}(N_{t-h})}\le l_y$; to obtain $(b)$, we use the $\beta_y$-smoothness of $f_{t;h}(\cdot)$, which is a direct consequence of the $\beta_y$-smoothness of $f_t(\cdot)$ shown by Lemma~\ref{lemma:properties of f_t(N)}. It remains to upper bound
\begin{align}\nonumber
\sum_{t=h}^{T-1}\norm{N_t-N_{t-h}}&\le\sum_{t=h}^{T-1}\sum_{i=0}^{h-1}\norm{N_{t-i}-N_{t-i-1}}\\\nonumber
&\overset{(a)}{\le}\sum_{t=h}^{T-1}\sum_{i=0}^{h-1}\norm{\eta_{t-i-1}\nabla f_t(N_{t-i-1})}\\\nonumber
&\overset{(b)}{\le} l_y\sum_{t=h}^{T-1}\sum_{i=0}^{h-1}\eta_{t-i-1}\\
&\overset{(c)}{\le} \frac{2l_yh}{\alpha_0\log^{2}T}\sum_{t=1}^{T-1}\frac{1}{t}\le\frac{2l_yh}{\alpha_0\log^{2}T}(1+\log T),\label{eqn:difference between N_t and N_t-k output estimation}
\end{align}
where $(a)$ follows from line~5 of Algorithm~\ref{alg:OCO for output estimation} and \cite[Proposition~2.2]{bansal2019potential}, $(b)$ uses Lemma~\ref{lemma:properties of f_t(N)}, and $(c)$ follows from the choice of the step size $\eta_t=\frac{2}{(\alpha_0\log^{2}T)t}$ for $t\ge1$ and $\eta_0=0$.

{\bf Upper bound on term~$(ii)$ in \eqref{eqn:upper bound on f_t difference output estimation}.} Similarly, we may decompose each term $(f_{t}-f_{t;h})(N_t)$ in the summation in $(ii)$ as
\begin{align}
f_t(N_t)-f_{t;h}(N_t)=(f_t-f_{t;h})(N_{t-k})+(f_t-f_{t;h})(N_t)-(f_t-f_{t;h})(N_{t-k}),\label{eqn:difference between f_t and f_t;k output estimation}
\end{align}
where we notice that $\E[(f_t-f_{t;h})(N_{t-h})|\F_{t-h}]=0$. The remaining two terms in \eqref{eqn:difference between f_t and f_t;k output estimation} can be upper bounded as 
\begin{align*}
&\sum_{t=h}^{T-1}\big((f_t-f_{t;h})(N_t)-(f_t-f_{t;h})(N_{t-h})\big)\\
\le&\sum_{t=h}^{T-1}\big(f_t(N_t)-f_t(N_{t-h})+f_{t;h}(N_{t-h})-f_{t;h}(N_t)\big)\\
\le&2l_y\sum_{t=h}^{T-1}\norm{N_t-N_{t-h}}\overset{(b)}{\le}\frac{4l_y^2h}{\alpha_0\log^{2}T}(1+\log T),
\end{align*}
where the second inequality follows from the $l_y$-Lipschitzness of $f_t(\cdot)$ and $f_{t;h}(\cdot)$ shown by Lemma~\ref{lemma:properties of f_t(N)} and $(b)$ follows from similar arguments to those for \eqref{eqn:difference between N_t and N_t-k output estimation}.

Going back to \eqref{eqn:upper bound on f_t difference output estimation} completes the proof of the lemma.$\hfill\blacksquare$

\subsection{Proof of Lemma~\ref{lemma:high prob bound on X_t output estimation}}
We first decompose $X_{h}(N^{\star}),\dots,X_{T-1}(N^{\star})$ into martingale difference (sub)sequences. Defining time indices $t_{i,j}=h+i+(j-1)h$ for all $i\in\{0,\dots,h-1\}$ and all $j\in[T_i]$, where $T_i=\max\{j:t_{i,j}\le T-1\}$, we may write $\sum_{t=h}^{T-1}X_{t}=\sum_{i=1}^{h}\sum_{j=1}^{T_i}X_{t_{i,j}}$. Since $\E[X_t|\F_{t-h}]=0$ as shown by Lemma~\ref{lemma:upper bound on f_t difference output estimation} for all $t \ge h$ and $t_{i,j}=t_{i,j-1}+h$, we have $\E[X_{t_{i,j}}|\F_{t_{i,j}-1}]=0$. Considering any $i\in\{0,\dots,h-1\}$ and the corresponding filtration $\{\F_{t_{i,j}}\}_{j\ge1}$, we see that $\{X_{t_{i,j}}\}_{j\ge1}$ forms a martingale difference sequence with respect to $\{\F_{t_{i,j}}\}_{j\ge1}$. In addition, under the event $\CE_w\cap\CE_v$, we have that for any $t \ge h$, 
\begin{align*}
|X_t(N^{\star})|&\le\norm{\nabla(f_{t}-f_{t;h})(N_{t-h})}\norm{N_{t-h}-N^{\star}}+|f_t(N_{t-h})-f_t(N^{\star})|+|f_{t;h}(N^{\star})-f_{t;h}(N_{t-h})|\\
&\overset{(a)}{\le}2l_y\norm{N_{t-h}-N^{\star}}+2l_y\norm{N_{t-h}-N^{\star}},
\end{align*}
where $(a)$ follows from the fact that both $\norm{\nabla f_t(N)}$ and $\norm{\nabla f_{t;h}(N)}$ are upper bounded by $l_y$ for all $N\in\K_N$ and the fact that both $f_t(\cdot)$ and $f_{t;h}(\cdot)$ are $l_y$-Lipschitz continuous, as shown by Lemma~\ref{lemma:properties of f_t(N)}. 

Now, for any $i\in\{0,\dots,h-1\}$, we can apply the Azuma inequality (Lemma~\ref{lemma:Azuma}) and obtain
\begin{align}
\P(\sum_{j=1}^{T_i}X_{t_{i,j}}(N^{\star})\ge\epsilon)\le\exp\Big(-\frac{\epsilon^2}{32l_y^2\sum_{j=1}^{T_i}\norm{N_{t_{i,j}-h}-N^{\star}}^2}\Big).\label{eqn:Azuma used in output estimation}
\end{align}
Fixing a probability $\delta/(3h)$ for some $\delta\in(0,1)$ and letting 
\begin{equation*}
\frac{\delta}{h}=\exp\Big(-\frac{\epsilon^2}{32l_y^2\sum_{j=1}^{T_i}\norm{N_{t_{i,j}-h}-N^{\star}}^2}\Big),
\end{equation*}
we compute that
\begin{equation*}
\epsilon=\sqrt{(32l_y^2\log\frac{3h}{\delta})\sum_{j=1}^{T_i}\norm{N_{t_{i,j}-h}-N^{\star}}^2}.
\end{equation*}
It then follows from \eqref{eqn:Azuma used in output estimation} that with probability at least $1-\delta/(3h)$, 
\begin{align*}
\sum_{j=1}^{T_i}X_{t_{i,j}}(N^{\star})\le\sqrt{(32l_y^2\log\frac{3h}{\delta})\sum_{j=1}^{T_i}\norm{N_{t_{i,j}-h}-N^{\star}}^2}.
\end{align*}
Applying a union bound, we further get that with probability at least $1-\delta/3$,
\begin{align*}
\sum_{t=h}^{T-1}X_t(N^{\star})=\sum_{i=0}^{h-1}\sum_{j=1}^{T_i}X_{t_{i,j}}(N^{\star})&\le\sum_{i=0}^{h-1} \sqrt{(32l_y^2\log\frac{3h}{\delta})\sum_{j=1}^{T_i}\norm{N_{t_{i,j}-h}-N^{\star}}^2}\\
&\overset{(a)}{\le}\sum_{i=0}^{h-1}\Big(\frac{16l_y^2\log\frac{3h}{\delta}}{\varrho}+\frac{\varrho}{2}\sum_{j=1}^{T_i}\norm{N_{t_{i,j}-h}-N^{\star}}^2\Big)\\
&\overset{(b)}{=}\frac{16l_y^2h\log\frac{3h}{\delta}}{\varrho}+\frac{\varrho}{2}\sum_{t=h}^{T-1}\norm{N_{t-h}-N^{\star}}^2\\
&\le\frac{16l_y^2h\log\frac{3h}{\delta}}{\varrho}+\frac{\varrho}{2}\sum_{t=0}^{T-1}\norm{N_t-N^{\star}}^2,
\end{align*}
where $(a)$ follows from the inequality $\sqrt{ab}\le\frac{a}{2\varrho}+\frac{\varrho}{2}b$ that holds for any $a,b,\varrho\in\R_{>0}$, and $(b)$ follows from the way we set $t_{i,j}$. Setting $\varrho=(\alpha_0\log^{2}T)/2$ gives the result of the lemma. $\hfill\blacksquare$

\subsection{Proof of Lemma~\ref{lemma:upper bound on term (i) in output truncation regret}}
Consider a single term $\norm{y_t-\hat{y}_t^{\tt TK}}^2-\norm{y_t-\hat{y}_t^{\tt SK}}^2$ for any $t\in\{0,\dots,T-1\}$. We first recall from Lemma~\ref{lemma:upper bound on state and output} that $\norm{y_t}\le R_y$ (under the event $\CE_w\cap\CE_v$). We then get from \eqref{eqn:output KF truncated} that
\begin{align*}
\norm{y_t^{\tt TK}}=\norm{\sum_{s=1}^{h}N_s^{\star}y_{t-s}}&\le\norm{Y_{t-1:t-h}}\norm{N^{\star}}\\
&\le\sqrt{h}R_y R_N,
\end{align*}
where we used the upper bounds $\norm{Y_{t-1:t-h}}\le\sqrt{h}R_y$ and $\norm{N^{\star}}\le R_N$ shown in Lemmas~\ref{lemma:upper bound on state and output} and~\ref{lemma:properties of f_t(N)}, respectively. Moreover, recalling \eqref{eqn:output steady-state KF}, we also have
\begin{align}
\norm{\hat{y}_t^{\tt SK}}=\norm{\sum_{s=1}^{t}N_s^{\star}y_{t-s}}\le R_y\norm{L}\norm{C}\frac{\kappa_F}{1-\gamma_F},\label{eqn:upper bound on hat y_t^SK}
\end{align}
where the inequality follows from Lemma~\ref{lemma:truncated KF output}. Furthermore, we obtain that 
\begin{align*}
\norm{\hat{y}_t^{\tt SK}-\hat{y}_t^{\tt TK}}&=\norm{\sum_{s=h+1}^tN_s^{\star}y_{t-s}}\\
&\le R_y\norm{L}\norm{C}\frac{\kappa_F\gamma_F^h}{1-\gamma_F}.
\end{align*}

Combining the above arguments, we deduce that
\begin{align*}\nonumber
\norm{y_t-\hat{y}_t^{\tt TK}}^2-\norm{y_t-\hat{y}_t^{\tt SK}}^2&=\big(y_t-\hat{y}_t^{\tt TK}-y_t+\hat{y}_t^{\tt SK}\big)^{\top}\big(2y_t-\hat{y}_t^{\tt TK}-\hat{y}_t^{\tt SK}\big)\\
&\le\norm{\hat{y}_t^{\tt SK}-\hat{y}_t^{\tt TK}}\norm{2y_t-\hat{y}_t^{\tt TK}-\hat{y}_t^{\tt SK}}\\
&\le\Big(2R_y+\sqrt{h}R_yR_N+R_y\norm{L}\norm{C}\frac{\kappa_F}{1-\gamma_F}\Big)\norm{\hat{y}_t^{\tt SK}-\hat{y}_t^{\tt TK}}\\\nonumber
&\le\Big(2+\sqrt{h}R_N+\norm{L}\norm{C}\frac{\kappa_F}{1-\gamma_F}\Big)\kappa_F\norm{L}\norm{C}\frac{R_y^2\gamma_F^h}{1-\gamma_F}.
\end{align*}
Now, since we set $h=\lfloor\log T/\log(1/\gamma_F)\rfloor$, it holds that $\gamma_F^h\le 1/T$. Summing over $t=0,\dots,T-1$ yields \eqref{eqn:upper bound on truncation regret output estimation}.$\hfill\blacksquare$

\subsection{Proof of Lemma~\ref{lemma:upper bound on term (ii) in output truncation regret}}
We assume throughout this proof that the event $\CE_w\cap\CE_v$ holds. First, consider a single term $\norm{y_t-\hat{y}_t^{\tt SK}}^2-\norm{y_t-\hat{y}_t^{\tt KF}}^2$ for any $t\in\{0,\dots,T-1\}$, we have
\begin{align}\nonumber
\norm{y_t-\hat{y}_t^{\tt SK}}^2-\norm{y_t-\hat{y}_t^{\tt KF}}^2&=(\hat{y}_t^{\tt KF}-\hat{y}_t^{\tt SK})^{\top}(2y_t-\hat{y}_t^{\tt SK}-\hat{y}_t^{\tt KF})\\\nonumber
&\le\norm{\hat{y}_t^{\tt KF}-\hat{y}_t^{\tt SK}}\norm{2y_t-\hat{y}_t^{\tt SK}-\hat{y}_t^{\tt KF}}\\
&\le\norm{\hat{y}_t^{\tt KF}-\hat{y}_t^{\tt SK}}\big(2\norm{y_t}+\norm{\hat{y}_t^{\tt SK}}+\norm{\hat{y}_t^{\tt KF}}\big).\label{eqn:single term in steady-state regret output estimation}
\end{align}
Recall from Lemma~\ref{lemma:upper bound on state and output} that $\norm{y_t}\le R_y$ holds for all $t\in\{0,\dots,T-1\}$, and we 
have shown in the proof of Lemma~\ref{lemma:upper bound on term (i) in output truncation regret} that $\norm{\hat{y}_t^{\tt SK}}\le \sqrt{h}R_yR_N+R_y\kappa_F\norm{L}\norm{C}\frac{\gamma_F^h}{1-\gamma_F}$ for all $t\in\{0,\dots,T-1\}$. Thus, we see that to upper bound \eqref{eqn:single term in steady-state regret output estimation}, we need to further upper bound $\norm{\hat{y}_t^{\tt KF}-\hat{y}_t^{\tt SK}}$ and $\norm{\hat{y}_t^{\tt KF}}$ for all $t\in\{0,\dots,T-1\}$. To this end, we first prove the following intermediate result.
\begin{lemma}\label{lemma:upper bound on hat x^KF and hat x^SK}
Consider the Kalman filter $\hat{x}_t^{\tt KF}$ given by \eqref{eqn:KF} initialized with $\hat{x}_0^{\tt KF}=0$. For any $t\in\{0,\dots,T-1\}$, it holds that 
\begin{align}
\norm{\hat{x}_t^{\tt KF}}\le R_{\hat{x}}(t)\triangleq\sqrt{\frac{\norm{W+LVL^{\top}}\kappa_F^2}{\lambda_{\min}(W)(1-\gamma_F^2)}}R_y\norm{L}t+\frac{\kappa_F^4\norm{W-\Sigma}\norm{C}\norm{W+LVL^{\top}}}{\lambda_{\min}(V)\lambda_{\min}(W)(1-\gamma_F)^2(1+\gamma_F)}.\label{eqn:upper bound on hat x_t^KF}
\end{align}
In addition, for any $t\in\{0,\dots,T-1\}$, it holds that 
\begin{align*}
\norm{\hat{x}_t^{\tt SK}}\le\norm{L}R_y\frac{\kappa_F}{1-\gamma_F},
\end{align*}
where $\hat{x}_t^{\tt SK}$ is given by \eqref{eqn:steady-state KF}.
\end{lemma}
\begin{proof}
Recall that we denote 
\begin{align*}
\Psi_{k,l}=(A-L_{k-1}C)(A-L_{k-2}C)\times\cdots\times(A-L_lC),
\end{align*}
for any $k,l\in\BZ_{\ge0}$ with $k\le l$, and let $\Psi_{k,l}=I_n$ if $k>l$. We then get from \eqref{eqn:KF} that
\begin{align}\nonumber
\norm{\hat{x}_t^{\tt KF}}&=\norm{\sum_{s=0}^{t-1}\Psi_{t,s+1}L_sy_s}\\\nonumber
&\overset{(a)}{\le}\sqrt{\frac{\norm{W+LVL^{\top}}\kappa_F^2}{\lambda_{\min}(W)(1-\gamma_F^2)}}R_y\sum_{s=0}^{t-1}\norm{L_s}\\\nonumber
&\overset{(b)}{\le}\sqrt{\frac{\norm{W+LVL^{\top}}\kappa_F^2}{\lambda_{\min}(W)(1-\gamma_F^2)}}R_y\sum_{s=1}^{t-1}\left(\norm{L}+\frac{\kappa_F^3\gamma_F^{s-1}\norm{W-\Sigma}\norm{C}}{\alpha_0}\sqrt{\frac{\norm{W+LVL^{\top}}}{{\lambda_{\min}(W)(1-\gamma_F^2)}}}\right)\\
&\le\sqrt{\frac{\norm{W+LVL^{\top}}\kappa_F^2}{\lambda_{\min}(W)(1-\gamma_F^2)}}R_y\left(t\norm{L}+\frac{\kappa_F^3\norm{W-\Sigma}\norm{C}}{(1-\gamma_F)\alpha_0}\sqrt{\frac{\norm{W+LVL^{\top}}}{{\lambda_{\min}(W)(1-\gamma_F^2)}}}\right),\label{eqn:upper bound on norm of hat x_t^KF}
\end{align}
where $(a)$ follows from \eqref{eqn:upper bound on Psi_t,1 with tilde_Sigma_t refined} in Lemma~\ref{lemma:KF convergence properties}, $(b)$ follows from \eqref{eqn:convergence of L_t} in Lemma~\ref{lemma:KF convergence properties} and the fact $L_0=0$ since $\Sigma_0=0$.

Next, we get from \eqref{eqn:steady-state KF} and Lemma~\ref{lemma:truncated KF} that for any $t\in\{0,\dots,T-1\}$,
\begin{align*}
\norm{\hat{x}_t^{\tt SK}}=\norm{\sum_{s=1}^{t}M_s^{\star}y_{s-1}}\le\norm{L}R_y\frac{\kappa_F}{1-\gamma_F}.
\end{align*}
\end{proof}

In the remainder of this proof, we split our arguments into $t\in\{0,\dots,2s_0-1\}$ and $t\in\{2s_0,\dots,T-1\}$,\footnote{Note that if $T\le 2s_0-1$, we simply consider $t\in\{0,\dots,2s_0-1\}$ in the rest of the proof.} where we let
\begin{align}
s_0=\max\left\{\frac{(1-\gamma_F)\lambda_{\min}(V)}{2\kappa_F^4\norm{W-\Sigma}\norm{C}^2\log(1/\gamma_F)}\sqrt{\frac{\lambda_{\min}(W)(1-\gamma_F^2)}{\norm{W+LVL^{\top}}}},\frac{\log T}{\log(1/\gamma_F)}\right\}+1.\label{eqn:def of s_0}
\end{align}

{\bf Upper bound on $\norm{\hat{y}_t^{\tt KF}-\hat{y}_t^{\tt SK}}$ for all $t\in\{2s_0,\dots,T-1\}$.} First, note that for any $t\in\{2s_0,\dots,T-1\}$,
\begin{align*}
\hat{x}_t^{\tt KF}=\Psi_{t,s_0}\hat{x}_{s_0}^{\tt KF}+\sum_{s=s_0}^{t-1}\Psi_{t,s+1}L_sy_s,
\end{align*}
Similarly, note that the (untruncated) steady-state Kalman filter may also be written as 
\begin{align*}
\hat{x}_{t+1}^{\tt SK}=(A-LC)\hat{x}_t^{\tt SK}+Ly_t
\end{align*}
initialized with $\hat{x}_0^{\tt SK}=0$. It follows that for any $t\ge s_0$,
\begin{align*}
\hat{x}_t^{\tt SK}&=(A-LC)^{t-s_0}\hat{x}_{s_0}^{\tt SK}+\sum_{s=s_0}^{t-1}(A-LC)^{t-(s+1)}Ly_s.
\end{align*}
Since $\hat{y}_t^{\tt KF}=C\hat{x}_t^{\tt KF}$ and $\hat{y}_t^{\tt SK}=C\hat{x}_t^{\tt SK}$, we compute that for any $t\ge 2s_0$, 
\begin{align}
\hat{y}_t^{\tt KF}-\hat{y}_t^{\tt SK}=\underbrace{C\Psi_{t,s_0}\hat{x}_{s_0}^{\tt KF}-C(A-LC)^{t-s_0}\hat{x}_{s_0}^{\tt SK}}_{(i)}+\underbrace{C\sum_{s=s_0}^{t-1}\big(\Psi_{t,s+1}L_s-(A-LC)^{t-(s+1)}L\big)y_s}_{(ii)}.\label{eqn:decompose y_t^KF-y_t^SK}
\end{align}

For term~$(i)$ in \eqref{eqn:decompose y_t^KF-y_t^SK}, we have
\begin{align}\nonumber
\norm{(i)}&\le\norm{C\big(\Psi_{t,s_0}-(A-LC)^{t-s_0}\big)\hat{x}_{s_0}^{\tt KF}}+\norm{C(A-LC)^{t-s_0}(\hat{x}_{s_0}^{\tt KF}-\hat{x}_{s_0}^{\tt SK})}\\
&\le\norm{C}\norm{\hat{x}_{s_0}^{\tt KF}}\norm{\Psi_{t,s_0}-(A-LC)^{t-s_0}}+\norm{C}\norm{(A-LC)^{t-s_0}}\norm{\hat{x}_{s_0}^{\tt KF}-\hat{x}_{s_0}^{\tt SK}}.\label{eqn:upper bound on term (i) in hat y_t^KF-hat y_t^SK}
\end{align}
By the expression of $\Psi_{t,s_0}$, we have
\begin{align*}
\norm{\Psi_{t,s_0}-(A-LC)^{t-s_0}}&=\norm{(A-L_{t-1}C)\cdots(A-L_{s_0}C)-(A-LC)^{t-s_0}}.
\end{align*}
Note from Lemma~\ref{lemma:KF convergence properties} that for any $s\ge s_0$,
\begin{align}
\norm{A-L_sC-(A-LC)}\le\norm{C}\norm{L_s-L}\le\underbrace{\frac{\kappa_F^3\gamma_F^{s_0}\norm{W-\Sigma}\norm{C}^2}{\lambda_{\min}(V)}\sqrt{\frac{\norm{W+LVL^{\top}}}{{\lambda_{\min}(W)(1-\gamma_F^2)}}}}_{\varepsilon_{s_0}}.\label{eqn:epsilon_s0}
\end{align}
It then follows from Lemma~\ref{lemma:power of perturbed matrix} that for any $t\ge s_0$,
\begin{align}\nonumber
\norm{\Psi_{t,s_0}-(A-LC)^{t-s_0}}&\le(t-s_0)\kappa_F^2(\kappa_F\varepsilon_{s_0}+\gamma_F)^{t-s_0-1}\varepsilon_{s_0}\\
&\le(t-s_0)\kappa_F^2\left(\frac{1+\gamma_F}{2}\right)^{t-s_0-1}\varepsilon_{s_0},\label{eqn:convergence of Psi_t,s_0}
\end{align}
where we also used the fact $\varepsilon_{s_0}\le(1-\gamma_F)/(2\kappa_F)$, i.e., $\kappa_F\varepsilon_{s_0}+\gamma_F\le(1+\gamma_F)/2$, which can be shown via the choice of $s_0$. Going back to \eqref{eqn:upper bound on term (i) in hat y_t^KF-hat y_t^SK}, we obtain that for any $t\ge2s_0$,
\begin{align}\nonumber
\norm{(i)}&\le\norm{C}R_{\hat{x}}(s_0)(t-s_0)\kappa_F^2\left(\frac{1+\gamma_F}{2}\right)^{t-s_0-1}\varepsilon_{s_0}+\norm{C}\kappa_F\gamma_F^{t-s_0}\left(R_{\hat{x}}(s_0)+\norm{L}R_y\frac{\kappa_F}{1-\gamma_F}\right)\\
&=(t-s_0)\left(\frac{1+\gamma_F}{2}\right)^{t-s_0-1}\CO\left(\frac{\log^{1.5}T}{T}\right)+\CO\left(\frac{\log^{1.5}T}{T}\right).\label{eqn:upper bound on term (i) in hat y_t^KF-hat y_t^SK refined}
\end{align}
To obtain \eqref{eqn:upper bound on term (i) in hat y_t^KF-hat y_t^SK refined}, we first recall the expressions of $\kappa_F,\gamma_F$ given in \eqref{eqn:known constants}, the bounds $\alpha_0,\alpha_1,\bar{\sigma},\psi$ given in \eqref{eqn:known constants}, and $\norm{L}\le\kappa_F$ shown by Lemma~\ref{lemma:strongly stable of L}. One can then show via the choice of $s_0$ that $\gamma_F^{s_0-1}=1/T$, which also implies $\varepsilon_{s_0}=\CO(1/T)$, where $\CO(\cdot)$ hides the factors stated in Theorem~\ref{thm:regret for output estimation}. Moreover, we also deduce from Lemma~\ref{lemma:upper bound on hat x^KF and hat x^SK} and the facts $s_0=\CO(\log T)$ (see~\eqref{eqn:def of s_0}) and $R_y=\CO(\sqrt{\log T})$ (see \eqref{eqn:R_x and R_y}) that $R_{\hat{x}}(s_0)=\CO(\log^{1.5}T)$.

Similarly, for term~$(ii)$ in \eqref{eqn:decompose y_t^KF-y_t^SK}, we have
\begin{align}\nonumber
\norm{(ii)}&\le\norm{C}R_y\sum_{s=s_0}^{t-1}\norm{\Psi_{t,s+1}L_s-(A-LC)^{t-(s+1)}L}\\\nonumber
&\le\norm{C}R_y\sum_{s=s_0}^{t-1}\big(\lVert\Psi_{t,s+1}-(A-LC)^{t-(s+1)}\rVert\lVert L_s\rVert+\lVert(A-LC)^{t-(s+1)}\rVert\lVert L_s-L\rVert\big)\\\nonumber
&\overset{(a)}{\le} R_y\sum_{s=s_0}^{t-1}\left((t-s-1)\kappa_F^2\left(\frac{1+\gamma_F}{2}\right)^{t-s-2}\varepsilon_{s_0}(\norm{C}\norm{L}+\varepsilon_{s_0})+\kappa_F\gamma_F^{t-(s+1)}\varepsilon_{s_0}\right)\\\nonumber
&\overset{(b)}{\le} \frac{4R_y}{(1-\gamma_F)^2}\varepsilon_{s_0}(\norm{C}\norm{L}+\varepsilon_{s_0})+\frac{R_y\kappa_F}{1-\gamma_F}\varepsilon_{s_0}\\
&\overset{(c)}{=}\CO\left(\frac{\sqrt{\log T}}{T}\right),\label{eqn:upper bound on term (ii) in hat y_t^KF-hat y_t^SK}
\end{align}
where $(a)$ follows from \eqref{eqn:epsilon_s0} and \eqref{eqn:convergence of Psi_t,s_0}.\footnote{More precisely, one can show that \eqref{eqn:convergence of Psi_t,s_0} still holds if the $s_0$ term is replaced by $s+1$, since $s\ge s_0$ in relation $(a)$.} To obtain $(b)$, we used the relation $\sum_{s=s_0}^{t-1}(t-s-1)((1+\kappa_F)/2)^{t-s-2}\le1/(1-\kappa_F)^2$ proved in Lemma~\ref{lemma:series upper bound}. To obtain $(c)$, we again notice that $R_y=O(\sqrt{\log T})$ and $\varepsilon_0=O(1/T)$ as we argued above.

Combining \eqref{eqn:upper bound on term (i) in hat y_t^KF-hat y_t^SK} and \eqref{eqn:upper bound on term (ii) in hat y_t^KF-hat y_t^SK}, we obtain from \eqref{eqn:decompose y_t^KF-y_t^SK} that for any $t\in\{2s_0,\dots,T-1\}$,
\begin{align}\nonumber
\norm{\hat{y}_t^{\tt KF}-\hat{y}_t^{\tt SK}}&\le(t-s_0)\left(\frac{1+\gamma_F}{2}\right)^{t-s_0-1}\CO\left(\frac{\log^{1.5}T}{T}\right)+\CO\left(\frac{\log^{1.5}T}{T}\right)\\
&=\CO\left(\frac{\log^{1.5}T}{T}\right),\label{eqn:upper bound on norm hat y_t^KF- hat y_t^SK}
\end{align}
where the equality follows from the fact $((1+\gamma_F)/2)^{s_0-1}\le 1/T$ which can also be shown via the choice of $s_0$.

{\bf Upper bound on $\norm{\hat{y}_t^{\tt KF}}$ for all $t\in\{2s_0,\dots,T-1\}$.} For any $t\in\{2s_0,\dots,T-1\}$, we know that 
\begin{align*}
\norm{\hat{y}_t^{\tt KF}}\le\norm{\hat{y}_t^{\tt KF}-\hat{y}_t^{\tt SK}}+\norm{\hat{y}_t^{\tt SK}}.
\end{align*}
As we have shown in the proof of Lemma~\ref{lemma:upper bound on term (i) in output truncation regret} that for any $t\in\{2s_0,\dots,T-1\}$,
\begin{align*}
\norm{\hat{y}_t^{\tt SK}}\le\norm{C}\norm{L}R_y\frac{\kappa_F}{1-\gamma_F}=\CO(\sqrt{\log T}).
\end{align*}
It follows that for any $t\in\{2s_0,\dots,T-1\}$,
\begin{align}
\norm{\hat{y}^{\tt KF}_t}=\CO\left(\frac{\log^{1.5}T}{T}\right)+\CO(\sqrt{\log T}).\label{eqn:upper bound on hat y_t^KF}
\end{align}

Now, going back to \eqref{eqn:single term in steady-state regret output estimation} and using the upper bounds in \eqref{eqn:upper bound on norm hat y_t^KF- hat y_t^SK} and \eqref{eqn:upper bound on hat y_t^KF}, we obtain that for any $t\in\{2s_0,\dots,T-1\}$,
\begin{align}
\norm{y_t-\hat{y}_t^{\tt SK}}^2-\norm{y_t-\hat{y}_t^{\tt KF}}^2=\CO\left(\frac{\log^{1.5}T}{T}\right)\left(\CO(\sqrt{\log T})+\CO\left(\frac{\log^{1.5}T}{T}\right)\right),\label{eqn:single term in steady-state regret output estimation refined}
\end{align}
where we also used $R_y=\CO(\sqrt{\log T})$ and the above upper bound on $\norm{y_t^{\tt SK}}$. Summing over all $t\in\{2s_0,\dots,T-1\}$, we get 
\begin{align}
\sum_{t=2s_0}^{T-1}\big(\norm{y_t-\hat{y}_t^{\tt SK}}^2-\norm{y_t-\hat{y}_t^{\tt KF}}^2\big)&=\CO\left(\log^2 T\right)+\CO\left(\frac{\log^3 T}{T}\right)=\CO(\log^2 T),\label{eqn:upper bound on steady-state regret output estimation}
\end{align}
where we used the fact $T\ge\log T$.

{\bf Upper bound on the summation of $\norm{y_t-\hat{y}_t^{\tt SK}}^2-\norm{y_t-\hat{y}_t^{\tt KF}}^2$ over $t\in\{0,\dots,2s_0-1\}$.} We have
\begin{align}\nonumber
\sum_{t=0}^{2s_0-1}\big(\norm{y_t-\hat{y}_t^{\tt SK}}^2-\norm{y_t-\hat{y}_t^{\tt KF}}^2\big)&\overset{(a)}{\le}\sum_{t=0}^{2s_0-1}\big(4\norm{y_t}^2+2\norm{y_t^{\tt SK}}^2+2\norm{\hat{y}_t^{\tt KF}}^2\big)\\\nonumber
&\overset{(b)}{\le}\sum_{t=0}^{2s_0-1}\big(4R_y+2R_y\norm{L}\norm{C}\frac{\kappa_F}{1-\gamma_F}+\CO(R_y^2t^2)\big)\\\nonumber
&\le2s_0\big(\CO(\sqrt{\log T})+\CO((2s_0-1)^2\log T)\big)\\
&\overset{(c)}\le\CO(\log^4T),\label{eqn:upper bound on steady-state regret output estimation remaining terms}
\end{align}
where $(a)$ follows from the fact $\norm{a+b}^2\le2\norm{a}^2+2\norm{b}^2$. To obtain $(b)$, we first recall the upper bounds $\norm{y_t}\le R_y$ and the upper bound on $\norm{y_t^{\tt SK}}$ given by \eqref{eqn:upper bound on hat y_t^SK} in the proof of Lemma~\ref{lemma:upper bound on term (i) in output truncation regret}, for all $t\in\{0,\dots,T-1\}$. Next, we notice from \eqref{eqn:upper bound on hat x_t^KF} in Lemma~\ref{lemma:upper bound on hat x^KF and hat x^SK} that $\norm{\hat{x}_t^{\tt KF}}=\CO(R_yt)$, which implies that $\norm{\hat{y}_t^{\tt KF}}^2\le\norm{C}^2\norm{\hat{y}_t^{\tt KF}}^2=\CO(R_y^2t^2)$ for all $t\in\{0,\dots,T-1\}$. Finally, $(c)$ follows from the choice of $s_0=\CO(\log T)$. 

Combining \eqref{eqn:upper bound on steady-state regret output estimation} and \eqref{eqn:upper bound on steady-state regret output estimation remaining terms} completes the proof of Lemma~\ref{lemma:upper bound on term (ii) in output truncation regret}.$\hfill\blacksquare$

\subsection{Doubling Trick to Remove the Knowledge of $T$ in Algorithm~\ref{alg:OCO for output estimation}}\label{app:ddoubling trick for output estimation}
Theorem~\ref{thm:regret for output estimation} shows that Algorithm~\ref{alg:OCO for output estimation} achieves $\CO(\log^4T)$ regret when the horizon length $T$ is known a priori. We now argue that using a doubling trick, one can achieve the same regret without the knowledge of $T$. The arguments follow by  dividing the total horizon $T$ into $M$ consecutive intervals with length $L_i$ for $i\ge0$, where $L_i=T_i-T_{i-1}$ with $T_i=2^{2^i}$ and $T_{-1}=0$, and $M=\lceil 
\log(\frac{\log T}{\log 2})/\log 2\rceil=\CO(\log\log T)$. For each interval of length $L_i$, we apply Algorithm~\ref{alg:OCO for output estimation} with the horizon length set to be $L_i$, which yields the regret ${\tt R_y}(L_i)=\CO(\log^4L_i)$. Summing over all $i\in\{0,\dots,M\}$, we obtain
\begin{align*}
{\tt R_y}(T)=\sum_{i=0}^M{\tt R_y}(T_i-T_{i-1})&\le\sum_{i=0}^M\CO(\log^4T_i)\\
&=\sum_{i=0}^{M}\CO(2^{4i})\\
&=\CO\Big(\frac{1-2^{4(M+1)}}{1-2^4}\Big)\\
&=\CO(2^{4(M+1)})\\
&=\CO\big(2^{4\log\log T}\big)=\CO(\log^4T).
\end{align*}

\section{Omitted Proofs and Discussions in Section~\ref{sec:state estimation} for State Estimation}\label{app:state estimation}
\subsection{Proof of Theorem~\ref{thm:regret lower bound}}
By Yao's principle \cite{yao1977probabilistic,arora2009computational}, to establish \eqref{eqn:expected lower bound} in the theorem, it is enough to demonstrate a random instance of the state estimation problem and prove that the expected regret of any deterministic algorithm when applied to solve such a random problem instance is large.\footnote{Since we consider any algorithm including the optimal Kalman filter that designs the state estimate $\hat{x}_t$ for each time step $t\ge0$ based on the past outputs $y_0,\dots,y_{t-1}$. A deterministic algorithm will return a state estimate $\hat{x}_t$ that is a deterministic function of $y_0,\dots,y_{t-1}$.}

First, we describe the random problem instance to be considered. To this end, we define a discrete random variable $r\in\R$ with $\P(r=1)=\P(r=4)=\P(r=-2)=1/3$, and consider a class of scalar LIT systems (parameterized by the value of $r$) given by
\begin{equation}
\begin{split}\label{eqn:LTI in lower bound proof}
x_{t+1}&=ax_t+w_t\\
y_t&=cx_t+v_t.
\end{split}
\end{equation}
Here, we let $a=1/5$, $c=1/r$, $v_t\overset{i.i.d.}{\sim}\CN(0,\sigma_v)$,  $w_t=r\tilde{w}_t$ with $\tilde{w}_t\overset{i.i.d.}{\sim}\CN(0,\sigma_w)$,  $\sigma_v,\sigma_w\in\R_{>0}$, and $w_t$, $v_t$ and $r$ are assumed to be independent of each other for all $t\ge0$; set the initial condition to be $x_0=0$. Conditioned on different values of $r$, the LTI system in~\eqref{eqn:LTI in lower bound proof} takes different forms. Specifically, denoting $r_1=1$, $r_2=4$ and $r_3=-2$, we obtain that conditioned on $r=r_i$ for any $i\in\{1,2,3\}$, system~\eqref{eqn:LTI in lower bound proof} satisfies that
\begin{equation}\label{eqn:three sys in lower bound proof}
\begin{split}
x_{t+1}^{\tt s(i)}&=\frac{1}{5}x_t^{\tt s(i)}+r_i\tilde{w}_t\\
y_t^{\tt s(i)}&=\frac{1}{r_i}x_t^{\tt s(i)}+v_t,
\end{split}
\end{equation}
where we use $x_t^{\tt s(i)}$ and $y_t^{\tt s(i)}$ to denote the corresponding state and output of the system indexed by ${\tt s(i)}$, respectively, and note that we also have $x_t^{\tt s(i)}=0$. One can now verify that the LTI systems~${\tt s(1)}$, ${\tt s(2)}$ and ${\tt s(3)}$ are related to each other by some similarity transformation captured by $r$. Specifically, one can show the following result.
\begin{lemma}
\label{lemma:similar systems}
Given the same initial condition $x_0^{\tt s(1)}=x_0^{\tt s(2)}=x_0^{\tt s(3)}=0$, it holds that (i) $x_t^{\tt s(2)}=r_2x_t^{\tt s(1)}$, $x_t^{\tt s(3)}=r_3x_t^{\tt s(1)}$ for all $t\ge0$; and (ii) $y_t^{\tt s(1)}=y_t^{\tt s(2)}=y_t^{\tt s(3)}=y_t$ for all $t\ge0$, where $y_t$ is given by \eqref{eqn:LTI in lower bound proof}. 
\end{lemma}

To proceed, recalling the regret ${\tt R_x}(T)$ for the state estimation problem defined in \eqref{eqn:R_T^x}, we know that the expected regret of any deterministic algorithm satisfies that
\begin{equation}\label{eqn:E[R_x(T)]}
\E\big[{\tt R_x}(T)\big]=\E\left[\sum_{t=0}^{T-1}\big((\hat{x}_t-x_t)^2-(\hat{x}_t^{\tt KF}-x_t)^2\big)\right],
\end{equation}
where $\E[\cdot]$ denotes the expectation with respect to $\{w_t\}_{t\ge0}$, $\{v_t\}_{t\ge0}$ and $r$, $\hat{x}_t$ is the state estimate returned by the algorithm, and $\hat{x}_t^{\tt KF}$ is the Kalman filter given by \eqref{eqn:KF}. For our analysis in the sequel, let $\E_{w,v}[\cdot]$ denote the expectation with respect to $\{w_t\}_{t\ge0}$ and $\{v_t\}_{t\ge0}$, and let $\E_r[\cdot]$ denote the expectation with respect to $r$. Moreover, let $\hat{x}_t^{\tt s(i)}$ and $\hat{x}_t^{\tt KF,s(i)}$ respectively denote the state estimate returned by the deterministic algorithm and the Kalman filter, conditioned on $r=r_i$, for all $i\in\{1,2,3\}$, which are initialized as $\hat{x}_0^{\tt KF,s(i)}=\hat{x}_0^{\tt s(i)}=0$. We then deduce that for any $t\in\{0,\dots,T-1\}$, the corresponding term within the summation in~\eqref{eqn:E[R_x(T)]} satisfies 
\begin{align}\nonumber
\E\big[(\hat{x}_t-x_t)^2-(\hat{x}_t^{\tt KF}-x_t)^2\big]=&\E_{w,v}\Big[\E_r\big[(\hat{x}_t-x_t)^2-(\hat{x}_t^{\tt KF}-x_t)^2\big]\Big]\\\nonumber
=&\E_{w,v}\Big[\sum_{i=1}^3\frac{1}{3}\big((\hat{x}_t^{\tt s(i)}-x_t^{\tt s(i)})^2-(\hat{x}_t^{\tt KF,s(i)}-x_t^{\tt s(i)})^2\big)\Big]\\
=&\sum_{i=1}^3\frac{1}{3}\underbrace{\E_{w,v}\Big[(\hat{x}_t^{\tt s(i)}-x_t^{\tt s(i)})^2-(\hat{x}_t^{\tt KF,s(i)}-x_t^{\tt s(i)})^2\Big]}_{R^{\tt s(i)}_t},\label{eqn:def of R_i}
\end{align}
where the first equality follows from the fact that $r$ is independent of $\{w_t\}_{t\ge0}$ and $\{v_t\}_{t\ge0}$. 

Our goal next is to lower bound~\eqref{eqn:def of R_i} by a positive number. Using the notation $R_t^{\tt s(i)}$ introduced in \eqref{eqn:def of R_i}, our strategy is to first show that $(R_t^{\tt s(2)}-R_t^{\tt s(1)})+(R_t^{\tt s(3)}-R_t^{\tt s(1)})$ is lower bounded by a positive number. If we can also show that $R^{\tt s(1)}_t\ge0$, then we obtain a positive lower bound for 
\eqref{eqn:def of R_i}. To this end, for any $i\in\{2,3\}$ and any $t\ge0$, we first write
\begin{align}\nonumber
R^{\tt s(i)}_t-R_t^{\tt s(1)}
&=\E_{w,v}\Big[\E_{w,v}\Big[(\hat{x}_t^{\tt s(i)}-x_t^{\tt s(i)})^2-(\hat{x}_t^{\tt s(1)}-x_t^{\tt s(1)})^2\Big]\Big|y_0,\dots,y_{t-1}\Big]\\\nonumber
&\qquad\qquad\qquad\qquad\qquad+\E_{w,v}\Big[(\hat{x}_t^{\tt KF,s(1)}-x_t^{\tt s(1)})^2-(\hat{x}_t^{\tt KF,s(i)}-x_t^{\tt s(i)})^2\Big]\\\nonumber
&=\E_{w,v}\Big[\E_{w,v}\Big[(\hat{x}_t-x_t^{\tt s(i)})^2-(\hat{x}_t-x_t^{\tt s(1)})^2\Big]\Big|y_0,\dots,y_{t-1}\Big]\\
&\qquad\qquad\qquad\qquad\qquad+\E_{w,v}\Big[(\hat{x}_t^{\tt KF,s(1)}-x_t^{\tt s(1)})^2-(\hat{x}_t^{\tt KF,s(i)}-x_t^{\tt s(i)})^2\Big],\label{eqn:decomposition of R_2-R_1}
\end{align}
where the first equality uses the total expectation law, and the second equality uses the fact that the state estimate $\hat{x}_t$ returned by the deterministic algorithm are determined given $y_0,\dots,y_{t-1}$, which allows us to use the same symbol $\hat{x}_t$ (resp., $\hat{x}_t^{\tt KF}$) for the state estimate returned by the deterministic algorithm corresponding to the systems~${\tt s(1)}$ and ${\tt s(2)}$. The following results provide lower bounds for the two terms on the right-hand side of \eqref{eqn:decomposition of R_2-R_1}.
\begin{lemma}\label{lemma:lower bound term (i) in R^i-R^1}
For any $i\in\{2,3\}$ and any $t\ge1$, it holds that 
\begin{multline}
\E_{w,v}\Big[\E_{w,v}\Big[(\hat{x}_t-x_t^{\tt s(i)})^2-(\hat{x}_t-x_t^{\tt s(1)})^2\Big]\Big|y_0,\dots,y_{t-1}\Big]\\\ge\E_{w,v}\Big[\E_{w,v}\Big[2(1-r_i)x_t^{\tt s(1)}\hat{x}_t\Big|y_0,\dots,y_{t-1}\Big]\Big]+(r_i^2-1)\sigma_w.\label{eqn:upper bound on term (i) of R_2-R_1}
\end{multline}
\end{lemma}
\begin{proof}
For any $i\in\{2,3\}$ and any $t\ge1$, we have
\begin{align}\nonumber
&\E_{w,v}\Big[\E_{w,v}\Big[(\hat{x}_t-x_t^{\tt s(i)})^2-(\hat{x}_t-x_t^{\tt s(1)})^2\Big]\Big|y_0,\dots,y_{t-1}\Big]\\\nonumber
=&\E_{w,v}\Big[\E_{w,v}\Big[(x_t^{\tt s(1)}-x_t^{\tt s(i)})(2\hat{x}_t-x_t^{\tt s(i)}-x_t^{\tt s(1)})\Big|y_0,\dots,y_{t-1}\Big]\Big]\\\nonumber
\overset{(a)}{=}&\E_{w,v}\Big[\E_{w,v}\Big[(1-r_i)x_t^{\tt s(1)}\big(2\hat{x}_t-(r_i+1)x_t^{\tt s(1)}\big)\Big|y_0,\dots,y_{t-1}\Big]\Big]\\\nonumber
=&\E_{w,v}\Big[\E_{w,v}\Big[2(1-r_i)x_t^{\tt s(1)}\hat{x}_t-(1-r_i^2)(x_t^{\tt s(1)})^2\Big|y_0,\dots,y_{t-1}\Big]\Big]\\\nonumber
\overset{(b)}{=}&\E_{w,v}\Big[\E_{w,v}\Big[2(1-r_i)x_t^{\tt s(1)}\hat{x}_t\Big|y_0,\dots,y_{t-1}\Big]\Big]-\E_{w,v}\Big[(1-r_i^2)(x_t^{\tt s(1)})^2\Big],
\end{align}
where $(a)$ follows from Lemma~\ref{lemma:similar systems}, and $(b)$ uses again the total expectation law. Additionally, we know that for any $t\ge1$,
\begin{align*}
\E_{w,v}\Big[(r_i^2-1)(x_t^{\tt s(1)})^2\Big]&=(r_i^2-1)\E_{w,v}\left[\left(\frac{1}{5}x_{t-1}^{\tt s(1)}+\tilde{w}_t\right)^2\right]\\
&\overset{(a)}{=}(r_i^2-1)\E_{w,v}\left[\frac{1}{25}(x_{t-1}^{\tt s(1)})^2+\frac{2}{5}x_{t-1}^{\tt s(1)}\tilde{w}_t+\tilde{w}_t^2\right]\\
&\overset{(b)}{\ge}(r_i^2-1)\E_{w,v}\left[\frac{2}{5}x_{t-1}^{\tt s(1)}\right]\E_{w,v}[\tilde{w}_t]+(r_i^2-1)\E_{w,v}[\tilde{w}_t^2]\\
&\overset{(c)}{=}(r_i^2-1)\sigma_w,
\end{align*}
where $(a)$ uses \eqref{eqn:three sys in lower bound proof}, $(b)$ uses the fact $r_i^2-1>0$ for $r_2=4$ and $r_3=-2$, and $(c)$ uses the fact that $\tilde{w}_t$ is independent of $\tilde{x}_{t-1}^{\tt s(1)}$. This completes the proof of the lemma.
\end{proof}

\begin{lemma}\label{lemma:lower bound term (ii) in R^i-R^1}
For any $i\in\{2,3\}$ and any $t\ge0$, it holds that 
\begin{align*}
\E_{w,v}\Big[(\hat{x}_t^{\tt KF,s(1)}-x_t^{\tt s(1)})^2-(\hat{x}_t^{\tt KF,s(i)}-x_t^{\tt s(i)})^2\Big]\ge-\frac{r_i^3+25r_i\sigma_w}{625\sigma_wT}.
\end{align*}
\end{lemma}\label{lemma:}
\begin{proof}
Considering any system~${\tt s(i)}$ for $i\in\{1,2,3\}$, we know from \eqref{eqn:KF} that $\hat{x}_t^{\tt KF,s(i)}$ is given by
\begin{equation}\label{eqn:KF for scalar system in lower bound proof}
\hat{x}_{t+1}^{\tt KF,s(i)}=\left(\frac{1}{5}-\frac{l_t^{\tt s(i)}}{r_i}\right)\hat{x}_t^{\tt KF,s(i)}+l_t^{\tt s(i)}y_t,
\end{equation}
initialized with $ \hat{x}_t^{\tt KF,s(1)}=0$, where 
\begin{equation*}
l_t^{\tt s(i)}=\frac{\sigma_t^{\tt s(i)}}{5r_i(\sigma_t^{\tt s(i)}/r_i^2+\sigma_v)},
\end{equation*}
and $\sigma_t^{\tt s(i)}\ge0$ is given by the recursion
\begin{equation}\label{eqn:scalar DARE in lower bound proof}
\sigma_t^{\tt s(i)}=\frac{\sigma_{t-1}^{\tt s(i)}}{25}-\frac{(\sigma_{t-1}^{\tt s(i)})^2}{25r_i^2(\sigma_{t-1}^{\tt s(i)}/r_i^2+\sigma_v)}+\sigma_w,
\end{equation}
initialized with $\sigma_0^{\tt s(i)}=0$. 

Note that the error covariance of the Kalman filter satisfies that $\E_{w,v}\big[(\hat{x}_t^{\tt KF,s(i)}-x_t^{\tt s(i)})^2\big]=\sigma_t^{\tt s(i)}$ for all $t\ge0$ (see, e.g., \cite{anderson2005optimal}), which yields
\begin{align}
\E_{w,v}\Big[(\hat{x}_t^{\tt KF,s(1)}-x_t^{\tt s(1)})^2-(\hat{x}_t^{\tt KF,s(i)}-x_t^{\tt s(i)})^2\Big]=\sigma_t^{\tt s(1)}-\sigma_t^{\tt s(i)}.\label{eqn:KF error difference in the lower bound proof}
\end{align}
Moreover, observe from \eqref{eqn:scalar DARE in lower bound proof} that $\sigma_t^{\tt s(i)}$ can be viewed as a function of $\sigma_v$ for all $t\ge0$, and we write $\sigma_t^{\tt s(i)}(\sigma_v)$ accordingly. Also observe from \eqref{eqn:scalar DARE in lower bound proof} that $\sigma_t^{\tt s(i)}(0)=\sigma_w$ for all $t\ge0$ and all $i\in\{1,2,3\}$. Thus, we have $\sigma_t^{\tt s(1)}(0)-\sigma_t^{\tt s(2)}(0)=0$ for all $t\ge0$. Below, we argue by the continuity of $\sigma_t^{\tt s(i)}(\sigma_v)$ in $\sigma_v$ that $\sigma_t^{\tt s(1)}(\sigma_v)-\sigma_t^{\tt s(2)}(\sigma_v)$ is also small under our choice of $\sigma_v$ (which is close to $0$). 

Let us consider any system~${\tt s(i)}$ for $i\in\{1,2,3\}$. For any $t\ge2$, we calculate
\begin{align}\nonumber
\sigma_t^{\tt s(i)}(\sigma_v)-\sigma_t^{\tt s(i)}(0)&=\frac{\sigma_{t-1}^{\tt s(i)}}{25}-\frac{(\sigma_{t-1}^{\tt s(i)}(\sigma_v))^2}{25r_i^2(\sigma_{t-1}^{\tt s(i)}(\sigma_v)/r_i^2+\sigma_v)}+\sigma_w-\sigma_w\\
&=\frac{\sigma_vr_i\sigma_{t-1}^{\tt s(i)}(\sigma_v)}{25\sigma_{t-1}^{\tt s(i)}(\sigma_v)+25r_i\sigma_v}.\label{eqn:sigma_t(sigma_v)-sigma_t(0)}
\end{align}
To further upper bound \eqref{eqn:sigma_t(sigma_v)-sigma_t(0)}, we will invoke lower and upper bounds on $\sigma_{t-1}^{\tt s(i)}(\sigma_v)$ for all $t\ge2$. A lower bound can be easily obtained from \eqref{eqn:scalar DARE in lower bound proof} as $\sigma_{t-1}^{\tt s(i)}\ge\sigma_w$ for all $t\ge2$. To obtain an upper bound, we leverage the optimality of the Kalman filter $\hat{x}_t^{\tt s(i)}$. Specifically, we recall from \eqref{eqn:KF optimization problem} that the Kalman filter $\hat{x}_t^{\tt KF,s(i)}$ with the Kalman gain $l_t^{\tt s(i)}$ given by \eqref{eqn:KF for scalar system in lower bound proof} achieves the minimum MSEE. Hence, if we replace the optimal Kalman gain $l_t^{\tt s(i)}$ in \eqref{eqn:KF for scalar system in lower bound proof} with any other gain, e.g., $\frac{r_i}{5}$, we get a suboptimal filter denoted as $\tilde{x}_t^{\tt s(i)}$ that satisfies
\begin{equation*}
\tilde{x}_{t+1}^{\tt s(i)}=\frac{r_i}{5}y_t,
\end{equation*}
initialized with $\tilde{x}_0^{\tt s(i)}=0$. We know from e.g., \cite[Chapter~4]{anderson2005optimal} that the error covariance of this suboptimal filter satisfies that $\E_{w,v}\big[(\tilde{x}_t^{\tt s(i)}-x_t^{\tt s(i)})^2\big]=\tilde{\sigma}_t^{\tt s(i)}$, where $\tilde{\sigma}_t^{\tt s(i)}$ is given by 
\begin{equation*}
\tilde{\sigma}_t^{\tt s(i)}=\frac{r_i^2\sigma_v}{25}+\sigma_w,\ \forall t\ge1
\end{equation*}
with $\tilde{\sigma}_t^{\tt s(i)}=0$. Since $\tilde{x}_t^{\tt s(i)}$, we have $\sigma_t^{\tt s(i)}\le\tilde{\sigma}_t^{\tt s(i)}$ for all $t\ge1$. 

Using the above lower and upper bounds on $\sigma_t^{\tt s(i)}$ for all $t\ge1$ in \eqref{eqn:sigma_t(sigma_v)-sigma_t(0)}, we obtain that for any $t\ge2$, 
\begin{align}\nonumber
\sigma_t^{\tt s(i)}(\sigma_v)-\sigma_t^{\tt s(i)}(0)&\le\frac{\sigma_vr_i\left(\frac{r_i^2\sigma_v}{25}+\sigma_w\right)}{25\sigma_w+25r_i\sigma_v}\\\nonumber
&=\frac{\frac{r_i^2}{25T^2}+\frac{r_i\sigma_w}{T}}{25\sigma_w+\frac{r_i}{25T}}\\\nonumber
&\le\frac{\frac{r_i^3}{25}+r_i\sigma_w}{25\sigma_w}\cdot\frac{1}{T}\\
&=\frac{r_i^3+25r_i\sigma_w}{625\sigma_wT},\label{eqn:sigma_t(sigma_v)-sigma_t(0) upper bound}
\end{align}
where we plug in the choice of $\sigma_v=1/T$. Since \eqref{eqn:sigma_t(sigma_v)-sigma_t(0) upper bound} holds for all $i\in\{1,2,3\}$, and \eqref{eqn:sigma_t(sigma_v)-sigma_t(0)} directly implies that $\sigma_t^{\tt s(i)}(\sigma_v)-\sigma_t^{\tt s(i)}(0)\ge0$ for all $i\in\{1,2,3\}$ and all $t\ge2$, we get from \eqref{eqn:KF error difference in the lower bound proof} that for any $t\ge2$,
\begin{align*}
\E_{w,v}\Big[(\hat{x}_t^{\tt KF,s(1)}-x_t^{\tt s(1)})^2-(\hat{x}_t^{\tt KF,s(i)}-x_t^{\tt s(i)})^2\Big]&=\sigma_t^{\tt s(1)}(\sigma_v)-\sigma_t^{\tt s(i)}(\sigma_v)\\
&\ge\sigma_t^{\tt s(1)}(0)-\sigma_t^{\tt s(i)}(0)-\frac{r_i^3+25r_i\sigma_w}{625\sigma_wT}\\
&\ge-\frac{r_i^3+25r_i\sigma_w}{625\sigma_w T},
\end{align*}
where the second inequality uses the relation $\sigma_t^{\tt s(1)}(0)=\sigma_t^{\tt s(i)}(0)=0$ as we argued above. Finally, we notice from \eqref{eqn:scalar DARE in lower bound proof} that $\sigma_0^{\tt s(i)}(\sigma_v)=0$ and $\sigma_1^{\tt s(i)}(\sigma_v)=\sigma_w$ for all $\sigma_v\in\R_{\ge0}$, which implies that $\sigma_0^{\tt s(i)}(\sigma_v)-\sigma_0^{\tt s(i)}(0)=\sigma_1^{\tt s(i)}(\sigma_v)-\sigma_1^{\tt s(i)}(0)=0$, completing the proof of the lemma.
\end{proof}

Applying Lemmas~\ref{lemma:lower bound term (i) in R^i-R^1}-\ref{lemma:lower bound term (ii) in R^i-R^1} to \eqref{eqn:decomposition of R_2-R_1}, we obtain that for any $t\ge1$, 
\begin{align*}
R_t^{\tt s(2)}-R_t^{\tt s(1)}+R_t^{\tt s(3)}-R_t^{\tt s(1)}&\ge\E_{w,v}\Big[\E_{w,v}\Big[2(2-r_2-r_3)x_t^{\tt s(1)}\hat{x}_t\Big|y_0,\dots,y_{t-1}\Big]\Big]\\
&\qquad+(r_2^2+r_3^2-2)\sigma_w-\frac{r_2^3+r_3^3+25\sigma_w(r_2+r_3)}{625\sigma_w T}\\
&=18\sigma_w-\frac{56+50\sigma_w}{625\sigma_w T}.
\end{align*}
Recalling that $R_t^{\tt s(1)}=\E_{w,v}\big[(\hat{x}_t^{\tt s(i)}-x_t^{\tt s(i)})^2-(\hat{x}_t^{\tt KF,s(i)}-x_t^{\tt s(i)})^2\big]$ for all $t\ge0$, we know by the optimality of the Kalman filter (i.e., \eqref{eqn:KF optimization problem}) that $R_t^{\tt s(1)}\ge0$. Going back to \eqref{eqn:def of R_i}, we now get that for any $t\ge1$,
\begin{align}\nonumber
\E\big[(\hat{x}_t-x_t)^2-(\hat{x}_t^{\tt KF}-x_t)^2\big]&=\frac{1}{3}\big(R_t^{\tt s(2)}-R_t^{\tt s(1)}+R_t^{\tt s(3)}-R_t^{\tt s(1)}+3R_t^{\tt s(1)}\big)\\
&\ge6\sigma_w-\frac{56+50\sigma_w}{1875\sigma_w T}.\label{eqn:time t regret in regret lower bound}
\end{align}
Finally, noticing that the initial values satisfy $\hat{x}_0=\hat{x}_t^{\tt s(i)}=x_0=0$ and summing \eqref{eqn:time t regret in regret lower bound} over all $t\in\{1,\dots,T-1\}$, we complete the proof of \eqref{eqn:expected lower bound} in the theorem.

Notice that we have proved \eqref{eqn:expected lower bound} holds for any algorithm that only has access to $y_0,\dots,y_{t-1}$ when designing $\hat{x}_t$ for all $t\ge0$, and \eqref{eqn:expected lower bound} holds even when the algorithm is applied to the class of stable scalar LTI systems given by \eqref{eqn:LTI in lower bound proof}. Thus, to prove \eqref{eqn:high probability lower bound}, we consider any algorithm that returns $\hat{x}_t$ such that $\norm{\hat{x}_t}\le c_{\A}$ for all $t\ge0$, and also apply the algorithm to the class of systems in \eqref{eqn:LTI in lower bound proof}. Moreover, we see that to prove \eqref{eqn:high probability lower bound}, it suffices for us to lower bound ${\tt R_x}(T)-\E[{\tt R_x}(T)]$ (i.e., upper bound $\E[{\tt R_x}(T)]-{\tt R_x}(T)$) with a high probability. 
Towards this end, we first prove the following result.
\begin{lemma}
\label{lemma:upper bound on x_t lower bound system}
Consider system~\eqref{eqn:LTI in lower bound proof}, any $T\ge3$ and any $\delta\in(0,1)$. Then, under the event $\CE_w$ defined in \eqref{eqn:probabilistic event} for system~\eqref{eqn:LTI in lower bound proof}, it holds that
\begin{align*}
|x_t|^2\le125\sigma_w\log\frac{T}{\delta},\ \forall t\in\{0,\dots,T-1\}.
\end{align*}
In addition, it holds that 
\begin{align*}
\E\big[|x_t|^2\big]\le125\sigma_w(1+\log T),\ \forall t\in\{0,\dots,T-1\}.
\end{align*}
\end{lemma}
\begin{proof}
For any $t\ge0$, we derive from \eqref{eqn:LTI in lower bound proof} that 
\begin{align*}
x_t&=a^tx_0+\sum_{s=0}^{t-1}a^{t-(s+1)}w_t\\
&=\sum_{s=0}^{t-1}a^{t-(s+1)}r\tilde{w}_t,
\end{align*}
where we recall that $x_0=0$, and $r\in\{1,4,-2\}$ and $a=1/5$. It follows that for any $t\in\{0,\dots,T-1\}$,
\begin{align*}
|x_t|&\le\max_{0\le t\le T-1}|\tilde{w}_t|\sum_{s=0}^{t-1}4\left(\frac{1}{5}\right)^{t-(s+1)}\\
&\le5\max_{0\le t\le T-1}|\tilde{w}_t|,
\end{align*}
which also implies that 
\begin{align}
x_t^2\le25\max_{0\le t\le T-1}|\tilde{w}_t|^2.\label{eqn:upper bound on |x_t|^2 in the lower bound proof}
\end{align}
By the definition of the event $\CE_w$ defined in \eqref{eqn:probabilistic event}, we know that under the event $\CE_w$,
\begin{equation*}
|\tilde{w}_t|\le\sqrt{5\sigma_w\log\frac{3T}{\delta}},\ \forall t\in\{0,\dots,T-1\}. 
\end{equation*}
It follows from \eqref{eqn:upper bound on |x_t|^2 in the lower bound proof} that
\begin{align}\nonumber
x_t^2&\le125\sigma_w\log\frac{3T}{\delta},\ \forall t\in\{0,\dots,T-1\}.
\end{align}

From \eqref{eqn:upper bound on |x_t|^2 in the lower bound proof}, we also get that 
\begin{align*}
\E\big[x_t^2\big]\le25\E\Big[\max_{0\le t\le T-1}|\tilde{w}_t|^2\Big],
\end{align*}
which implies via Lemma~\ref{lemma:expected gaussian concentration bound} that 
\begin{align*}
\E\big[x_t^2\big]\le125\sigma_w(1+\log T),\ \forall t\in\{0,\dots,T-1\}.
\end{align*}
\end{proof}

Now, considering the term in $\E[{\tt R_x}(T)]-{\tt R_x}(T)$ corresponding to any $t\in\{0,\dots,T-1\}$ as per \eqref{eqn:E[R_x(T)]}, we have
\begin{align}\nonumber
&\Big|\E\big[(\hat{x}_t-x_t)^2-(\hat{x}_t^{\tt KF}-x_t)^2\big]-\big((\hat{x}_t-x_t)^2-(\hat{x}_t^{\tt KF}-x_t)^2\big)\Big|\\\nonumber
\le&2\E\big[\hat{x}_t^2+x_t^2\big]+2\big(\hat{x}^{\tt KF}\big)^2+2x_t^2\\
\le&2c_{\A}+2\E\big[x_t^2\big]+2\big(\hat{x}_t^{\tt KF}\big)^2+2x_t^2.\label{eqn:upper bound on ER_x-R_x single term}
\end{align}
From Lemma~\ref{lemma:upper bound on hat x^KF and hat x^SK} used in the proof of Lemma~\ref{lemma:upper bound on term (ii) in output truncation regret}, one can show that $|\hat{x}_t^{\tt KF}|= \CO(\sqrt{\log T}t)$ for all $t\in\{0,\dots,T-1\}$, under the event $\CE_w\cap\CE_v$, where $\CE_w$ and $\CE_v$ are given by \eqref{eqn:probabilistic event} defined for system~\eqref{eqn:LTI in lower bound proof}. Moreover, using similar arguments to those for \eqref{eqn:upper bound on hat y_t^KF}, one can show that $|\hat{x}_t^{\tt KF}|=\CO(\sqrt{\log T})$ for all $t\in\{2s_0,\dots,T-1\}$, under the event $\CE_w\cap\CE_v$, where $s_0$ is given by \eqref{eqn:def of s_0} and satisfies $s_0=\CO(\log T)$. Combining the above arguments, we see that under the event $\CE_w\cap\CE_v$, $|\hat{x}_t^{\tt KF}|=\CO(\log^{1.5}T)$ holds for all $t\in\{0,\dots,T-1\}$. Applying Lemma~\ref{lemma:upper bound on x_t lower bound system}, we then get from \eqref{eqn:upper bound on ER_x-R_x single term} that under the event $\CE_w\cap\CE_v$, 
\begin{align*}
&\Big|\E\big[(\hat{x}_t-x_t)^2-(\hat{x}_t^{\tt KF}-x_t)^2\big]-\big((\hat{x}_t-x_t)^2-(\hat{x}_t^{\tt KF}-x_t)^2\big)\Big|\\
\le&2c_{\A}+250\sigma_w\log\frac{3T}{\delta}+250\sigma_w(1+\log T)+\CO(\log^{1.5}T)\\
=&\CO(c_{\A}+\log^{1.5}T).
\end{align*}
One can now apply Lemma~\ref{lemma:Azuma} and deduce that under the event $\CE_w\cap\CE_v$, the following holds with probability at least $1-\delta/3$:
\begin{align*}
&\E\big[{\tt R_x}(T)\big]-{\tt R_x}(T)\\
=&\sum_{t=0}^{T-1}\Big(\E\big[(\hat{x}_t-x_t)^2-(\hat{x}_t^{\tt KF}-x_t)^2\big]-\big((\hat{x}_t-x_t)^2-(\hat{x}_t^{\tt KF}-x_t)^2\big)\Big)\\
\le&\sqrt{2T\CO(c_{\A}+\log^{1.5}T)\log\frac{3}{\delta}}.
\end{align*}
We now obtain from a union bound and \eqref{eqn:expected lower bound} that the following holds with probability at least $1-\delta$.
\begin{align*}
{\tt R_x}(T)\ge6T\sigma_w-\frac{56+50\sigma_w}{1875\sigma_w}-{\CO}\left(\sqrt{(c_{\A}+\log^{1.5}T)T\log\frac{3}{\delta}}\right).
\end{align*}
\hfill$\blacksquare$

\subsection{Proof of Lemma~\ref{lemma:properties of f_t;h state estimation}}
First, noting that $\nabla^2f_t(M)$ shares the same expression with $\nabla^2f_t(N)$ in the output estimation problem for any $M$ and any $N$ (see our discussions in Section~\ref{sec:cost functions}), we know from our arguments in the proof of Lemma~\ref{lemma:hessian of f_t:k(N)} that $\nabla^2f_{t;h}(N)\succeq\alpha_0$ for all $N\in\R^{nph}$, which proves that $f_{t;h}(\cdot)$ is $\alpha_0$-strongly convex. Moreover, we know from Lemma~\ref{lemma:properties of f_t(M)} that $f_t(\cdot)$ in the state estimation problem is $l_x$-Lipschitz under the event $\CE_w\cap\CE_v$. It follows from our arguments in the proof of Lemma~\ref{lemma:hessian of f_t:k(N)} that $f_{t;h}(\cdot)$ is $l_x$-Lipschitz under the event $\CE_w\cap\CE_v$.\hfill$\blacksquare$

\subsection{Proof of Lemma~\ref{lemma:upper bound on f_t;h difference state estimation in set S}}
Consider the function $f_{t;h}(\cdot)$ defined in Definition~\ref{def:conditional cost state estimation} for the state estimation problem. Note that the update rule in line~9 of Algorithm~\ref{alg:OCO for state estimation} is based on $\nabla\tilde{f}_t(M_t)$ (rather than the true gradient $\nabla f_t(M_t)$). Now, following similar arguments in the proof of Lemma~\ref{lemma:upper bound on f_t;k difference output estimation} leading up to \eqref{eqn:upper bound on nabla f_t;h dot N_t-N*} (replacing $N$ with $M$ and replacing $\nabla f_t(\cdot)$ with $\nabla\tilde{f}_t(\cdot)$), one can get that 
\begin{align*}
&\sum_{t_j\in S^{\prime}}\big(f_{t_j;h}(M_{t_j})-f_{t_j;h}(M^{\star})\big)\\
\le&\sum_{t_j\in S^{\prime}}\big(\nabla f_{t_j;h}(M_{t_j})^{\top}(M_{t_j}-M^{\star})-\frac{\alpha_0}{2}\norm{M_{t_j}-M^{\star}}\big)\\
\le&\frac{1}{2}\sum_{t_j\in S^{\prime}}\Big(\frac{1}{\eta_{j+1}}-\frac{1}{\eta_j}-\alpha_0\Big)\norm{M_{t_j}-M^{\star}}+\frac{\norm{M_{t_1}-M^{\star}}^2}{2\eta_{1}}\\
&\qquad\qquad+\sum_{t_j\in S^{\prime}}\frac{\eta_j}{2}\norm{\nabla\tilde{f}_{t_j}(M_{t_j})}^2-\sum_{t_j\in S^{\prime}}\langle\nabla\tilde{f}_{t_j}(M_{t_j})-\nabla f_{t_j;h}(M_{t_j}),M_{t_j}-M^{\star}\rangle\\
=&-\frac{\alpha_0}{4}\sum_{t_j\in S^{\prime}}\norm{M_{t_j}-M^{\star}}^2+\frac{R_M^2\alpha_0}{2}+\sum_{t_j\in S^{\prime}}\frac{\eta_j}{2}\norm{\nabla\tilde{f}_{t_j}(M_{t_j})}^2\\
&\qquad\qquad\qquad\qquad-\sum_{t_j\in S^{\prime}}\langle\varepsilon_{t_j}^s,M_{t_j}-M^{\star}\rangle-\sum_{t_j\in S^{\prime}}\langle\tilde{\varepsilon}_{t_j}^s,M_{t_j}-M^{\star}\rangle,
\end{align*}
where the last inequality follows from the choice of the step size $\eta_j=\frac{2}{\alpha_0j}$ for all $j\ge1$ and the upper bound on the norm of any $M\in\K_M$ shown by Lemma~\ref{lemma:properties of f_t(M)}.\hfill$\blacksquare$

\subsection{Proof of Lemma~\ref{lemma:upper bound on f_t difference state estimation}}
First, for $t=0,1,\dots$ in Algorithm~\ref{alg:OCO for state estimation}, let $\theta_t=1$ if $t\mod\tau=b_i$ (i.e., if Algorithm~\ref{alg:OCO for output estimation} performs the update in line~9), and let $\theta_t=0$ if otherwise. It then follows from the definition of $S^{\prime}$ in \eqref{eqn:set S prime for state estimation} that the left-hand side of the inequality shown in Lemma~\ref{lemma:upper bound on f_t difference in set S} can be written as
\begin{align*}
\sum_{t_j\in S^{\prime}}\big(f_{t_j}(M_{t_j})-f_{t_j}(M^{\star})\big)=\sum_{t=\tau}^{T-1}\big(f_t(M_t)-f_t(M^{\star})\big)\theta_t.
\end{align*}
Taking the expectation $\E_b[\cdot]$ with respect to $\{b_i\}_{i\ge0}$ sampled in line~3 of Algorithm~\ref{alg:OCO for state estimation}, we obtain that 
\begin{align}\nonumber
\E_b\Big[\sum_{t_j\in S^{\prime}}\big(f_{t_j}(M_{t_j})-f_{t_j}(M^{\star})\big)\Big]&=\E_b\Big[\sum_{t=\tau}^{T-1}\big(f_t(M_t)-f_t(M^{\star})\big)\theta_t\Big]\\\nonumber
&\overset{(a)}{=}\sum_{t=\tau}^{T-1}\E_b\Big[f_t(M_t)-f_t(M^{\star})\Big]\E_b[\theta_t]\\
&\overset{(b)}{=}\E_b\Big[\sum_{t=\tau}^{T-1}\big(f_t(M_t)-f_t(M^{\star})\big)\Big]\frac{1}{\tau}.\label{eqn:expand f_t from set S to all t}
\end{align}
To obtain $(a)$, we first recall that $b_i\overset{i.i.d.}{\sim}{\tt Unif}(\{0,\dots,\tau-1\})$ (and thus $\{\theta_t\}_{t\ge0}$) is assumed to be independent of $\{w_t\}_{t\ge0}$, $\{v_t\}_{t\ge0}$ and $\{\tilde{v}_t\}_{t\ge0}$.  Since $M_t$ is measurable with respect to the sigma-field $\sigma(w_0,\dots,w_{t-1},v_0,\dots,v_{t-1},\tilde{v}_t,\dots,\tilde{v}_{t-1},b_0,\dots,b_{\lfloor t/\tau\rfloor-1})$ and $\theta_t$ is independent of $b_0,\dots,b_{\lfloor t/\tau\rfloor-1}$, we conclude that $\theta_t$ is independent of $M_t$ for all $t\in\{k,\dots,T-1\}$. To obtain $(b)$, we use the facts that $\theta_t\in\{0,1\}$ and $\P(\theta_t=1)=1/\tau$.

Similarly, we take the expectation $\E_b[\cdot]$ on each term on the right-hand side of the inequality shown in Lemma~\ref{lemma:upper bound on f_t difference in set S}. Using similar arguments to those for \eqref{eqn:expand f_t from set S to all t}, one can show that 
\begin{align*}
\E_b\Big[\sum_{t_j\in S^{\prime}}X_{t_j}(M^{\star})\Big]&=\E_b\Big[\sum_{t=\tau}^{T-1}X_t(M^{\star})\Big]\frac{1}{\tau}\\
\E_b\Big[\sum_{t_j\in S^{\prime}}\langle\tilde{\varepsilon}_{t_j}^s,M_{t_j}-M^{\star}\rangle\Big]&=\E_b\Big[\sum_{t=\tau}^{T-1}\langle\tilde{\varepsilon}_t^s,M_t-M^{\star}\rangle\Big]\frac{1}{\tau},
\end{align*}
where we similarly use the independence of $\theta_t$ from $X_t(M^{\star})$ and $\tilde{\varepsilon}_t^s$.\hfill$\blacksquare$

\subsection{Proof of Lemma~\ref{lemma:upper bound on nabla tilde f_t}}
Suppose throughout this proof that the event $\CE_w\cap\CE_v$ holds and consider any $\delta\in(0,1)$. 

{\it Proof of part~(i).} The proof of part~(i) follows directly from that of Lemma~\ref{lemma:high prob bound on X_t output estimation}. Specifically, we observe that if we replace $N$ with $M$, instead define the time indices $t_{i,j}=k+i+(j-1)h$ for all $i\in\{0,\dots,h-1\}$ and all $j\in[T_i]$, where $T_i=\max\{j:t_{i,j}\le T-1\}$, and replace $\delta$ with $\delta/4$, the proof of Lemma~\ref{lemma:high prob bound on X_t output estimation} still holds, which yields 
\begin{align*}
\sum_{t=\tau}^{T-1}X_t(M^{\star})\le\frac{32l_x^2h\log\frac{12h}{\delta}}{\alpha_0}+\frac{\alpha_0}{4}\sum_{t=0}^{T-1}\norm{M_t-M^{\star}}.
\end{align*}
Since the above inequality holds for any realization of the random variables $b_0,\dots,b_{\lfloor T/\tau\rfloor}$, taking the expectation $\E_b[\cdot]$ with respect to $b_0,\dots,b_{\lfloor T/\tau\rfloor}$ on both sides of the above inequality proves part~(i).

{\it Proof of part~(ii).}
First, we get from $\tilde{f}_{t_j}(M_{t_j})=\norm{x_{t_j}+\tilde{v}_{t_j}-\hat{x}_{t_j}(M_{t_j})}^2$ that
\begin{align}\nonumber
\nabla\tilde{f}_{t_j}(M_{t_j})&=2Y_{t_j-1:t_j-h}^{\top}(Y_{t_j-1:t_j-h}M_{t_j}-x_{t_j}-\tilde{v}_{t_j})\\
&=\nabla f_{t_j}(M_{t_j})-2Y_{t_j-1:t_j-h}\tilde{v}_{t_j},\label{eqn:nabla tilde f_t}
\end{align}
where $Y_{t_j-1:t_j-h}$ is defined in \eqref{eqn:Y_t-h:t-1} and the second equation follows from \eqref{eqn:gradient of f_t(M_t)}. Now, under the event $\CE_{\tilde{v}}$ defined in \eqref{eqn:event E_tilde v}, we have $\max_{0\le t\le T-1}\norm{\tilde{v}_t}\le R_{\tilde{v}}$, which implies
\begin{align}\nonumber
\norm{\nabla\tilde{f}_{t_j}(M_{t_j})}^2&\le\big(\norm{\nabla f_{t_j}}(M_{t_j})+2\norm{Y_{t-1:t-h}}\norm{\tilde{v}_{t_j}}\big)^2\\
&\le(l_x+2\sqrt{h}R_yR_{\tilde{v}})^2,\ \forall t_j\in S,\label{eqn:upper bound on nabla f_tilde in set S}
\end{align}
where we also use the upper bound $\norm{Y_{t-1:t-h}}\le\sqrt{h}R_y$ shown in Lemma~\ref{lemma:upper bound on state and output} and $\norm{\nabla\tilde{f}_t(M_{t_j})}\le l_x$ shown in Lemma~\ref{lemma:properties of f_t(M)}.

Thus, we get that under the event $\CE_{\tilde{v}}$,
\begin{align*}
\sum_{t_j\in S^{\prime}}\frac{\eta_j}{2}\norm{\nabla\tilde{f}_{t_j}(M_{t_j})}^2&\le(l_x+2\sqrt{h}R_y)^2\sum_{t_j\in S^{\prime}}\frac{\eta_j}{2}\\
&=\frac{(l_x+2\sqrt{h}R_yR_{\tilde{v}})^2}{\alpha_0}\sum_{j=1}^{\lfloor T/\tau\rfloor-1}\frac{1}{j}\\
&\le\frac{(l_x+2\sqrt{h}R_yR_{\tilde{v}})^2}{\alpha_0}(1+\log\lfloor T/\tau\rfloor),
\end{align*}
where the equality follows from the choice of the step size $\eta_j=\frac{1}{\alpha_0j}$ for all $j\ge1$ and the fact $|S^{\prime}|=\lfloor T/\tau\rfloor-1$. Since the above inequality holds for any realization of the random variables $b_0,\dots,b_{\lfloor T/\tau\rfloor}$, taking the expectation $\E_b[\cdot]$ with respect to $b_0,\dots,b_{\lfloor T/\tau\rfloor}$ on both sides of the above inequality completes the proof of part~(ii).
 
{\it Proof of part~(iii).}  From the definition of $S^{\prime}$ in \eqref{eqn:set S prime for state estimation} and the assumption that $\tau\ge h$, we see that from $M_{t_j-h}$ to $M_{t_j}$ for any $t_j\in S^{\prime}$, Algorithm~\ref{alg:OCO for state estimation} performs the update in line~9 at most once at the step $t_{j-1}$. Thus, we have 
\begin{align*}
\norm{M_{t_j}-M_{t_j-h}}&\overset{(a)}{\le}\norm{\eta_{j-1}\nabla\tilde{f}_{t_{j-1}}(M_{t_{j-1}})}\\
&\overset{(b)}{\le}\frac{1}{\alpha_0(j-1)}(l_x+2\sqrt{h}R_yR_{\tilde{v}}),\ \forall 2\le j\le\lfloor T/\tau\rfloor,
\end{align*}
where $(a)$ follows from line~9 of Algorithm~\ref{alg:OCO for state estimation} and \cite[Proposition~2.2]{bansal2019potential}; $(b)$ follows from the choice of the step size $\eta_j=\frac{1}{\alpha_0j}$ for all $j\ge1$ and the upper bound shown in \eqref{eqn:upper bound on nabla f_tilde in set S}. Also note that $\norm{M_{t_1}-M_{t_1-h}}\le\norm{\eta_0\nabla\tilde{f}_t(M_{t_1-h})}=0$ since $\eta_0=0$. Summing over all $t\in S^{\prime}$, we get
\begin{align*}
\sum_{t_j\in S^{\prime}}\norm{M_{t_j}-M_{t_j-h}}&\le\frac{l_x+2\sqrt{h}R_yR_{\tilde{v}}}{\alpha_0}\sum_{j=2}^{\lfloor T/\tau\rfloor-1}\frac{1}{j-1}\\
&\le\frac{l_x+2\sqrt{h}R_yR_{\tilde{v}}}{\alpha_0}(1+\log\lfloor T/\tau\rfloor).
\end{align*}
Since the above inequality holds for any realization of the random variables $b_0,\dots,b_{\lfloor T/\tau\rfloor}$, taking the expectation $\E_b[\cdot]$ with respect to $b_0,\dots,b_{\lfloor T/\tau\rfloor}$ on both sides of the above inequality completes the proof of part~(iii).

{\it Proof of part~(iv).} We show that $\{-\langle\tilde{\varepsilon}_t^s,M_t-M^{\star}\rangle\}_{t\ge \tau}$ form a martingale difference sequence. For any $t\in\{\tau,\dots,T-1\}$, we have
\begin{align*}
\E_{\tilde{v}}\Big[\E_b\Big[-\langle\tilde{\varepsilon}_t^s,M_t-M^{\star}\rangle\Big]\Big]&=\E_b\Big[-\E_{\tilde{v}}\langle\nabla\tilde{f}_t(M_t)-\nabla f_t(M_t),M_t-M^{\star}\rangle\Big]\\
&=\E_b\Big[\E_{\tilde{v}}\Big[\langle2Y_{t-1:t-h}\tilde{v}_t,M_t-M^{\star}\rangle\Big]\Big]\\
&\overset{(a)}{=}\E_b\Big[\langle2Y_{t-1:t-h}\E_{\tilde{v}}[\tilde{v}_t],\E_{\tilde{v}}[M_t]-M^{\star}\rangle\Big]\overset{(b)}{=}0,
\end{align*}
where $\E_{\tilde{v}}[\cdot]$ denotes the expectation with respect to $\tilde{v}_0,\dots,\tilde{v}_{T-1}$. To obtain $(a)$, we use the fact that $\{\tilde{v}_t\}_{t\ge0}$ is independent of $M_t$ and $\{y_t\}_{t\ge0}$ (and thus $Y_{t-1:t-h}$); and $(b)$ follows from the fact that $\tilde{v}_t\overset{i.i.d.}{\sim}\CN(0,\tilde{V})$. In addition, we have
\begin{align*}
\E_b\Big[\big|-\langle\tilde{\varepsilon}_t^s,M_t-M^{\star}\rangle\big|\Big]&\le\E_b\Big[\norm{\tilde{\varepsilon}_t^s}\norm{M_t-M^{\star}}\Big]\\
&\le\E_b\Big[2\norm{Y_{t-1:t-h}}\norm{\tilde{v}_t}\norm{M_t-M^{\star}}\Big]\\
&\le4\sqrt{h}R_yR_{\tilde{v}}R_M,
\end{align*}
where the last inequality again uses Lemmas~\ref{lemma:upper bound on state and output}-\ref{lemma:properties of f_t(M)} and the fact $M_t,M^{\star}\in\K_M$. Hence, applying Lemma~\ref{lemma:Azuma}, we conclude that further under the event $\CE_{\tilde{v}}$, it holds with probability at least $1-\delta/12$ that
\begin{align*}
\sum_{t=\tau}^{T-1}\E_b\Big[-\langle\tilde{\varepsilon}_t^s,M_t-M^{\star}\rangle\Big]\le4\sqrt{h}R_yR_{\tilde{v}}R_M\sqrt{T\log\frac{24}{\delta}}.
\end{align*}
\hfill$\blacksquare$

\subsection{Proof of Lemma~\ref{lemma:upper bound on f_t difference state estimation expected to true}} 
First, for any $t\in\{\tau,\dots,T-1\}$, we get from Lemma~\ref{lemma:properties of f_t(M)} that under the event $\CE_w\cap\CE_v$, 
\begin{align*}
\big|f_t(M_t)-f_t(M^{\star})-\E_b\big[f_t(M_t)-f_t(M^{\star})\big]\big|&\le8R_x^2+8hR_y^2\min\{p,n\}\frac{\kappa_F^2\norm{L}^2}{(1-\gamma_F)^2}\\
&=8R_x^2+8hR_y^2R_M^2.
\end{align*}
Applying Lemma~\ref{lemma:Azuma}, we obtain that under the event $\CE_w\cap\CE_v$, \eqref{eqn:relate f_t difference to f_t expected difference} holds with probability at least $1-\delta/12$.\hfill$\blacksquare$

\subsection{Doubling Trick to Remove the Knowledge of $T$ in Algorithm~\ref{alg:OCO for state estimation}}\label{app:doubling trick for state estimation}
Following the same arguments in Appendix~\ref{app:ddoubling trick for output estimation}, we divide the total horizon $T$ into consecutive intervals with length $L_i=T_i-T_{i-1}$, where $T_i=2^{2^i}$ with $T_{-1}=0$, apply Algorithm~\ref{alg:OCO for state estimation} to each interval, and sum the resulting regret over all the intervals. Moreover, we have 
\begin{align*}
{\tt R_x}(T)&=\sum_{i=0}^M{\tt R_x}(L_i)\\
&=\sum_{i=0}^M\CO\big(\tau(\log^4L_i)\log(L_i/\tau)+\sqrt{L_i}\log^2L_i\big)\\
&\le\sum_{i=0}^M\Big(\CO(\tau\log(T/\tau)\log^4 L_i)+\CO(\sqrt{T_i}\log^2T_i)\Big)\\
&\overset{(a)}{\le}\CO(\tau\log(T/\tau)\log^4T)+\CO(\sqrt{T})\sum_{i=0}^M\CO(2^{2i})\\
&=\CO(\tau\log(T/\tau)\log^4T)+\CO(\sqrt{T})\CO(2^{2\log\log T})\\
&\overset{(b)}{=}\CO\big(\tau\log(T/\tau)\log^4T+\sqrt{T}\log^2T\big),
\end{align*}
where $(a)$ follows from the arguments in Appendix~\ref{app:ddoubling trick for output estimation} and $(b)$ uses the fact $M=\CO(\log\log T)$.

\section{Technical Lemmas}\label{app:aux results}
\subsection{Statistical Inequalities}
\begin{lemma}\cite[Lemma~14]{cassel2021online}
\label{lemma:gaussian concentration bound}
Let $w\sim\CN(0,\Sigma_w)$ with $w\in\R^d$ and $\Sigma_w\succeq0$. For any $z\ge1$, it holds that 
\begin{equation*}
\P(\norm{w}>\sqrt{5\Tr(\Sigma_w)z})\le e^{-z}.
\end{equation*}
Moreover, for any $\delta\in(0,1/e)$, it holds with probability at least $1-\delta$ that
\begin{equation*}
\norm{w}\le\sqrt{5\Tr(\Sigma_w)\log\frac{1}{\delta}}.
\end{equation*}
\end{lemma}

\begin{lemma}
\label{lemma:expected gaussian concentration bound}
Let $w_t\overset{i.i.d.}{\sim}\CN(0,\Sigma_w)$ with $w_t\in\R^d$ and $\Sigma_w\succeq0$ for $t=0,\dots,T-1$. Then, for any $T\ge1$,  
\begin{align*}
&\E\Big[\max_{0\le t\le T-1}\norm{w_t}^2\Big]\le5\Tr(\Sigma_w)(1+\log T).
\end{align*}
\end{lemma}
\begin{proof}
First, we get from Lemma~\ref{lemma:gaussian concentration bound} and a union bound over all $t\in\{0,\dots,T-1\}$ that 
\begin{align*}
\P\Big(\max_{0\le t\le T-1}\norm{w_t}^2>x\Big)\le T\exp\Big(-\frac{x}{5\Tr(\Sigma_w)}\Big),
\end{align*}
for any $x\ge 5\Tr(\Sigma_w)\log T$. It then follows that 
\begin{align*}
\E\Big[
\max_{0\le t\le T-1}\norm{w_t}^2\Big]&=\int_{0}^{\infty}\P\Big(\max_{0\le t\le T-1}\norm{w_t}^2>x\Big)dx\\
&\le \int_{0}^{5\Tr(\Sigma_w)\log T}1 dx+\int_{5\Tr(\Sigma_w)\log T}^{\infty}T\exp\Big(-\frac{x}{5\Tr(\Sigma_w)}\Big)dx\\
&=5\Tr(\Sigma_w)\log T+5\Tr(\Sigma_w)T\int_{\log T}^{\infty}e^{-y}dy\qquad\Big(\text{Let}\ y=\frac{x}{5\Tr(\Sigma_w)}\Big)\\
&=5\Tr(\Sigma_w)\log T+5\Tr(\Sigma_w).
\end{align*}
\end{proof}

\begin{lemma}
\label{lemma:Azuma}\cite[Corollary~3.9]{van2014probability}
Let $\{X_k\}_{k\ge0}$ be a martingale difference sequence adapted to a filtration $\{\F_k\}_{k\ge0}$. Suppose $|X_k|\le c_k$ for all $k\ge0$. Then, for any $s\ge1$ and any $\epsilon\in\R$, it holds that
\begin{equation*}
\P\Big(\sum_{k=1}^sX_k>\epsilon\Big)\le\exp\Big(-\frac{\epsilon^2}{2\sum_{k=1}^sc_k^2}\Big).
\end{equation*}
Equivalently, for any $0<\delta<1$ and any $\varepsilon\in\R$, it holds with probability at least $1-\delta$ that
\begin{align*}
\sum_{k=1}^sX_k\le\sqrt{\log\frac{1}{\delta}\sum_{k=1}^s2c_k^2}.
\end{align*}
\end{lemma}

\subsection{Kalman Filter Properties}
\begin{lemma}
\label{lemma:strongly stable of L}
Consider system~\eqref{eqn:LTI} and suppose Assumption~\ref{ass:pairs} holds. Let $\kappa_F$ and $\gamma_F$ be given by \eqref{eqn:def of gamma_F and kappa_F}. Then, it holds that $0<\gamma_F<1$ and $\norm{(A-LC)^k}\le\kappa_F\gamma_F^k$ for all $k\in\BZ_{\ge0}$, where $L$ is given by \eqref{eqn:Kalman gain} and satisfies $\norm{L}\le\kappa_F$.
\end{lemma}
\begin{proof}

First, we note from \eqref{eqn:known constants} that $W\succeq\alpha_0I_n$ and $\Sigma$ given by \eqref{eqn:Riccati steady} satisfies $\Sigma\preceq\bar{\sigma}I_n$, which also implies via~\eqref{eqn:known constants} that $W\succeq(\alpha_0/\bar{\sigma})\Sigma=2(1-\gamma_F)\Sigma$. Since we know from \cite[Chapter~4]{anderson2005optimal} that $\Sigma$ also satisfies the following equation:
\begin{align}
\Sigma=(A-LC)\Sigma(A-LC)^{\top}+W+LVL^{\top},\label{eqn:equivalent steady-state DARE}
\end{align}
we get
\begin{align*}
\Sigma\succeq(A-LC)\Sigma(A-LC)^{\top
}+2(1-\gamma_F)\Sigma,
\end{align*}
which further implies 
\begin{align*}
\Sigma^{-1/2}(A-LC)\Sigma(A-LC)^{\top}\Sigma^{-1/2}\preceq(2\gamma_F-1)I_n.
\end{align*}
Denoting $H=\Sigma^{1/2}$ and $\Lambda=\Sigma^{-1/2}(A-LC)\Sigma^{1/2}$, we have $A-LC=H\Lambda H^{-1}$. In the following, we will provide upper bounds on $\norm{H}$, $\lVert H^{-1}\rVert$ and $\norm{\Lambda}$.

Observing that 
\begin{align*}
\Sigma^{-1/2}(A-LC)\Sigma(A-LC)^{\top}\Sigma^{-1/2}=H^{-1}(H\Lambda H^{-1})HH(H\Lambda H^{-1})^{\top}H^{-1}=\Lambda\Lambda^{\top}\preceq(2\gamma_F-1)I_n,
\end{align*}
we obtain $\norm{\Lambda}\le\sqrt{2\gamma_F-1}\le\gamma_F$. Since $H=\Sigma^{1/2}$, we directly have $\norm{H}=\norm{\Sigma^{1/2}}=\sqrt{\bar{\sigma}}$. From \eqref{eqn:equivalent steady-state DARE}, we also know that $\Sigma\succeq W\succeq\alpha_0I_n$, which implies that $\Sigma^{-1/2}\preceq\sqrt{1/\alpha_0}I_n$ and thus $\lVert\Sigma^{-1/2}\rVert\le\sqrt{1/\alpha_0}$. Since $H^{-1}=\Sigma^{-1/2}$, we have $\lVert H^{-1}\rVert\le\sqrt{1/\alpha_0}$. Combining the above arguments, we see that for any $k\in\BZ_{\ge0}$, $\lVert(A-LC)^k\rVert=\lVert H\rVert\lVert H^{-1}\rVert\lVert\Lambda\rVert^k\le\kappa_F\gamma_F^k$. Moreover, since $V\succeq\alpha_0I_p$ by \eqref{eqn:known constants}, we get from \eqref{eqn:equivalent steady-state DARE} that $\Sigma\succeq LVL^{\top}\succeq\alpha_0LL^{\top}$ which implies $\lVert L\rVert\le\sqrt{\bar{\sigma}/\alpha_0}$.

Finally, our above arguments readily show that $\bar{\sigma}\ge\alpha_0$, which implies via $\gamma_F=1-\alpha_0/(2\bar{\sigma})$ by~\eqref{eqn:known constants} that $0<\gamma_F<1$.
\end{proof}

\begin{lemma}\label{lemma:KF convergence properties}
Consider the Kalman filter $\hat{x}_t^{\tt KF}$ for state estimation given by \eqref{eqn:KF} initialized with $\hat{x}_t^{\tt KF}=0$. In addition, consider the corresponding recursion of $\Sigma_t$ in \eqref{eqn:Riccati} initialized with $\Sigma_0=0$ and the Riccati equation in \eqref{eqn:Riccati steady}. For any $k,l\in\BZ_{\ge0}$, let
\begin{align*}
\Psi_{k,l}=(A-L_{k-1}C)(A-L_{k-2}C)\times\cdots\times(A-L_lC),
\end{align*}
if $k>l$, and let $\Psi_{k,l}=I_n$ if $k\le l$. Then, for any $t\ge1$,
\begin{align}
&\norm{\Psi_{t,1}}\le\sqrt{\frac{\norm{W+LVL^{\top}}\kappa_F^2}{\lambda_{\min}(W)(1-\gamma_F^2)}},\label{eqn:upper bound on Psi_t,1 with tilde_Sigma_t refined}\\
&\norm{\Sigma_t-\Sigma}\le\kappa_F^2\gamma_F^{t-1}\norm{W-\Sigma}\sqrt{\frac{\norm{W+LVL^{\top}}}{\lambda_{\min}(W)(1-\gamma_F^2)}},\label{eqn:convergence of Sigma_t}\\
&\norm{L_t-L}\le\frac{\kappa_F^3\gamma_F^{t}\norm{W-\Sigma}\norm{C}}{\lambda_{\min}(V)}\sqrt{\frac{\norm{W+LVL^{\top}}}{{\lambda_{\min}(W)(1-\gamma_F^2)}}},\label{eqn:convergence of L_t}
\end{align}
where $\Sigma$ is the solution to \eqref{eqn:Riccati steady} and $L$ is given by \eqref{eqn:Kalman gain}. 
\end{lemma}
\begin{proof}
We first prove \eqref{eqn:convergence of Sigma_t}.  It follows from the recursion of $\Sigma_t$ in \eqref{eqn:Riccati} initialized with $\Sigma_0=0$ that $\Sigma_1=W$. We then know from \cite[Chapter~4]{anderson2005optimal} that for any $t\ge1$,
\begin{align*}
\Sigma_t\succeq\Psi_{t,1}\Sigma_1\Psi_{t,1}^{\top}\succeq\lambda_{\min}(W)\Psi_{t,1}\Psi_{k,1}^{\top},
\end{align*}
which implies via Lemma~\ref{lemma:norm relation of PSD matrices} that 
\begin{align}
\lVert\Psi_{t,1}\rVert\lVert\Psi_{t,1}^{\top}\rVert\le\frac{\norm{\Sigma_t}}{\lambda_{\min}(W)}\implies\norm{\Psi_{t,1}}\le\sqrt{\frac{\norm{\Sigma_t}}{\lambda_{\min}(W)}},\label{eqn:upper bound on Psi_t,1 with tilde_Sigma_t}
\end{align}
where we also used fact $\lVert\Psi_{t,1}\rVert=\lVert\Psi_{t,1}^{\top}\rVert$. 

To get a uniform upper bound on $\norm{\Psi_{t,1}}$ over $t\ge1$, we invoke the optimality of Kalman filtering. Specifically, it is well-known (see, e.g., \cite{anderson2005optimal}) that $\Sigma_t=\E\big[(x_t-\hat{x}_t^{\tt KF})(x_t-\hat{x}_t^{\tt KF})^{\top}\big]$ for all $t\ge0$. Thus, the filter $\hat{x}_t^{\tt SK}$ given by the recursion $\hat{x}_{t+1}^{\tt SK}=(A-LC)\hat{x}_t^{\tt SK}+Ly_t$ initialized with $\hat{x}_0^{\tt SK}=0$ is a suboptimal filter. Denoting the error covariance of $\hat{x}_t^{\tt SK}$ as $\tilde{\Sigma}_t=\E\big[(x_t-\hat{x}_t^{\tt SK})(x_t-\hat{x}_t^{\tt SK})^{\top}\big]$, we obtain from the suboptimality of $\hat{x}_t^{\tt SK}$ that $\Sigma_t\preceq\tilde{\Sigma}_t$ for all $t\ge0$ \cite[Chapter~4]{anderson2005optimal}. Moreover, the error covariance $\tilde{\Sigma}_t$ satisfies the following recursion \cite[Chapter~4]{anderson2005optimal}:
\begin{align*}
\tilde{\Sigma}_{t+1}=(A-LC)\tilde{\Sigma}_t(A-LC)^{\top}+W+LVL^{\top},
\end{align*}
initialized with $\tilde{\Sigma}_0=0$. Unrolling this recursion, we get that for any $t\ge0$,
\begin{align}\nonumber
&\tilde{\Sigma}_t=\sum_{k=0}^{t-1}(A-LC)^k(W+LVL^{\top})\big((A-LC)^{\top}\big)^k\\\nonumber
\implies&\tilde{\Sigma}_t\preceq\norm{W+LVL^{\top}}\sum_{k=0}^{t-1}\big((A-LC)(A-LC)^{\top}\big)^k\\
\implies&\lVert\tilde{\Sigma}_t\rVert\le\norm{W+LVL^{\top}}\sum_{k=0}^{t-1}\norm{A-LC}^{2k}\le\norm{W+LVL^{\top}}\frac{\kappa_F^2}{1-\gamma_F^2},\label{eqn:upper bound on tilde_Sigma_t}
\end{align}
where the last inequality follows from similar arguments to those for Lemma~\ref{lemma:truncated KF}. Using \eqref{eqn:upper bound on tilde_Sigma_t} in \eqref{eqn:upper bound on Psi_t,1 with tilde_Sigma_t}, we obtain that \eqref{eqn:upper bound on Psi_t,1 with tilde_Sigma_t refined} holds for all $t\ge1$. 

Next, we prove \eqref{eqn:convergence of Sigma_t}. We adopt the following recursion of $\Sigma_t-\Sigma$ (see, e.g., \cite[Chapter~4]{anderson2005optimal}):
\begin{align*}
\Sigma_t-\Sigma=(A-LC)(\Sigma_{t-1}-\Sigma)(A-L_tC)^{\top}.
\end{align*}
Unrolling the above recursion from $t=1$, we further obtain that for any $t\ge1$,
\begin{align*}
\Sigma_t-\Sigma=(A-LC)^{t-1}(\Sigma_1-\Sigma)\Psi_{t,1}^{\top}.
\end{align*}
It follows from the above arguments that
\begin{align*}
\norm{\Sigma_t-\Sigma}&=\norm{(A-LC)^{t-1}}\norm{W-\Sigma}\lVert\Psi_{t,1}^{\top}\rVert\\
&\le\kappa_F\gamma_F^{t-1}\norm{W-\Sigma}\sqrt{\frac{\norm{W+LVL^{\top}}\kappa_F^2}{\lambda_{\min}(W)(1-\gamma_F^2)}},
\end{align*}
which proves \eqref{eqn:convergence of Sigma_t}.

Finally, we prove \eqref{eqn:convergence of L_t}. From \eqref{eqn:Kalman gain} and \eqref{eqn:steady state Kalman gain}, one can show that (see e.g. \cite{zhang2023learning}) for any $t\ge0$,
\begin{align*}
L_t-L&=A\Sigma_tC^{\top}(C\Sigma_t C^{\top}+V)^{-1}-A\Sigma C^{\top}(C\Sigma C^{\top}+V)^{-1}\\
&=A\Sigma C^{\top}\Big((C\Sigma_t C^{\top}+V)^{-1}-(C\Sigma C^{\top}+V)^{-1}\Big)+A(\Sigma_t-\Sigma)C^{\top}(C\Sigma_t C^{\top}+V)^{-1}\\
&=A\Sigma C^{\top}(C\Sigma C^{\top}+V)^{-\top}C(\Sigma-\Sigma_t)C^{\top}(C\Sigma_t C^{\top})^{-1}-A(\Sigma-\Sigma_t)C^{\top}(C\Sigma C^{\top}+V)^{-1}\\
&=(LC-A)(\Sigma-\Sigma_t)C^{\top}(C\Sigma_t C^{\top}+V)^{-1}.
\end{align*}
It follows from \eqref{eqn:convergence of Sigma_t} and Lemma~\ref{lemma:strongly stable of L} that for any $t\ge1$, 
\begin{align*}
\norm{L_t-L}&\le\norm{A-LC}\norm{\Sigma_t-\Sigma}\norm{C}\lVert (C\Sigma_t C^{\top}+V)^{-1}\rVert\\
&\le\kappa_F^3\gamma_F^{t}\norm{W-\Sigma}\sqrt{\frac{\norm{W+LVL^{\top}}}{{\lambda_{\min}(W)(1-\gamma_F^2)}}}\cdot\frac{\norm{C}}{\lambda_{\min}(V)},
\end{align*}
which proves \eqref{eqn:convergence of L_t}.
\end{proof}

\subsection{Matrix and Algebra}
\begin{lemma}
\label{lemma:norm relation of PSD matrices} For any positive semi-definite matrix $P,Q\in\R^{n\times n}$ with $P\preceq Q$, it holds that $\norm{P}\le\norm{Q}$.
\end{lemma}
\begin{proof}
Since $P$ and $Q$ are symmetric, we know that $\norm{P}=\lambda_{\max}(P)$ and $\norm{Q}=\lambda_{\max}(Q)$. By the relation $P\preceq Q$, we have $x^{\top}Px\preceq x^{\top}Qx$ for all $x\in\R^n$. Since $\lambda_{\max}(M)=\max_{x\in\R^n,\norm{x}=1}x^{\top}Mx$ for any $M\in\R^{n\times n}$, we obtain that $\lambda_{\max}(P)\le\lambda_{\max}(Q)$.
\end{proof}

\begin{lemma}
\label{lemma:power of perturbed matrix}
Consider any matrix $M\in\R^{n\times n}$ and any sequence of matrices $\Delta_0,\Delta_1,\dots$, with $\Delta_k\in\R^{n\times n}$ for all $k\ge0$. Suppose $\norm{\Delta_k}\le\varepsilon$ for all $k\ge0$ and for some $\varepsilon\in\R_{>0}$, and suppose exist $\kappa\in\R_{>0}$ and $\gamma\in\R_{>0}$ such that $\norm{M^k}\le\kappa\gamma^k$ for all $k\ge0$. Then,
\begin{equation*}
\label{eqn:power of perturbed matrix}
\norm{(M+\Delta_{k_2-1})(M+\Delta_{k_2-2})\cdots\\\times (M+\Delta_{k_1})-M^{k_2-k_1}}\le(k_1-k_1)\kappa^2(\kappa\varepsilon+\gamma)^{k_2-k_1-1}\varepsilon,
\end{equation*}
for all $k_1,k_2\in\BZ_{\ge0}$ with $k_2>k_1$.
\end{lemma}
\begin{proof}
The proof follows from that of \cite[Lemma~5]{mania2019certainty}.
\end{proof}

\begin{lemma}
\label{lemma:series upper bound}
For any $\gamma\in(0,1)$, it holds that 
\begin{align*}
\sum_{k=2s}^t(k-s)\gamma^{k-s-1}&\le\frac{\gamma^{s-1}\big((1-\gamma)s+\gamma\big)}{(1-\gamma)^2}.
\end{align*}
\end{lemma}
\begin{proof}

We have
\begin{align*}
\sum_{k=2s}^t(k-s)\gamma^{k-s-1}=\sum_{k=s}^{t-s-1}k\gamma^{k-1}\le\sum_{k=s}^{\infty}k\gamma^{k-1}.
\end{align*}
Similarly, consider the series
\begin{align*}
\sum_{k=s}^{\infty}\gamma^k=\gamma^s\sum_{k=0}^{\infty}\gamma^k=\frac{\gamma^s}{1-\gamma}.
\end{align*}
Differentiating both sides of the above equation with respect to $\gamma$, one can show that 
\begin{align*}
\sum_{k=s}^{\infty}k\gamma^{k-1}=\frac{d}{d\gamma}\left(\frac{\gamma^s}{1-\gamma}\right)=\frac{\gamma^{s-1}\big((1-\gamma)s+\gamma\big)}{(1-\gamma)^2},
\end{align*}
which proves the second inequality of the lemma.
\end{proof}

\end{document}